%% file: main.tex
\documentclass[sigconf]{acmart}

\AtBeginDocument{%
  }

\setcopyright{acmlicensed}
\copyrightyear{2026} 
\acmYear{2026} 

\acmDOI{XXXXXXX.XXXXXXX}

\acmConference[CCS '26]{Proceedings of the 2026 ACM SIGSAC Conference on Computer and Communications Security}{November 15-19, 2026}{The Hague, The Netherlands.}
\acmISBN{978-1-4503-XXXX-X/2018/06}

\usepackage{hyperref}
\usepackage{xspace}

\usepackage{graphicx}
\usepackage{wrapfig}
\usepackage{float}
\usepackage{afterpage}
\usepackage{placeins}

\usepackage[compatibility=false]{caption}
\usepackage{subcaption}
\usepackage{booktabs}
\usepackage{multirow}
\usepackage{multicol}
\usepackage{makecell}
\usepackage{caption}
\usepackage{enumerate} 
\usepackage{stfloats}

\usepackage{amsthm}
\usepackage{amsmath}
\usepackage{glossaries}
\usepackage{bbm}
\usepackage{bm}

\usepackage{algorithm}
\usepackage{algorithmic}

\usepackage[T1]{fontenc}
\usepackage{xcolor}
\usepackage{color}
\usepackage{comment}
\usepackage{bbding}
\usepackage{pifont}
\usepackage[utf8]{inputenc}
\usepackage{soul}       
\usepackage{xcolor}     

\sethlcolor{yellow}

\input{math_commands.tex}

\newtheorem{proposition}{Proposition}

\newtheorem{definition}{Definition}
\newtheorem{Assumption}{Assumption}

\def\eg{\emph{e.g.,}\xspace}

\def\ie{\emph{i.e.,}\xspace}
\def\etal{\emph{et al.}\xspace}

\begin{document}

\title{Rethinking Backdoor Adversarial Unlearning through the Lens of Catastrophic Forgetting in Continual Learning}


\author{Zhenqian Zhu}
\email{23b3510101@stu.hit.edu.cn}
\affiliation{
  \institution{Harbin Institute of Technology, Shenzhen}
  \city{Shenzhen}
  \state{Guangdong}
  \country{China}
}

\author{Yamin Hu}
\email{huyamin@hit.edu.cn}
\affiliation{%
  \institution{Harbin Institute of Technology, Shenzhen}
  \city{Shenzhen}
  \state{Guangdong}
  \country{China}
}

\author{Yujiang Liu}
\email{23s151125@stu.hit.edu.cn}
\affiliation{%
  \institution{Harbin Institute of Technology, Shenzhen}
  \city{Shenzhen}
  \state{Guangdong}
  \country{China}
}

\author{Luping Wei}
\email{2416752344@qq.com}
\affiliation{%
 \institution{Harbin Institute of Technology, Shenzhen}
 \city{Shenzhen}
 \state{Guangdong}
 \country{China}
}

\author{Wenbo Hou}
\email{24b951083@stu.hit.edu.cn}
\affiliation{%
  \institution{Harbin Institute of Technology, Shenzhen}
  \city{Shenzhen}
  \state{Guangdong}
  \country{China}
}

\author{Bin Li}
\email{libin@szu.edu.cn}
\affiliation{%
  \institution{Shenzhen Key Laboratory of Media Security, Shenzhen University}
  \city{Shenzhen}
  \state{Guangdong}
  \country{China}
}

\author{Haodong Li}
\email{lihaodong@szu.edu.cn}
\affiliation{%
  \institution{Shenzhen Key Laboratory of Media Security, Shenzhen University}
  \city{Shenzhen}
  \state{Guangdong}
  \country{China}
}

\author{Wenjian Luo}
\email{luowenjian@hit.edu.cn}
\affiliation{%
  \institution{Harbin Institute of Technology, Shenzhen}
  \city{Shenzhen}
  \state{Guangdong}
  \country{China}
}


\begin{abstract}
Existing studies reveal that current backdoor defenses exhibit limited robustness and often fail against specific types of attacks. More concerningly, 
prevailing safety tuning strategies tend to provide only superficial safety protection, as they fall short of completely eliminating the backdoor effects.
In this work, we present a novel formulation of backdoor learning and unlearning as a sequential, three-stage process from a continual learning perspective.
Within this framework, we formally define complete backdoor unlearning and further derive the necessary conditions for achieving it based on the mechanism of 
catastrophic forgetting. Guided by these insights, we propose Blind Inversion-Backdoor Adversarial Unlearning (BI-BAU), which formulates the generation of 
adversarial examples satisfying the unlearning conditions as a blind inversion problem. We solve this by integrating the bi-level optimization process of 
adversarial training into an Expectation-Maximization (EM) algorithm framework to optimize the maximum a posteriori (MAP) objective. Furthermore, BI-BAU is extended 
to untargeted adversarial scenarios with unknown target classes, as well as to multi-modal contrastive learning tasks, enhancing its applicability to real-world 
deployment scenarios where pre-trained models may be compromised. Extensive experiments demonstrate that our method exhibits general applicability across a wide 
spectrum of backdoor attacks, and can effectively and thoroughly eliminate the backdoor effects from a backdoor model.
\end{abstract}





\maketitle

\input{section/Introduction}
\input{section/Related_Work}

\input{section/Preliminary}
\input{section/Methodology}

\input{section/Experiments}

\input{section/Conclusion}

\input{section/Acknowledgment}

\bibliographystyle{ACM-Reference-Format}
\bibliography{Reference}


\input{section/Supplementary}
\end{document}

%% file: math_commands.tex
\def\vtheta{{\bm{\theta}}}

\DeclareMathAlphabet{\mathsfit}{\encodingdefault}{\sfdefault}{m}{sl}
\SetMathAlphabet{\mathsfit}{bold}{\encodingdefault}{\sfdefault}{bx}{n}

\def\x{\bm{x}}

%% file: section/Introduction.tex
\section{Introduction}

Extensive research~\cite{yin2024physical, li2022backdoor,gu2019badnets} has demonstrated that deep neural networks (DNNs) are vulnerable to backdoor attacks,
where an adversary could implant a hidden `shortcut'~\cite{geirhos2020shortcut,yu2022availability,huang2022backdoor}  that links a predefined trigger to 
adversary-specified malicious model behavior through data poisoning. As a result, the compromised model behaves normally on clean inputs but produces attacker-desired 
outputs whenever the trigger is present. To mitigate such threats and enable the secure deployment of models, a variety of safety tuning strategies have been proposed, 
which can be broadly categorized into three categories, including trigger synthesis-based methods~\cite{wang2019neural,qiao2019defending,xu2024towards}, 
model reconstruction-based methods~\cite{li2021neural,zheng2022pre,wu2021adversarial,liu2018fine}, and backdoor adversarial unlearning methods~\cite{gao2023effectiveness,zeng2021adversarial,wei2023shared}.

\begin{figure}[!h]
	\graphicspath{{figures/re_activation/}}
	\centering
	\begin{subfigure}[b]{0.23\textwidth}
		\centering
		\includegraphics[width=\linewidth]{BadNets.pdf}
		\caption{BadNets}
	\end{subfigure}
	\begin{subfigure}[b]{0.23\textwidth}
		\centering
		\includegraphics[width=\linewidth]{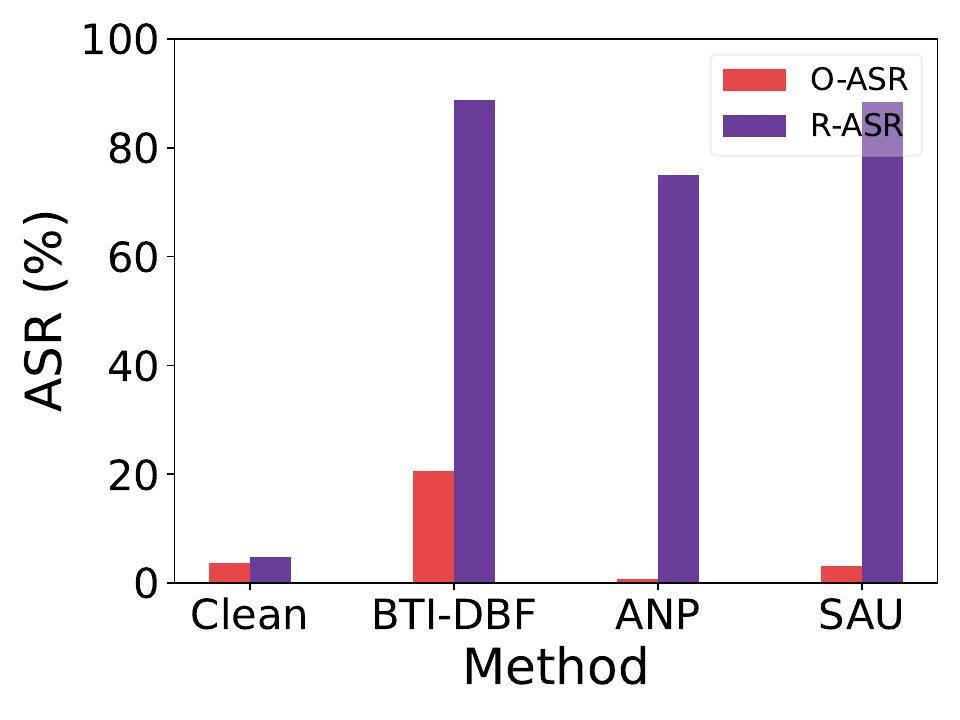}
		\caption{Refool}
	\end{subfigure}
	\begin{subfigure}[b]{0.23\textwidth}
		\centering
		\includegraphics[width=\linewidth]{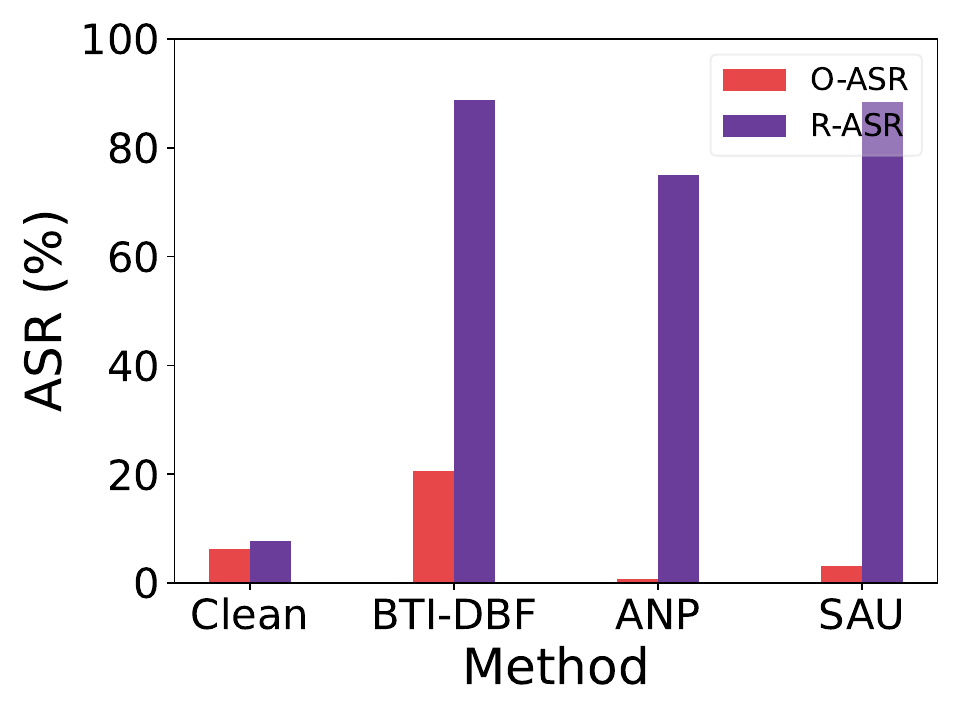}
		\caption{IAD}
	\end{subfigure}
  \begin{subfigure}[b]{0.23\textwidth}
		\centering
		\includegraphics[width=\linewidth]{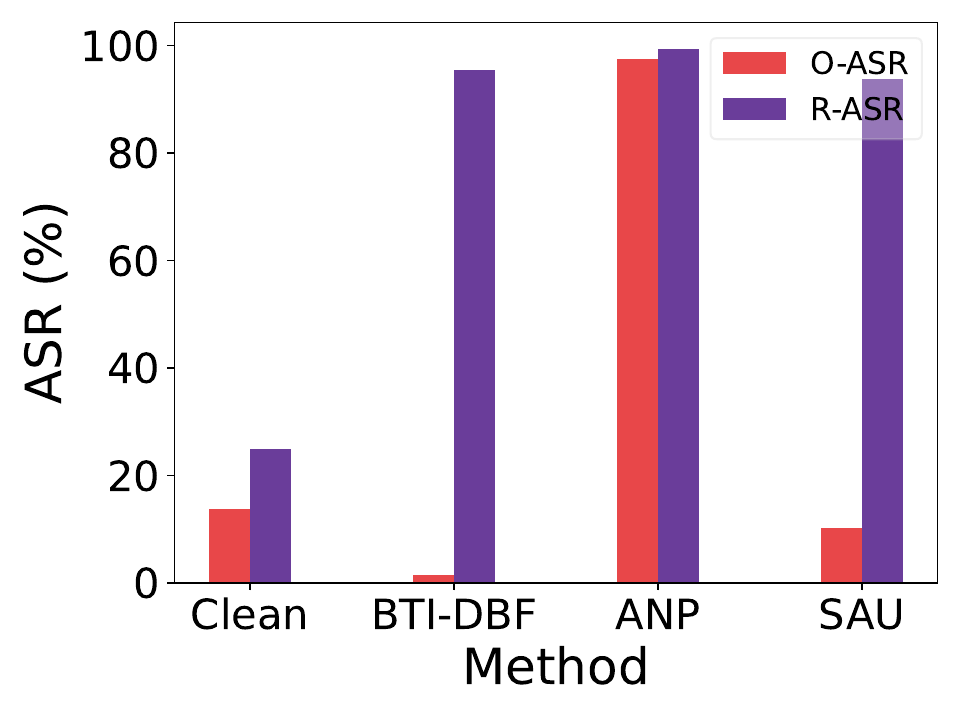}
		\caption{Narcissus}
	\end{subfigure}
	\caption{\small Robustness performance of state-of-the-art safety tuning strategies, including BTI-DBF~\cite{xu2024towards}, ANP~\cite{wu2021adversarial} 
	and SAU~\cite{wei2023shared} under various attack settings. The O-ASR metric denotes the attack success rate (ASR) under the original defense methods, 
	while the R-ASR indicates the ASR after applying the Retuning Attack. }
	\vspace{-3mm}
	\label{fig:retuning_attack}
\end{figure}

Despite the promising progress in safety tuning strategies, prior studies~\cite{qi2023revisiting, zhang2024exploring} show that their robustness remains limited,
particularly against adaptive attacks or those characterized by low orthogonality or low linearity. More alarmingly, growing evidence~\cite{zhu2024breaking,min2024uncovering} 
suggests that these defenses may merely provide a false sense of security: the purified models still retain residual backdoor features that can be readily reactivated.
As illustrated by the Retuning Attack proposed by~\citet{min2024uncovering}, the backdoor can be effectively reactivated by fine-tuning a purified model 
for a few epochs using a dataset that contains only a small number of poisoned samples. As shown in Figure~\ref{fig:retuning_attack}, although state-of-the-art 
safety tuning methods achieve low attack success rates (ASR) against diverse backdoor attacks, they remain notably vulnerable to the Retuning Attack (RA). 
When exposed to RA, the backdoor can be rapidly reactivated, causing the R-ASR (\ie the ASR after applying RA) to surge above 80\% in most cases, whereas clean models
exhibit strong resilience under the same attack settings.  These observations indicate that existing safety tuning techniques fail to fundamentally remove embedded 
backdoor features. Query-based Reactivation Attacks (QRA)~\cite{min2024uncovering} further confirm that even in a black-box setting, residual backdoors can be 
efficiently reactivated.


Most existing defense strategies focus on output-level robustness, addressing backdoor effects solely through prediction outcomes. Such approaches lack 
a principled understanding of how backdoor knowledge is encoded in model parameters, leaving residual backdoor effects that make purified models vulnerable to 
reactivation attacks. Although model reconstruction-based methods~\cite{li2021neural,zheng2022pre,wu2021adversarial,liu2018fine} attempt to directly identify 
and remove backdoor-related parameters, they struggle to balance backdoor mitigation and model performance, particularly when clean and backdoor features are 
entangled, as in attacks with low orthogonality or low linearity~\cite{zhang2024exploring}.

In this work, we conceptualize backdoor effects as a form of task-specific memory embedded in model parameters, aiming to establish a theoretical framework 
for deep backdoor memory removal while preserving clean task knowledge, which guides the design of backdoor unlearning methods. Continual learning (CL)~\cite{thrun1995lifelong}, 
which studies memory retention and forgetting across sequential tasks, provides a natural theoretical foundation. It allows us to systematically analyze the dynamics 
of backdoor forgetting through unlearning and derive foundational design principles. Compared with direct unlearning approaches, the forgetting perspective in 
Continual learning avoids directly operating in the parameter space and instead provides a principled framework for selectively forgetting backdoor-specific memory 
while preserving clean task knowledge through task alignment.

Continual learning (CL)~\cite{thrun1995lifelong,chen2022lifelong,wang2024comprehensive} is a setting in which an agent must learn from an incoming stream of data 
throughout its lifetime. A major challenge is Catastrophic Forgetting (CF)~\cite{kirkpatrick2017overcoming,bennani2020generalisation,doan2021theoretical}, 
where learning a new task could abruptly degrade performance on previously learned tasks. Prior work~\cite{nguyen2019toward} shows that forgetting is closely 
tied to task similarity: higher alignment between tasks results in more severe CF. \citet{doan2021theoretical} futher provide a general definition of CF and 
introduce the Neural Tangent Kernel (NTK) overlap matrix as a measure of task similarity. Recently, Zhang~\etal~\cite{zhang2024exploring} first formulate 
backdoor learning as a continual learning problem, with two independent phases: \ding{172} an initial rapid learning phase in which the backdoor task is 
acquired within just a few training epochs, followed by \ding{173} a subsequent phase of gradual adaptation to the clean task, which provides a novel perspective 
for understanding backdoor attacks.

Inspired by these studies, we build upon adversarial training and rethink backdoor adversarial unlearning through the lens of continual learning (CL).
Our goal is to establish a formal definition of backdoor unlearning based on the catastrophic forgetting (CF) mechanism\cite{bennani2020generalisation,doan2021theoretical} 
and investigate its core challenges, particularly the fundamental principles that govern effective backdoor removal.

\emph{What conditions must an unlearning task satisfy to completely eliminate backdoor effects while preserving performance on clean tasks?}

In this paper, we address the above problem through the lens of catastrophic forgetting in continual learning. First, we uniformly consider the processes of 
backdoor learning and backdoor unlearning as a stream of three-stage learning tasks $\mathcal{T}=\{\tau_c, \tau_b, \tau_u\}$, where $\mathcal{T}$ consists of 
a sequential order of clean task $\tau_c$, backdoor task $\tau_b$, and backdoor unlearning task $\tau_u$. Then, we provide a formal definition of complete backdoor unlearning,
in which the backdoor effects are completely removed after the unlearning operation, and the resulting purified model achieves performance on the clean task 
that is comparable to that of the clean model. Furthermore, we introduce a necessary condition that the unlearning task $\tau_u$ must satisfy to approximately 
achieve complete backdoor unlearning: the unlearning task $\tau_u$ should be aligned with the backdoor task $\tau_b$ while remaining orthogonal to the clean task $\tau_c$. 
Specifically, for the problem of backdoor adversarial unlearning, we propose a fundamental principle for crafting adversarial examples during adversarial training,
\ie the adversarial perturbation should closely align with the backdoor features while preserving the primary semantic content of the original sample.

Subsequently, we formally cast backdoor adversarial unlearning as a blind inversion problem with a maximum a posteriori (MAP) objective. Furthermore, we integrate 
a bi-level formulation of adversarial training within the Expectation-Maximization (EM) framework to effectively optimize the MAP$_{\theta}$ objective. 
Based on the above analysis, we propose a robust backdoor adversarial unlearning method, termed Blind Inversion-Backdoor Adversarial Unlearning (BI-BAU). 
We further extend BI-BAU to untargeted adversarial scenarios with unknown target classes $y_t$, as well as to multi-modal contrastive learning settings, thereby 
broadening its applicability in real-world scenarios. Extensive experiments demonstrate the universal effectiveness of BI-BAU against a wide range of backdoor attacks,
including those characterized by low orthogonality or low linearity. Moreover, BI-BAU exhibits strong post-purification robustness, maintaining a low ASR even 
under backdoor re-activation attacks.

In conclusion, our main contributions are three-fold:
\begin{itemize}
    \item \textbf{Theoretical foundation for backdoor unlearning.} To the best of our knowledge, we are the first to systematically investigate the fundamental 
	challenges of backdoor unlearning from the perspective of continual learning. By leveraging the catastrophic forgetting mechanism, we provide 
	a formal definition of complete backdoor unlearning and derive the necessary condition that an unlearning task must satisfy to achieve thorough removal 
	of backdoor effects.
	
    \item \textbf{Blind Inversion-Backdoor Adversarial Unlearning (BI-BAU).} Building on the theoretical insights, we propose BI-BAU, which formalizes 
	backdoor adversarial unlearning as a blind inversion problem with a $\text{MAP}_{\theta}$ optimization objective. We solve this problem through a bi-level 
	adversarial training process integrated within the Expectation-Maximization (EM) framework. BI-BAU is further extended to untargeted attacks with unknown 
	target classes and to multi-modal contrastive learning scenarios, enhancing its applicability to realistic deployment settings.
	
    \item \textbf{Comprehensive robustness evaluation.} We conduct a comprehensive evaluation of existing backdoor attacks, demonstrating that our method is 
	broadly effective across a wide spectrum of attacks, including more challenging cases characterized by low orthogonality or low linearity, 
	while achieving near-complete removal of backdoor effects while maintaining high performance on clean tasks.
	
\end{itemize}

%% file: section/Related_Work.tex
\section{Related Work}

\subsection{Continual Learning}

Continual learning has been introduced to meet the needs of intelligent systems to incrementally acquire, update, accumulate, and exploit knowledge throughout 
their lifetime, enabling them to adapt to dynamic environments~\cite{wang2024comprehensive}. Continual learning is characterized by learning from dynamic data 
distributions, where the model receives a stream of learning tasks with training samples drawn from different distributions.
Despite significant advancements in the field, continual learning remains explicitly constrained by catastrophic forgetting~\cite{mccloskey1989catastrophic}, 
wherein learning a new task often leads to a severe degradation in performance on previously learned tasks. 
In the context of continual learning, existing research on catastrophic forgetting~\cite{kirkpatrick2017overcoming,doan2021theoretical} primarily focuses on mitigating 
or avoiding its detrimental effects, such as the substantial reduction in the model's ability to retain previous distributions  when adapting to new ones.

However, in contrast to studies~\cite{kirkpatrick2017overcoming,doan2021theoretical} that seek to overcome forgetting, we shift our perspective to investigate 
its potential benefits. We aim to incorporate the problem of backdoor unlearning into the continual learning framework, so that we can leverage it as a theoretical 
tool to analyze how backdoors can be effectively erased while preserving the model's performance on clean tasks through the lens of catastrophic forgetting.

\subsection{Backdoor Attacks}

Gu~\etal~\cite{gu2019badnets} pioneered backdoor attacks in deep learning with BadNets, initiating extensive research in this area. Early attacks were largely 
sample-agnostic, using uniform triggers easily detected by synthesis-based defenses~\cite{wang2019neural,qiao2019defending}. This prompted the emergence of 
sample-specific backdoor attacks~\cite{li2021invisible,nguyen2020input,salem2022dynamic} with varying triggers. To enhance stealthiness, various invisibility strategies~\cite{nguyen2021wanet,liu2020reflection}
have been proposed. However, dirty-label attacks~\cite{gu2019badnets,chen2017targeted,nguyen2020input} remain detectable via label-image inconsistencies, clean-label attacks~\cite{turner2019label,zeng2023narcissus}
circumvent this by aligning poisoned samples’ labels with target classes. In addition to the various forms of backdoor attacks, research has also delved into exploring
their underlying mechanisms. Existing work~\cite{geirhos2020shortcut,huang2022backdoor} suggests that the effectiveness of backdoor attacks largely stems from 
shortcut learning, which induces a \emph{latent separability} phenomenon: feature representations of poisoned samples tend to form distinct clusters that are 
well separated from those of clean samples. This observation provides a key theoretical foundation upon which many existing backdoor defense methods are built.

However, Qi~\etal\cite{qi2023revisiting} challenge this assumption by proposing two adaptive attacks: Adaptive-Blend and Adaptive-Patch, as counter-examples.
Moreover, Zhang~\etal~\cite{zhang2024exploring} formulate backdoor attacks as a continual learning task and introduce two key properties of backdoor attacks: 
orthogonality and linearity. They conduct a systematic study on existing mainstream backdoor attacks and defenses and show that existing defenses are particularly
vulnerable to attacks with low orthogonality or low linearity. Collectively, these findings indicate that existing defenses offer only limited robustness.

\subsection{Backdoor Defenses}

Safety tuning strategies have been proposed to remove backdoor effects from a pre-deployed backdoor model, including trigger synthesis-based methods, 
model reconstruction-based methods, and backdoor adversarial unlearning methods. Trigger synthesis-based methods~\cite{wang2019neural,qiao2019defending,xu2024towards} 
reconstruct triggers via model inversion techniques and then unlearn the hidden backdoor, whereas model reconstruction-based approaches~\cite{li2021neural, zheng2022pre,wu2021adversarial,liu2018fine} 
prune backdoor-related neurons exhibiting abnormal sensitivities to poisoned samples and recover clean accuracy using a small clean set.
Backdoor adversarial unlearning~\cite{weng2020trade,gao2023effectiveness,zeng2021adversarial, wei2023shared} leverages adversarial training (AT) to enhance robustness,
though naive AT often incurs accuracy degradation due to training instability~\cite{weng2020trade,gao2023effectiveness}. Most existing work~\cite{zeng2021adversarial, wei2023shared} 
largely prioritize improving model performance on clean tasks, while offering only limited investigation into the core challenges of backdoor defense.

Despite the promising progress made by existing defenses, recent studies~\cite{zhu2024breaking,min2024uncovering} have revealed that current safety tuning techniques 
merely provide a false sense of security. Residual backdoor features remain in so-called purified models and can be readily reactivated by adversaries through 
sample-specific perturbations generated solely via model queries. Although the Path-Aware Minimization (PAM) method~\cite{min2024uncovering} achieves promising 
results in deep backdoor removal, its reliance on BTI-DBF~\cite{xu2024towards} for generating backdoor samples inherently limits its effectiveness against attacks 
characterized by low orthogonality or low linearity, due to BTI-DBF's constrained trigger reverse-engineering capability. Consequently, genuinely removing backdoors 
remains a critical and unresolved challenge in practice.



%% file: section/Preliminary.tex
\section{Preliminary}

\subsection{Notation}
\label{section:notation}

Let $\mathcal{D} \subseteq \mathbb R^{n_0} \times \mathbb R^{K}$ denote the training set, with $\mathcal{X}=\left\{x: (x,y)\in \mathcal{D}\right\}$ and $\mathcal{Y}=
\left\{y: (x,y)\in \mathcal{D}\right\}$ representing the input and label sets, respectively. In the context of backdoor attacks,  we consider a $K$-class ($K\geq 2$) 
classification problem. The model $f$ is implanted with a backdoor by training on a corrupted dataset $D$, which consists of both a clean subset $D_{c}$ and 
a poisoned backdoor subset $D_{p}$, \ie $D = D_{c} \cup D_{p}$. The backdoor model misclassifies any input $x$ carrying the backdoor trigger as the target label $y_t$.

For clarity and without loss of generality, we consider a fully-connected feed-forward network $f= FC \circ f_{L} \circ f_{L-1} \circ ...\circ f_{1}$ with $L$ hidden layers,
where the \textit{l}-th layer is denoted by $f_{l}$ for $l = 1, ..., L$, and FC denotes the fully-connected layers. Overall, the architecture of the model $f$ can be viewed as
consisting of a feature extractor $f_{\text{backbone}} = f_L \circ f_{L-1} \circ \dots \circ f_1 $ (i.e., the Backbone Network) followed by a fully connected (FC) classification head.
For an input $x\in\mathbb R^{n_0}$, let $h^l(x)$ and $z^l(x)\in\mathbb R^{n_l}$ denote the pre-activation and post-activation representations at layer $l$, respectively. 
In particular, we define the output of the feature extractor as $z(x) = f_{\text{backbone}}(x)$. The recurrence relation for a feed-forward network is defined as 
\begin{align}
    \label{eq:recurrence}
    \begin{cases}
        h^{l+1} =z^{l}w^{l+1} + b^{l+1}\\
        z^{l+1} =\phi\left(h^{l+1}\right) 
    \end{cases}
\end{align}
where $\phi$ is a pointwise activation function, and $w^{l+1}\in \mathbb R^{n_l\times n_{l+1}}$ and $b^{l+1}\in\mathbb R^{n_{l+1}}$ represent the weight matrix 
and bias vector, respectively.  

In continual learning (CL), we encounter a sequence of supervised learning tasks $\mathcal{T}_{\tau}$, where $\tau \in [T]$ serves as the task identifier, 
with $T \in \mathbb{N}^{*}$. Here, $X^{\tau}$ = $\{(x^{\tau}_{j}, y^{\tau}_{j})\}_{j=1}^{n_{\tau}}$ represents the dataset associated with task $\mathcal{T}_{\tau}$.
The objective is to learn a predictor $f_{\theta}: \mathcal{X} \times \mathcal{T} \to \mathcal{Y}$ parameterized by $\theta$ that achieves the prediction accuracy as high as possible.
In the CL setting, access to data from previous tasks is restricted unless explicitly stored in a memory buffer \cite{lopez2017gradient,parisi2019continual}.  
We denote the Euclidean norm of a vector or the spectral norm of a matrix by $\|\cdot\|_2$ and use $\left \langle \cdot,\cdot \right \rangle$ to represent the Euclidean dot product. 
The task index is denoted by $\tau$, and learnable parameters by $\theta$, where $\theta_{\tau}$ refers to the parameters during training on task $\tau$. 
Moreover, the symbol $*$ represents the state of a variable at the end of a task, so $\theta_{\tau}^{*}$ represents the learned parameters upon completing task $\tau$.
A comprehensive list of all notations in Appendix~\ref{appendix:summary_of_symbols}.

\subsection{Threat Model}
\label{section:threat_model}

To formally study backdoor adversarial unlearning, we define the capabilities of the adversary and defender in a realistic deployment scenario.

\textbf{Adversary}. We consider a realistic deployment scenario in which practitioners primarily obtain pre-trained models from third-party platforms (\eg HuggingFace~\cite{jain2022hugging}) 
to reduce training costs and accelerate downstream development. Despite their strong performance and practical convenience, such resource models may be compromised 
during pre-training and thus contain latent backdoors. To ensure security and reliability, these models are typically subjected to safety tuning or purification 
before being fine-tuned for downstream tasks or deployed in real-world systems.

In this work, we focus on post-purification security, motivated by recent evidence~\cite{zhu2024breaking,min2024uncovering} that existing safety tuning strategies 
are insufficient to fundamentally eliminate backdoor behaviors. This allows these compromised models to be fine-tuned for downstream tasks while the backdoors persist.
We consider an adversary aiming to reactivate these residual backdoors following the defender's application of safety tuning. Importantly, the adversary has no 
access to the training data, purification procedure, or model parameters. Instead, the attacker operates in a query-only black-box setting, which closely reflects 
practical deployment scenarios. Specifically, the adversary can only submit inputs to the post-purified model and observe the corresponding outputs. 
Under this threat model, the attacker can launch Query-based Reactivation Attacks (QRA)~\cite{min2024uncovering}, which exploit model queries to generate 
sample-specific perturbations that trigger residual backdoor behaviors.

\textbf{Defender}. Consistent with existing safety tuning approaches~\cite{zeng2021adversarial,wei2023shared}, we assume the defender has access to a small clean 
dataset $\mathcal{D}_{cl} = \left\{(x_i, y_i) \right\}_{i=1}^{N}$ for post-hoc fine-tuning. We further define $\mathcal{D}_{-y_t} = \left\{(x_i, y_i) | (x_i, y_i) \in \mathcal{D}_{cl}, y_i \ne y_t \right\}$
as the subset of these non-target samples. It is important to note that $\mathcal{D}_{cl}$ is limited, and the defender is unable to train a new model from scratch
based on $\mathcal{D}_{cl}$.

The defender's objective is to eliminate the backdoor effects embedded in the suspicious model while preserving its performance on clean tasks. 
To this end, we formulate backdoor adversarial unlearning as the following optimization problem:
\begin{equation}
    \label{eq:backdoor_adversarial_unlearning}
    \min_{\theta} \mathcal{R}_{cl} + \mathcal{R}_{adv}
\end{equation}
where 
\begin{equation}
\label{eq:items_for_backdoor_adversarial_unlearning}
\left\{\begin{matrix}
    \underset{\theta}{\min}\;\mathcal{R}_{cl} = \underset{\theta}{\min}\;\mathbb{E}_{(x,y)\sim\mathcal{D}_{cl}} (\ell(f(x;\theta),y)),\\
    \underset{\theta}{\min}\;\mathcal{R}_{adv} = \underset{\theta}{\min}\;\underset{\delta \in S}{\max}\;\mathbb{E}_{(x,y)\sim\mathcal{D}_{-y_t}} (\ell_{adv}(f(\tilde{x};\theta),y)).
\end{matrix}\right.
\end{equation}
The first term, $\mathcal{R}_{cl}$ in Eq.~\ref{eq:backdoor_adversarial_unlearning}, represents the classification risk of the model on the clean dataset,
where $\theta$ denotes the model parameters, and the second term, $\mathcal{R}_{adv}$, refers to the adversarial risk. Adversarial training (AT) formulates a minimax 
optimization problem to reduce this risk, where $\tilde{x} = x + \delta $ is the adversarial example and the perturbation set $S \subseteq \mathbb{R}^{n_0}$ 
characterizes the adversary's manipulative capability.

\subsection{NTK framework for Continual Learning}
\citet{lee2019wide} demonstrate that, under the NTK regime, neural networks evolve as a linear model:
\begin{equation}
    \label{eq:NTK_for_CL}
    f_{\tau}^{*}(x)=f_{\tau-1}^{*}(x)+\langle \nabla_{\theta}f_{0}(x), \theta_{\tau}^{*} - \theta_{\tau-1}^{*} \rangle
\end{equation}
with $\theta_{\tau}^{*}$ being the final weight after training on task $\tau$. The latter formulation implies the feature maps $\nabla_{\theta}f_{0}(x) \in \mathbb{R}^{1 \times p}$ 
are constant over time. Building on this framework, \citet{bennani2020generalisation} showed that continual learning (CL) dynamics can be characterized as 
a recursive kernel regression process. As a result, the time evolution of the linearized neural network can be analytically derived without running gradient descent. 


\subsection{Catastrophic Forgetting Under the NTK Regime}
In line with the theoretical framework established by \citet{doan2021theoretical}, we adopt their definitions of Catastrophic Forgetting under the Neural Tangent Kernel (NTK) regime, 
which serve as a basis for our analysis.
\begin{definition} [\textbf{Drift}] 
    \label{definition:drift}
    Let $\tau_S$ (respectively $\tau_T$) be the source task (respectively target task), $\mathcal{D}_{\tau_S}$ the source test set, the CF of task $\tau_S$ after training on all the subsequent tasks up to the target task $\tau_T$ is defined as:
    \begin{align}
        \delta^{\tau_S \rightarrow \tau_T}(X^{\tau_S})= \Bigl(f^{\star}_{\tau_T}(x)-f^{\star}_{\tau_S}(x)\Bigr)_{(x,y)\in \mathcal{D}_{\tau_S}}.
    \end{align}
\end{definition}
Note that $\delta^{\tau_S \rightarrow \tau_T}(X^{\tau_S})$ represents prediction changes for inputs from the task $\tau_S$ after learning task $\tau_T$.
In the classification setting, we consider the K-dimensional output of $f_{\tau}^{\star}$. To quantify the overall extent of forgetting on task $\tau_S$,
we use the squared norm of this change vector.

\begin{definition} [\textbf{Catastrophic Forgetting}] \
\label{definition:catastrophic_forgetting}
Let $\tau_S$ (respectively $\tau_T$) be the source task (respectively target task), $\mathcal{D}_{\tau_S}$ the source test set, the Catastrophic Forgetting of a 
source task $\tau_S$ with respect to a target task $\tau_T$ is given by:
\begin{align}
    \label{eq:CF}
    \Delta^{\tau_S \rightarrow \tau_T}(X^{\tau_S}) & = \|\delta^{\tau_S \rightarrow \tau_T}(X^{\tau_S})\|_2^2 \nonumber \\
    & = \sum_{(x,y)\in \mathcal{D}_{\tau_S}} (f^{\star}_{\tau_T}(x)-f^{\star}_{\tau_S}(x))^{2} .
\end{align}
Let $ \{ \theta_{\tau}^{\star}, \forall \tau \in [T] \}$ be the weight at the end of the training of task $\tau$, 
based on Eq.~\ref{eq:NTK_for_CL}, the catastrophic forgetting~\cite{lee2019wide} is further defined as follows:
\begin{align}
    \Delta^{\tau_S \rightarrow \tau_T}(X^{\tau_S}) & = \|\delta^{\tau_S \rightarrow \tau_T}(X^{\tau_S})\|_2^2 \nonumber \\
    & = \left \|\nabla_{\theta}f_{0}(X^{\tau_S})(\theta_{\tau_T}^{*} - \theta_{\tau_S}^{*})\right \|_{2}^{2}.
\end{align}
\end{definition}
This linear formulation under the NTK regime provides clear insights into the dynamics of catastrophic forgetting.
This framework allows us to quantify how backdoor features can be forgotten or preserved across tasks, which is critical for designing effective backdoor 
unlearning strategies.

%% file: section/Methodology.tex
\section{Methodology}

In Section~\ref{section:problem_formulation}, we first unify backdoor learning and subsequent unlearning within a continual learning framework, introduce 
a formal definition of complete backdoor unlearning, and further derive the key condition for achieving it.
Building on this, Section~\ref{section:blind_inversion} casts the problem as a blind inversion problem and proposes a solution based on maximum a posteriori estimation 
using the Expectation-Maximization (EM) algorithm. Finally, in section~\ref{section:proposed_method}, we propose Blind Inversion-Backdoor Adversarial Unlearning (BI-BAU) 
and extend it to the more general settings, including untargeted adversarial training and multimodal contrastive learning scenarios.

\subsection{Problem Formulation}
\label{section:problem_formulation}

Prior study~\cite{li2021anti} views backdoor learning as a composition of a clean task $\tau_c$ and a backdoor task $\tau_b$. \citet{zhang2024exploring} further 
characterize this process as a two-stage task flow: $\mathcal{T}=\{\tau_b, \tau_c\}$, consisting of an initial rapid learning phase of the backdoor task $\tau_b$ 
within a few training epochs followed by a subsequent phase of gradually learning over the clean task $\tau_c$. They also observe that the gradients from the two tasks 
tend to be orthogonal, indicating an orthogonality between the backdoor and clean task. \citet{doan2021theoretical} introduce the concept of task similarity 
and demonstrate that orthogonal tasks can be considered independent, implying that they can be executed separately without interfering with each other.

Building on these insights, we define backdoor learning process as $\mathcal{T}=\{\tau_b, \tau_c\}$. By incorporating the subsequent unlearning task $\tau_u$, 
the overall process of backdoor learning and unlearning can be expressed as $\mathcal{T} = \left\{\tau_b, \tau_c, \tau_u \right\}$. Given the orthogonality 
between tasks $\tau_b$ and $\tau_c$, we can conceptually reorder them for analysis, leading to the following assumption. 
\begin{Assumption}
    \label{assumption:backdoor_unlearning_as_continual_learning}
    \textbf{(Backdoor Learning and Unlearning as Continual Learning).} 
    The processes of backdoor learning and the subsequent backdoor unlearning can be formalized as a continual learning framework 
    consisting of a three-stage supervised learning task flow: $\mathcal{T}=\{\tau_c, \tau_b, \tau_u\}$, which includes a clean task $\tau_c$ on the clean dataset $X^{\tau_c}$, 
    a backdoor task $\tau_b$ on the poisoned dataset $X^{\tau_b}$, and a backdoor unlearning task $\tau_u$ on the unlearning dataset $X^{\tau_u}$.   
\end{Assumption}
We provide experimental validation supporting the plausibility of Assumption~\ref{assumption:backdoor_unlearning_as_continual_learning} in Appendix~\ref{subsection: validating_assumption_1}.

Building on Assumption~\ref{assumption:backdoor_unlearning_as_continual_learning}, we formalize backdoor learning and unlearning as a three-stage continual learning 
process, which enables us to characterize the  backdoor unlearning objective via catastrophic forgetting. Recalling our optimization objective in Eq.~\ref{eq:backdoor_adversarial_unlearning}: 
``Effectively and completely eliminate backdoor effects while preserving the model's performance on clean tasks as much as possible''. Accordingly, we define 
Complete Backdoor Unlearning to closely approximate this objective.
\begin{definition}
    \label{definition:complete_backdoor_unlearning}
    \textbf{(Complete Backdoor Unlearning).} Let the clean model $f_{\tau_c}^{*}(.;\theta_{\tau_c}^{*})$ is trained solely on the clean task $\tau_{c}$, and let 
    $f_{\tau_b}^{*}(.;\theta_{\tau_b}^{*})$ denote the backdoored model. Based on Assumption~\ref{assumption:backdoor_unlearning_as_continual_learning}, the purified model 
    $f_{\tau_u}^{*}(.;\theta_{\tau_u}^{*})$ can be viewed as the outcome of a continual learning process that starts from the source clean task $\tau_c$, proceeds
    the intermediate backdoor task $\tau_{b}$, and ultimately reaches the target unlearning task $\tau_u$. We say that task $\tau_u$ achieves complete backdoor unlearning,
    if the following conditions hold:
    \begin{equation}
        \label{eq:unlearning_objective}
        \left\{\begin{matrix}
        \Delta^{\tau_c->\tau_u}(X^{\tau_c}) = 0,\\
        \Delta^{\tau_c->\tau_u}(X^{\tau_b}) = 0.\\
        \end{matrix}\right.
    \end{equation}
    or equivalently,
    \begin{equation}
        \label{eq:complete_backdoor_unlearning_01}
         \left\{\begin{matrix}
           f^{*}_{\tau_u}(X^{\tau_b}) = f^{*}_{\tau_c}(X^{\tau_b}) ,\\
           f^{*}_{\tau_u}(X^{\tau_c}) = f^{*}_{\tau_c}(X^{\tau_c}) .\\
        \end{matrix}\right.
    \end{equation}  
\end{definition}

A concise interpretation of Definition~\ref{definition:complete_backdoor_unlearning} is as follows. Eq.~\ref{eq:unlearning_objective} characterizes complete backdoor 
unlearning through its goals: task $\tau_u$ minimizes the forgetting of knowledge from the clean dataset $X^{\tau_c}$ while simultaneously eliminating the backdoor 
effect of the poisoned dataset $X^{\tau_b}$. Eq.~\ref{eq:complete_backdoor_unlearning_01} offers an equivalent functional view, requiring the purified model 
$f_{\tau_u}^{*}$ to match the clean model $f_{\tau_c}^{}$ on both clean and poisoned inputs.

Let us further analyze the underlying implications of Definition~\ref{definition:complete_backdoor_unlearning}. According to the characteristics of backdoor attacks, 
where ``the backdoor model behaves normally on clean datasets, but predicts the poisoned sample into the target label $y_t$'', we can conclude that for the backdoor 
model $f^{*}_{\tau_b}$, for $\forall x_c \in X^{\tau_c}$, we have $f^{*}_{\tau_b}(x_c) = f^{*}_{\tau_c}(x_c)$. However, for $\forall x_b \in X^{\tau_b}$, we have 
$f^{*}_{\tau_b}(x_b) = y_t$, which differs from the clean model's prediction $f^{*}_{\tau_c}(x_b)$, \ie $f^{*}_{\tau_b}(x_b) \ne f^{*}_{\tau_c}(x_b) $.
By incorporating Eq.~\ref{eq:complete_backdoor_unlearning_01}, we obtain
\begin{equation}
    \label{eq:complete_backdoor_unlearning_02}
     \left\{\begin{matrix}
       f^{*}_{\tau_u}(X^{\tau_b})-f^{*}_{\tau_b}(X^{\tau_b}) = -(f^{*}_{\tau_b}(X^{\tau_b})-f^{*}_{\tau_c}(X^{\tau_b})) \ne 0,\\
       f^{*}_{\tau_u}(X^{\tau_c})-f^{*}_{\tau_b}(X^{\tau_c}) = -(f^{*}_{\tau_b}(X^{\tau_c})-f^{*}_{\tau_c}(X^{\tau_c})) = 0.\\
    \end{matrix}\right.
\end{equation}
Eq.~\ref{eq:complete_backdoor_unlearning_02} directly reflects the objective in Eq.~\ref{eq:unlearning_objective}. Specifically, for the objective $\Delta^{\tau_c->\tau_u}(X^{\tau_b}) = 0$, 
the drift induced by the backdoor task $\tau_b$ is compensated by the subsequent unlearning task $\tau_u$. However, for the objective $\Delta^{\tau_c->\tau_u}(X^{\tau_c}) = 0$,
there is no drift at all, implying that the subsequent unlearning task $\tau_u$ does not interfere with the prior clean task $\tau_c$ in any way. 
As a result, the backdoor knowledge associated with $X^{\tau_b}$ is completely forgotten, while the clean knowledge from $X^{\tau_c}$ is preserved.

Subsequently, we further investigate the conditions that the unlearning task $\tau_u$ must fulfill to approximately achieve complete backdoor unlearning,
leading to the following proposition.
\begin{proposition}
    \label{proposition:proposition_1}
    In a continual learning setting under the NTK regime, consider a three-stage learning task flow: $\mathcal{T}=\{\tau_c, \tau_b, \tau_u\}$, which 
    consists of a clean task $\tau_c$ on the dataset $X^{\tau_c}$, a backdoor task $\tau_b$ on $X^{\tau_b}$, and a backdoor unlearning task $\tau_u$ on $X^{\tau_u}$. 
    To approximately achieve complete backdoor unlearning, the task $\tau_u$ should be antagonistic and parallel to the backdoor task $\tau_b$, while remaining 
    orthogonal to the clean task $\tau_c$. Formally, the following conditions should hold:
    \begin{equation}
        \label{eq:condition_2_for_complete_backdoor_unlearning}
         \left\{\begin{matrix}
            \nabla_{\theta}f_{0}(x_u,\theta) \parallel(\theta_{\tau_{b}}^{*} - \theta_{\tau_{c}}^{*}), \forall x_{u} \in X_{\tau_u}, \\
            \nabla_{\theta}f_{0}(x_u,\theta)  \perp (\theta_{\tau_{c}}^{*} - \theta_{0}),  \forall x_{u} \in X_{\tau_u},\\
        \end{matrix}\right.
    \end{equation}
where $\nabla_{\theta}f_{0}(x_u,\theta)$ denotes the gradient of the network with respect to the weights  $\theta$ at samples from $X^{\tau_u}$. $\theta_{0}$ is 
is the initial model parameter, $\theta_{\tau_{b}}^{*} - \theta_{\tau_{c}}^{*}$ and $\theta_{\tau_{c}}^{*} - \theta_{0}$ approximately denote the gradient directions 
associated with the backdoor task $\tau_{b}$ and the clean task $\tau_{c}$, respectively, as interpreted under the NTK linearization.
\end{proposition}
Proof. see Appendix~\ref{proof:proof_of_proposition_1}.

\begin{figure}[!h]
	\graphicspath{{figures/continual_learning/}}
	\centering
    \includegraphics[width=\linewidth]{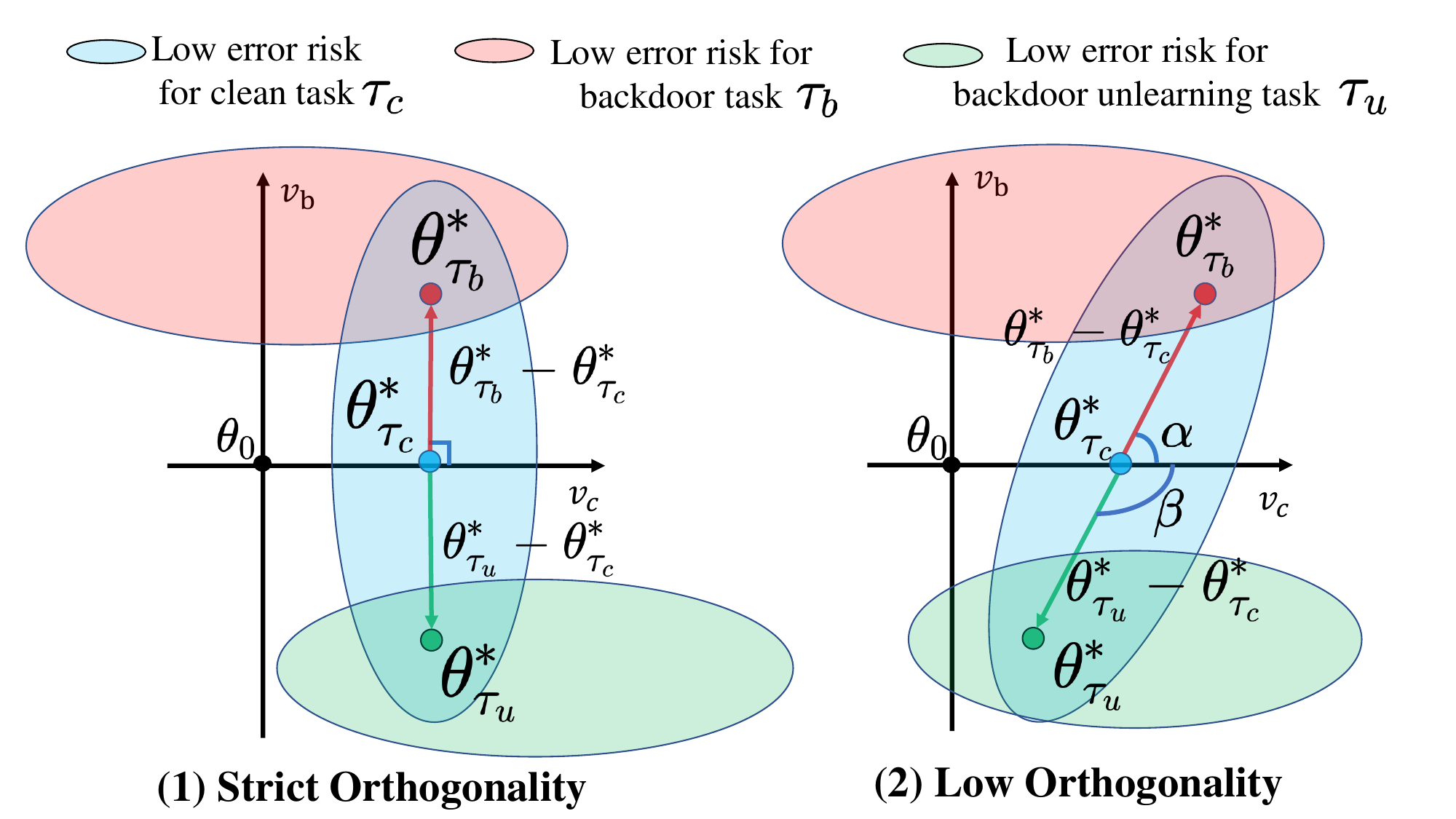}
    \caption{\small 
        Illustration of task similarities within a three-stage learning task flow $\mathcal{T}=\{\tau_c, \tau_b, \tau_u\}$. 
        Parameter regions with low error risk indicate areas yielding good task performance. 
        The vectors $v_b$ and $v_c$ denote the gradient directions associated with backdoor and clean tasks, respectively.
        The vector $\theta_{\tau_{b}}^{*} - \theta_{\tau_{c}}^{*}$ serves as an approximation of $v_b$, and $\theta_{\tau_{u}}^{*} - \theta_{\tau_{c}}^{*}$ 
        approximating the unlearning direction.
        The angles $\alpha$ and $\beta$ measure the similarity of the backdoor task and the unlearning task to the clean task, respectively.
    }
    \label{fig:task_similarity}
    \vspace{-2mm}
\end{figure}
As illustrated in Figure~\ref{fig:task_similarity} (1), Proposition~\ref{proposition:proposition_1} implies that achieving complete backdoor unlearning requires
constructing the unlearning task $\tau_u$ to be antagonistic and parallel to the backdoor task $\tau_b$, thereby maximizing their task similarity and amplifying 
catastrophic forgetting of backdoor features~\cite{doan2021theoretical}. At the same time, the task $\tau_u$ should remain orthogonal to the clean task $\tau_c$
so that the drift induced by $\tau_b$ can be effectively removed without interference with the learned knowledge of $\tau_c$.
In practice, however, the tasks $\tau_c$ and $\tau_b$ are not always strictly orthogonal, especially for attacks exhibiting low orthogonality~\cite{zhang2024exploring, zeng2023narcissus}.
As illustrated in Figure~\ref{fig:task_similarity} (2), for such cases, the angle $\alpha$ falls below $\pi/2$, suggesting a significant overlap or entanglement 
between the clean and backdoor tasks.

In the following, we consider the adversarial training procedure described in the threat model of Section~\ref{section:threat_model} as the unlearning task $\tau_{u}$.
We then focus on constructing adversarial examples in a principled manner to approximate the objective of complete backdoor unlearning as effectively as possible, 
which leads to Proposition~\ref{proposition:proposition_2}.
\begin{proposition}
    \label{proposition:proposition_2}
    In the context of backdoor adversarial unlearning, the adversarial training procedure can be regarded as the unlearning task $\tau_{u}$ with dataset $X^{\tau_u}$, 
    following the clean task $\tau_c$ and the backdoor task $\tau_b$. To approximate complete backdoor unlearning, an adversarial example $\tilde{x} = x + \delta$,
    with perturbation $\delta \in S$, should be constructed such that it preserves the true label $y$ under the clean model $f^{*}_{\tau_c}$  whereas 
    the backdoor model $f^{*}_{\tau_b}$ predicts it as target label $y_t$. Formally, $\forall \tilde{x} \in X^{\tau_u}$, it holds that
    $f^{*}_{\tau_c}(\tilde{x};\theta_{\tau_c}) = y$ and $f^{*}_{\tau_b}(\tilde{x};\theta_{\tau_b}) = y_t$.
\end{proposition}
Proof. see Appendix~\ref{proof:proof_of_proposition_2}.

Proposition~\ref{proposition:proposition_2} indicates that complete backdoor unlearning requires the adversarial example $\tilde{x}$ to exhibit properties 
similar to a genuine backdoor sample. Specifically, it should display the backdoor features under the backdoor model $f^{*}_{\tau_b}$, while preserving 
the main semantic features of the original sample $x$ when evaluated by the clean model $f^{*}_{\tau_c}$. However, the objective remains practically 
challenging due to the defender's lack of access to the poisoned training dataset, as assumed in the threat model in Section~\ref{section:threat_model}.

\subsection{Blind Inversion}
\label{section:blind_inversion}
As specified in the threat model in Section~\ref{section:threat_model}, we now focus on the adversarial risk term $\mathcal{R}_{adv}$ in the optimization objective, 
as formulated in Eq.~\ref{eq:backdoor_adversarial_unlearning}.
\begin{equation}
    \label{eq:second_term_of_backdoor_adversarial_unlearning}
    \underset{\theta}{\min}~\underset{\delta \in S}{\max}~\mathbb{E}_{(x,y)\sim\mathcal{D}_{-y_t}} (\ell_{adv}(f(\tilde{x};\theta),y)).
\end{equation}
According to Proposition~\ref{proposition:proposition_2}, to approximate complete backdoor unlearning, the adversarial example distribution $q(\tilde{x})$, 
generated by the inner optimization of Eq.~\ref{eq:second_term_of_backdoor_adversarial_unlearning}, should satisfy the conditions:`` $\forall \tilde{x} \in q(\tilde{x})$, 
\ding{172}:$f^{*}_{\tau_b}(\tilde{x};\theta_{\tau_b}) = y_t$ and \ding{173}:$f^{*}_{\tau_c}(\tilde{x};\theta_{\tau_c}) = y$''.
To encode condition \ding{172} in the optimization process, it is natural to formulate it using cross-entropy loss: $\ell_{ce}(f_{\tau_{b}}(\tilde{x}),y_t)$, 
where $f_{\tau_{b}}$ is the backdoor model.

However, since condition \ding{173} depends on the clean model $f_{\tau_{c}}$, which is inherently unknown, the construction of adversarial examples can be naturally 
cast as a blind inversion problem~\cite{gao2021deepgem}: solving an inverse problem to recover an unknown source from indirect measurements, under the assumption 
that the source and measurements are related via a forward model with incomplete or unknown knowledge of parameters.
In this model-based inversion, the adversarial example $\tilde{x}$ represents the unobserved sources, and the label distribution \( y \) corresponds to the 
observed measurements. They are related through a forward clean model $y = f_{\tau_c}(\tilde{x})$ with unknown parameters.
We parameterize the unknown clean model as $f_{\theta}(x) $ and then estimate the true model parameters $\theta^{*}$ by optimizing a maximum a posteriori (MAP) 
objective, denoted as \text{MAP$_{\theta}$}.

\textbf{\text{MAP$_{\theta}$}.} This is achieved by optimizing the forward model parameters $\theta$ over the full distribution $q(\tilde{x})$ of possible adversarial 
examples $\tilde{x}$ by solving 
\label{section:MAP_theta}
\begin{equation}
    \label{eq:MAP_theta}
    \begin{split}
    \theta^{*} &= \arg \max_{\theta}[\log p(\theta|y)] \\
                 &= \arg \max_{\theta}\left\{L(q,\theta)+ KL(q(\tilde{x})||p(\tilde{x}|\theta ,y))\right\} \\
    \end{split}
\end{equation}
where
\begin{equation}
    \left\{\begin{matrix}
        L(q,\theta) = E_{q(\tilde{x})}[log(p(\theta,\tilde{x}|y))] + H(q(\tilde{x})) \\
        KL(q(\tilde{x})||p(\tilde{x}|\theta ,y)) = -\int q(\tilde{x}) \log(\frac{p(\tilde{x}|\theta ,y)}{q(\tilde{x})})d\tilde{x}.
    \end{matrix}\right.
\end{equation}
$L(q,\theta)$ denotes the variational lower bound, $H(q(\tilde{x}))$  represents the entropy of the distribution $q(\tilde{x})$, and $\text{KL}(q(\tilde{x})||p(\tilde{x}|\theta,y))$
is the Kullback–Leibler (KL) divergence between the distribution $q(\tilde{x})$ and the posterior $p(\tilde{x}|\theta,y)$.

The derivation can be found in Appendix~\ref{proof:proof_of_MAP_theta}.

Expectation-Maximization (EM) algorithms have long been employed to efficiently solve \text{MAP$_{\theta}$}~\cite{dempster1977maximum}. 
This process follows an iterative two-step procedure. For instance, at the \textit{t}-th iteration:  
(1) Expectation step (E-Step): computing the posterior of \( \tilde{x} \) given the current forward model $\theta^{(t-1)}$;
(2) Maximization step (M-Step): updating \( \theta^{(t)} \) to maximize the expected value of the log-likelihood function with respect to $\theta$.

\textbf{Expectation Step.} E-step aims to learn a distribution $q(\tilde{x})$ to approximate the posterior distribution of $\tilde{x}$, given the current forward 
model parameters $\theta^{(t-1)}$, by optimizing
\begin{equation}
    \label{eq:E-step}
\resizebox{0.90\linewidth}{!}{$
    \begin{split}
        & \min_{q(\tilde{x})}KL(q(\tilde{x})||p(\tilde{x}|\theta^{(t-1)},y)) \\
        &= \int q(\tilde{x}) \log \frac{q(\tilde{x})}{p(\tilde{x}|\theta^{(t-1)},y)} d\tilde{x} \\
        &= \mathbb{E}_{\tilde{x}\sim q(\tilde{x})}\left[- \log(p(y|\tilde{x},\theta^{(t-1)})) - \log p(\tilde{x}|\theta^{(t-1)}) + \log q(\tilde{x}) \right]
    \end{split}
$}
\end{equation}
where $\log p(y|\tilde{x},\theta^{(t-1)})$ is the data likelihood and $\log p(\tilde{x}|\theta^{(t-1)})$ is a prior on the generated adversarial examples.
We note that when $KL(q(\tilde{x})||p(\tilde{x}|\theta^{(t-1)}, y))$ reaches its minimum value of 0, it holds that $q(\tilde{x}) = p(\tilde{x}|\theta^{(t-1)},y)$.

\textbf{Maximization Step.} M-step is to use the approximate posterior distribution $q(\tilde{x})$ from the prior E-step, to update $\theta$, the parameters of 
the model, yielding $\theta^{(t)}$. This is achieved by optimizing $\max_{\theta} L(q,\theta)$, while keeping the distribution $q(\tilde{x}) = p(\tilde{x}|\theta^{(t-1)},y)$ 
fixed. We have:
\begin{equation}
    \label{eq:M-step}
    \resizebox{0.90\linewidth}{!}{$
    \begin{split}
        \max_{\theta} L(q,\theta) &= \max_{\theta} \mathbb{E}_{\tilde{x}\sim q(\tilde{x})}[\log(p(\theta,\tilde{x}|y))] + H(q(\tilde{x})) \\
        &= \max_{\theta} \mathbb{E}_{\tilde{x}\sim q(\tilde{x})}[\log(p(y|\tilde{x},\theta)) + \log p(\theta)] + \text{const}
    \end{split}
    $}
\end{equation}
where $\log p(\theta)$ is the prior over the model parameters.

\subsection{Proposed Method}
\label{section:proposed_method}
In this section, we introduce the Blind Inversion-Backdoor Adversarial Unlearning (BI-BAU) algorithm, which instantiates the adversarial training objective in Eq.~\ref{eq:backdoor_adversarial_unlearning} 
as a principled EM-style optimization procedure, grounded in Propositions~\ref{proposition:proposition_1} and~\ref{proposition:proposition_2}.

\textbf{Inner Maximization.} In the E-step, we aim to approximate the distribution $q(\tilde{x})$ over adversarial examples $\tilde{x}$ that satisfy the two conditions
outlined in Section~\ref{section:blind_inversion}. Condition \ding{172}, requiring $\tilde{x}$ to trigger the backdoor model, is naturally formulated via 
a cross-entropy loss $\ell_{ce}(f_{\tau_b}(\tilde{x}), y_t)$. Condition \ding{173}, enforcing that $\tilde{x}$ preserves the clean semantics, corresponds to 
the data likelihood term $-\log p(y|\tilde{x},\theta^{(t-1)})$ and a prior term $\log p(\tilde{x}|\theta^{(t-1)})$ that regularizes the adversarial perturbation.
We now elaborate on the two components in light of the discussion in Section~\ref{section:problem_formulation}.

To instantiate the data likelihood term $-\log p(y|\tilde{x},\theta^{(t-1)})$, we adopt a knowledge distillation loss~\cite{hinton2015distilling} that encourages 
$\tilde{x}$ to retain the primary semantic features of the original sample $x$. Unlike conventional knowledge distillation, which transfers knowledge across models 
in the parameter space, We instead leverage knowledge distillation as a mechanism to enforce semantic consistency directly in the input space, even under adversarial 
perturbations.
\begin{equation}
\label{eq:ell_KD}
\resizebox{0.90\linewidth}{!}{$
    \ell_{kd} = \alpha \cdot T^2 \cdot {KL}\big(p_{\theta}(\tilde{x}; T)||p_{\tau_{b}}(x; T)\big) + (1 - \alpha) \cdot {\ell_{ce}}(f_{\theta}(\tilde{x}),y)
$}
\end{equation}
where $\ell_{ce}(f_{\theta}(\tilde{x}),y)$ is the cross-entropy loss between the updated model $f_{\theta}$ and the ground-truth label $y$, 
$\text{KL}(p_{\theta}(\tilde{x};T) || p_{\tau_b}(x;T))$ measures the discrepancy between the softened outputs of $f_{\theta}$ on $\tilde{x}$ and the backdoor 
model $f_{\tau_b}$ on $x$, $T$ is the temperature controlling output smoothness, and $\alpha$ balances the contributions of soft-label (KL) and hard-label (CE) terms.

The prior term $\log p(\tilde{x}|\theta^{(t-1)})$ in Eq.~\ref{eq:E-step} encodes condition in Proposition~\ref{proposition:proposition_1},
\emph{``$\nabla_{\theta}f_{0}(x_u,\theta) \parallel (\theta_{\tau_{b}}^{*} - \theta_{\tau_{c}}^{*}), \forall x_{u} \in X_{\tau_u}$''}, which requires that the gradient 
of the unlearning task aligns with the backdoor direction to maximally eliminate backdoor effects. Rather than enforcing this condition directly in parameter space, 
we realize it in the feature space. We now detail how this transformation is achieved.

For $\nabla_{\theta}f_{0}(x_u,\theta)$, based on Eq.~\ref{eq:recurrence}, we have
\begin{equation}
    \frac{\partial f(x,w)}{\partial w^{l+1}} = \frac{\partial f(x,w)}{\partial h^{l+1}} \cdot \frac{\partial h^{l+1}}{\partial w^{l+1}} = \frac{\partial f(x,w)}{\partial h^{l+1}} \cdot  z^{l}.
\end{equation}
In particular, for the fully connected (FC) classification head that outputs pre-softmax logits and thus does not involve any activation function $\phi$, 
the mapping from $h^{l+1}$ to the output is linear. Consequently, the Jacobian $\partial f / \partial h^{l+1}$ reduces to a constant linear mapping.
As a result, the gradient of the output with respect to the FC-layer weights is simply
\begin{equation}
 \frac{\partial f(x,w)}{\partial w_{FC}} = z(x)
\end{equation}
where $w_{FC}$ denotes the weights of the FC layer and $z(x) = f_{\text{backbone}}(x)$ represents the output of the feature extractor.

Therefore, for a linear FC classification head, the feature representation $z(x) = f_{\text{backbone}}(x)$ provides a first-order approximation to the gradient 
with respect to the FC weights. Accordingly, we treat the feature perturbation $z(\tilde{x}) - z(x)$ as a first-order proxy for the unlearning-task gradient 
with respect to the linear classification head.
Let $w_{\tau_b}$ and $w$ denote the FC-layer weights of the backdoored model and the current model, respectively; the difference $w_{\tau_b} - w$ serves as a proxy
for the backdoor gradient direction. To encourage alignment of the adversarial perturbation with the backdoor direction, we maximize the inner product
\begin{equation}
    \max \left \langle z(\tilde{x}) - z(x), w_{\tau_{b}} - w \right \rangle.
\end{equation}

Considering that the clean and backdoor tasks are not strictly orthogonal in practice~\cite{zhang2024exploring,liu2020reflection,barni2019new,qi2023revisiting}, 
as illustrated in Figure~\ref{fig:task_similarity} (2), in such cases, the angle $\alpha$ between the backdoor-task gradient and the clean-task gradient direction
$v_c$ is typically less than $\pi/2$. The corresponding angle $\beta$ denotes the angle between the unlearning-task gradient and the clean-task gradient direction $v_c$, 
which could be approximated by the angle between the perturbation $z(\tilde{x}) - z(x)$ and the clean feature $z(x)$. This approximation allows us to balance 
the alignment with the backdoor direction while preserving the semantic integrity of clean features. Accordingly, the prior loss is defined as:
\begin{align}
    \label{eq:ell_prior}
    \ell_{prior} = \left \langle z(\tilde{x}) - z(x), w_{\tau_{b}} - w \right \rangle - \cos(\beta + k \cdot \frac{\pi}{2})
\end{align}
where $\beta \approx \arccos \left ( \frac{\left \langle z(\tilde{x}) - z(x), z(x) \right \rangle }{\left | z(\tilde{x}) - z(x) \right | \cdot \left | z(x) \right |} \right )$,
and $k \in [0,1]$ controls the magnitude of the angle $\beta$. Notably, when $k=1.0$, $\ell_{prior}$ is maximized at $\beta = \pi/2$, ensuring the unlearning task 
is orthogonal to the clean task and minimally interferes with it. Conversely, $k=0.0$ aligns $\tilde{x}$ fully with backdoor features, at the cost of clean semantics.
Thus, $k$ balances backdoor alignment and clean feature preservation. The effect of $k$ on the robustness of BI-BAU is analyzed in the Appendix~\ref{appendix:Investigation_of_hyperparameter_k}. 

Combining conditions \ding{172} and \ding{173}, we have the adversarial loss for the inner optimization as follows:
\begin{equation}
    \label{eq:ell_in}
    \mathcal{L}_{adv-in} = - \ell_{ce}(f_{\tau_{b}}(\tilde{x}),y_t) - \lambda_{1} \ell_{kd} + \lambda_{2} \ell_{prior}
\end{equation}
where $\lambda_{1}$ and $\lambda_{2}$ are trade-off weights that balance the preservation of semantic features of the original sample and the strength of 
the adversarial perturbation, respectively.

\textbf{Outer Minimization.} In the M-step, we update the model parameters $\theta$ using the adversarial example distribution $q(\tilde{x})$ from the E-step to 
purify the backdoor model. This is achieved by optimizing $\max_{\theta} L(q, \theta)$, which involves the likelihood term $\log p( y\mid \tilde{x},\theta)$ 
as well as the prior distribution term $\log p(\theta)$ in Eq.~\ref{eq:M-step}.
The likelihood term is implemented via cross-entropy, $\ell_{ce}(f_{\theta}(\tilde{x};\theta), y)$, ensuring that $\tilde{x}$ is labeled 
with its true class to fine-tune the model and improve robustness.

The prior term $\log p(\theta)$ enforces the condition of Proposition~\ref{proposition:proposition_1},
\emph{``$\nabla_{\theta}f_{0}(x_c,\theta) \perp (\theta_{\tau_b}^{*} - \theta_{\tau_c}^{*}) $''}, 
which requires that the clean task $\tau_c$ to be orthogonal to the backdoor task $\tau_b$.
Equivalently, this condition can be expressed as minimizing the Euclidean inner product $\langle \nabla_{\theta} f_{0}(x_c,\theta), \theta_{\tau_b}^{} - \theta_{\tau_c}^{} \rangle$,
which aligns with the concept of Elastic Weight Consolidation (EWC)~\cite{kirkpatrick2017overcoming}, a technique commonly used to alleviate catastrophic forgetting in neural networks. 
Since the clean model is unavailable, we replace $\theta_{\tau_c}^{*}$ with the current model parameter $\theta$, yielding
\begin{equation}
    \ell_{ewc} = (\theta_{\tau_b}^{*} - \theta) ^{T} F_{\tau_c} (\theta_{\tau_b}^{*} - \theta)
\end{equation}
where $F_{\tau_c} = \mathbb{E} \left [ \nabla_{\theta}\log p(X^{\tau_c}|\theta) \nabla_{\theta}\log p(X^{\tau_c}|\theta)^{T} \right ]$ is the Fisher Information Matrix (FIM)
for the clean dataset $X^{\tau_c}$. The term $\ell_{ewc}$ ensures that the model's performance on the clean task is further preserved.

Additionally, Proposition~\ref{proposition:proposition_1} implies that ``$\nabla_{\theta}f_{0}(x_u,\theta) \parallel (\theta_{\tau_b}^{*} - \theta_{\tau_c}^{*}) $'',  
indicating that the unlearning task $\tau_u$ should be antagonistic and parallel to the backdoor task $\tau_b$. Therefore, for unlearning task $\tau_u$,
we adopt the inverse of the EWC loss, clipped to $\left [ 0, 1.0 \right ]$, yielding:
\begin{equation}
        \ell_{inv-ewc} = 1.0 / \left [ 1.0 + (\theta_{\tau_b}^{*} - \theta) ^{T} F_{\tau_u} (\theta_{\tau_b}^{*} - \theta)  \right ] 
\end{equation}
where $F_{\tau_u} = \mathbb{E} \left [ \nabla_{\theta}\log p(X^{\tau_u}|\theta_{\tau_b}^{*}) \nabla_{\theta}\log p(X^{\tau_u}|\theta_{\tau_b}^{*})^{T} \right ]$ is 
the Fisher Information Matrix (FIM) computed on the unlearning dataset $X^{\tau_u}$ under the backdoored model $f_{\tau_b}^{*}(.;\theta_{\tau_b}^{*})$. This design 
encourages parameter updates to follow the backdoor gradient direction, thereby promoting effective forgetting of backdoor-related features.

Therefore, the outer optimization loss is:
\begin{equation}
    \label{eq:ell_out}
    \mathcal{L}_{adv-out} = \ell_{ce}(f_{\theta}(\tilde{x};\theta), y) + \gamma_{1} \ell_{inv-ewc} + \gamma_{2} \ell_{ewc}
\end{equation}
where $\gamma_{1}$ and $\gamma_{2}$ are trade-off weights that balance effective backdoor unlearning and the preservation of clean-task performance, respectively.

\textbf{Overall objective.} In summary, combining Equations~\ref{eq:backdoor_adversarial_unlearning}, ~\ref{eq:ell_in} and~\ref{eq:ell_out}, we propose the bi-level optimization objective:
\begin{equation}
    \label{eq:overall_optimization_objective_for_targeted_adversarial}
    \resizebox{0.95\linewidth}{!}{$
    \begin{split}
    & \min_{\theta}~\mathbb{E}_{(x,y)\sim \mathcal{D}_{cl}} \left [ \ell(x,y;\theta) \right ] 
     + \mathbb{E}_{(x,y)\sim \mathcal{D}_{-y_t}} \left [\mathcal{L}_{adv-out}(\tilde{x},y;\theta)  \right ] \\ 
    & s.t.~~~ \delta^{*} = \arg~\max_{\delta \in S}~\mathcal{L}_{adv-in}(\tilde{x},y_t, y;\theta).
    \end{split}
    $}
\end{equation}

\textbf{Extension to Untargeted Adversarial Training.} In the case where the target label $y_t$ is unknown, for the inner optimization problem, 
we adopt untargeted adversarial examples. Consequently, we define the adversarial loss as $\mathcal{L}_{adv-in} = \ell_{ce}(f_{\tau_{b}}(\tilde{x}),y) -\lambda_{1} \ell_{kd} + \lambda_{2} \ell_{prior}$
and formulate the overall optimization objective as:
\begin{equation}
    \label{eq:overall_optimization_objective_for_untargeted_adversarial}
    \resizebox{0.85\linewidth}{!}{$
    \begin{split}
    & \min_{\theta}~\mathbb{E}_{(x,y)\sim \mathcal{D}_{cl}} \left [ \ell(x,y;\theta) + \mathcal{L}_{adv-out}(\tilde{x},y;\theta) \right ] \\
    & s.t.~~~ \delta^{*} = \arg~\max_{\delta \in S}~\mathcal{L}_{adv-in}(\tilde{x}, y;\theta).
    \end{split}
    $}
\end{equation}

\textbf{Extension to Multi-Modal Contrastive Learning.} 
The extension of BI-BAU to multi-modal contrastive learning settings is presented in Appendix~\ref{appendix:extending_BI-BAU_to_multi_modal_contrastive_learning}.

%% file: section/Experiments.tex
\section{Experiments}
\label{section:experiments}

\subsection{Experiment Settings}
\label{subsection:experiment_settings}

We evaluate our proposed BI-BAU method under two learning scenarios: \emph{supervised learning} and \emph{multi-modal contrastive learning}, considering 
diverse datasets, model architectures, and attack types. Detailed descriptions of the datasets and network architectures are provided in Appendix~\ref{subsection:datasets_and_network}.  

\textbf{Supervised learning.} For standard image classification tasks, we conduct all attack and defense tasks on three benchmark datasets: CIFAR-10~\cite{krizhevsky2009learning}, 
Tiny-ImageNet-200~\cite{wu2017tiny}, and GTSRB~\cite{stallkamp2011german}, using three neural network architectures: ResNet-18~\cite{he2016deep}, VGG19~\cite{simonyan2014very} 
and Vision Transformer (ViT)~\cite{dosovitskiy2020image}. Additional experiments on diverse datasets and models, including Tiny-ImageNet-200 and VGG19, are reported in Appendix~\ref{appendix:experiments_on_diverse_datasets_and_models}.

\textbf{Multi-Modal Contrastive Learning.} For multi-modal evaluation, we adopt OpenAI's open-source CLIP~\cite{radford2021learning} as the pre-trained backbone, 
trained on 400 million image-text pairs. Specifically, the ``ViT-B/32'' variant serves as the visual encoder and a Transformer as the text encoder. We first 
implant a backdoor into a pre-trained CLIP model on ImageNet-1K. Then, we construct a fine-tuning dataset by sampling 50,000 image-text pairs from 
Conceptual Captions (CC3M)~\cite{sharma2018conceptual}, corresponding to only ~0.6\% of the full dataset.

\textbf{Attack Baselines.} For supervised learning, we conduct an extensive investigation and adopt a wide range of representative backdoor attacks to evaluate 
our approach, categorized as follows: (1) Patch-based backdoor attacks: BadNets~\cite{gu2019badnets} and Blended~\cite{chen2017targeted}; (2) Invisible backdoor attacks:
WaNet~\cite{nguyen2021wanet} and Refool~\cite{liu2020reflection}; (3) Sample-specific backdoor attacks: IAD~\cite{nguyen2020input} and ISSBA~\cite{li2021invisible};
(4) Clean-label attacks: SIG~\cite{barni2019new} and Narcissus~\cite{zeng2023narcissus}. We further include adaptive attacks~\cite{qi2023revisiting} 
designed to circumvent latent separation-based defenses, such as Adap-Blend and Adap-Patch. For multi-modal contrastive learning, we evaluate attack baselines, 
include BadNet~\cite{gu2019badnets} and the state-of-the-art (SoTA) BadCLIP~\cite{liang2024badclip}. BadCLIP is specifically designed for multi-modal model and exhibits
strong resilience against existing defense mechanisms.

Moreover, we implement the backdoor re-activation attacks proposed by Min \etal~\cite{min2024uncovering}, including Query-based Re-activation Attack (QRA) 
and Retuning Attack (RA), to verify that our proposed BI-BAU method can effectively and thoroughly eliminate backdoor effects in backdoor models. Detailed attack 
configurations are provided in the Appendix~\ref{subsection:attack_settings}.

\textbf{Defense Baselines.} For supervised learning, we evaluate the effectiveness of our proposed BI-BAU method by comparing it with state-of-the-art post-purification 
approaches, including: (1) trigger synthesis based approaches, such as NC~\cite{wang2019neural} and BTI-DBF~\cite{xu2024towards}; (2) model reconstruction based 
approaches, including ANP~\cite{wu2021adversarial}, EP~\cite{zheng2022pre}, RNP~\cite{li2023reconstructive} and CLP~\cite{zheng2022data}; (3) backdoor adversarial 
unlearning approaches, such as I-BAU~\cite{zeng2021adversarial} and SAU~\cite{wei2023shared}. We further compare with the recently proposed PAM~\cite{min2024uncovering} 
by Min~\etal to highlight the its advantages. For multi-modal contrastive learning, we adopt CleanCLIP~\cite{bansal2023cleanclip} as a baseline defense for comparison.  
In addition, as a comparative reference, we construct an idealized backdoor unlearning scenario (termed IBU), in which we assume that the defender has access to 
the real backdoor samples along with a small subset of clean data, and performs fine-tuning to remove the backdoor from the model.
Further defense settings are provided in the Appendix~\ref{subsection:defense_settings}. 

\textbf{Evaluation Metrics.} By default, all defense methods have access to $5\%$ of clean training data, with evaluation metrics including Clean Accuracy (CA) 
and Attack Success Rate (ASR), where R-ASR and Q-ASR represent ASR after performing Retuning Attack (RA) and Query-based Re-activation Attack (QRA), respectively.
We further verify the thoroughness of backdoor removal by measuring the Backdoor Existence Coefficient (BEC)~\cite{zhu2024breaking} at the neural level.
The implementation details of BEC are provided in Appendix~\ref{appendix:layer-wise_activation_analysis}.~\footnote{In the following tables, we use \scalebox{0.75}{[\colorbox[HTML]{FFCCCC}{red}]} 
to mark some cases where the defense is ineffective, as indicated by a persistently high ASR, and \scalebox{0.75}{[\colorbox[HTML]{d9f2d9}{green}]} to highlight 
the best-performing result among the six baseline methods. ``BI-BAU (T)'' and ``BI-BAU (U)'' refer to the defense performance under targeted and untargeted attack 
settings, respectively.} In addition, a comparison of the efficiency of the defense methods is presented in Section~\ref{subsection:efficiency_and_computational_cost}.

\begin{table*}[!ht]
	\centering
	\caption{Performance (\%) of safety tuning strategies for eliminating backdoor effects on ResNet-18 (CIFAR-10).}
	\label{table:backdoor_unlearning_on_CIFAR-10}
	\renewcommand{\arraystretch}{1.5}
	\resizebox{0.98\textwidth}{!}{\begin{tabular}{c|cc|cc|cc|cc|cc|cc|cc|cc|cc}
		\toprule
		Attack $\to$ & \multicolumn{2}{c|}{BadNets} & \multicolumn{2}{c|}{Blended} & \multicolumn{2}{c|}{WaNet} & \multicolumn{2}{c|}{Refool} & \multicolumn{2}{c|}{IAD} & \multicolumn{2}{c|}{ISSBA} & \multicolumn{2}{c|}{SIG} & \multicolumn{2}{c|}{Narcissus} & \multicolumn{2}{c}{Average}\\ 
		\midrule
		Defense $\downarrow$ & \multicolumn{1}{c}{CA}&\multicolumn{1}{c|}{ASR} & \multicolumn{1}{c}{CA}&\multicolumn{1}{c|}{ASR} & \multicolumn{1}{c}{CA}&\multicolumn{1}{c|}{ASR} & \multicolumn{1}{c}{CA}&\multicolumn{1}{c|}{ASR} & \multicolumn{1}{c}{CA}&\multicolumn{1}{c|}{ASR} & \multicolumn{1}{c}{CA}&\multicolumn{1}{c|}{ASR} & \multicolumn{1}{c}{CA}&\multicolumn{1}{c|}{ASR} & \multicolumn{1}{c}{CA}&\multicolumn{1}{c|}{ASR} & \multicolumn{1}{c}{CA}&\multicolumn{1}{c}{ASR} \\ 
		\midrule
		
        No Defense  & 90.00&99.70 & 90.98&100.00 & 87.28&99.90 & 91.25&99.30 & 91.96&99.90 & 88.18&99.10 & 88.81&99.44 & 91.18&99.78 & 89.91&99.64\\
        \hline
		
		NC & 89.70&1.40 & \scalebox{1.00}{\colorbox[HTML]{d9f2d9}{90.62}}&\scalebox{1.00}{\colorbox[HTML]{d9f2d9}{0.00}} & 85.60&\scalebox{1.0}{\colorbox[HTML]{FFCCCC}{11.36}} & \scalebox{1.00}{\colorbox[HTML]{d9f2d9}{90.02}}&\scalebox{1.0}{\colorbox[HTML]{FFCCCC}{18.30}} & \scalebox{1.00}{\colorbox[HTML]{d9f2d9}{91.42}}&2.70 & 88.05&10.30 & \scalebox{1.00}{\colorbox[HTML]{d9f2d9}{88.10}}& \scalebox{1.0}{\colorbox[HTML]{FFCCCC}{99.57}} & \scalebox{1.00}{\colorbox[HTML]{d9f2d9}{89.90}}&\scalebox{1.0}{\colorbox[HTML]{FFCCCC}{80.88}} & \scalebox{1.00}{\colorbox[HTML]{d9f2d9}{89.17}}&28.06\\
        \hline
		BTI-DBF & 88.94&1.60 & 83.57&5.60 & 83.86&8.88 & 85.10&\scalebox{1.0}{\colorbox[HTML]{FFCCCC}{20.60}} & 86.95&3.80 & 82.51&9.50 & 84.30&\scalebox{1.0}{\colorbox[HTML]{FFCCCC}{83.26}} & 85.78&\scalebox{1.00}{\colorbox[HTML]{d9f2d9}{1.44}} &  85.12&16.83 \\
		\hline
		ANP     & 85.71&\scalebox{1.00}{\colorbox[HTML]{d9f2d9}{0.20}} & 85.05&4.30 & 80.22&\scalebox{1.00}{\colorbox[HTML]{d9f2d9}{0.40}} & 88.47&\scalebox{1.00}{\colorbox[HTML]{d9f2d9}{0.70}} & 88.06&4.20 & 88.04&3.20 & 83.87&\scalebox{1.0}{\colorbox[HTML]{FFCCCC}{11.95}}  & 87.10&\scalebox{1.0}{\colorbox[HTML]{FFCCCC}{97.52}} & 85.81&15.30    \\
        \hline
        EP  & 88.97&0.90 & 90.07&\scalebox{1.0}{\colorbox[HTML]{FFCCCC}{54.50}}& 86.52&\scalebox{1.0}{\colorbox[HTML]{FFCCCC}{15.90}}  & 87.90&\scalebox{1.0}{\colorbox[HTML]{FFCCCC}{99.30}} & 86.27&\scalebox{1.0}{\colorbox[HTML]{FFCCCC}{42.50}} & 87.32&5.30 & 87.50&\scalebox{1.0}{\colorbox[HTML]{FFCCCC}{93.88}} & 89.68&\scalebox{1.0}{\colorbox[HTML]{FFCCCC}{99.98}} &  88.02&51.53  \\
		\hline
		RNP & \scalebox{1.00}{\colorbox[HTML]{d9f2d9}{90.22}}&0.66 & 84.88&3.80 & \scalebox{1.00}{\colorbox[HTML]{d9f2d9}{87.72}}&3.10 & 89.46&4.04 & 88.86&3.04 & \scalebox{1.00}{\colorbox[HTML]{d9f2d9}{90.67}}&0.40 & 81.81&\scalebox{1.0}{\colorbox[HTML]{FFCCCC}{71.91}} & 79.54&\scalebox{1.0}{\colorbox[HTML]{FFCCCC}{99.48}} & 86.64&23.30    \\
        \hline
        CLP  & 77.28&6.14 & 83.96&\scalebox{1.0}{\colorbox[HTML]{FFCCCC}{92.10}} & 59.66&\scalebox{1.0}{\colorbox[HTML]{FFCCCC}{61.90}} & 24.28&\scalebox{1.0}{\colorbox[HTML]{FFCCCC}{100.00}} & 83.66&\scalebox{1.0}{\colorbox[HTML]{FFCCCC}{27.60}} & 83.86&0.80 & 77.05&\scalebox{1.0}{\colorbox[HTML]{FFCCCC}{99.66}} & 73.08&\scalebox{1.0}{\colorbox[HTML]{FFCCCC}{52.84}} &  70.35&55.13	\\
		\hline
		I-BAU & 89.86&1.50 & 89.96&\scalebox{1.0}{\colorbox[HTML]{FFCCCC}{46.40}}& 82.46&2.80 & 85.06&\scalebox{1.0}{\colorbox[HTML]{FFCCCC}{12.30}} & 88.18&4.30 & 84.24&9.60 & 85.81&\scalebox{1.0}{\colorbox[HTML]{FFCCCC}{93.62}} & 88.68&\scalebox{1.0}{\colorbox[HTML]{FFCCCC}{99.68}} &  86.78&33.77\\
		\hline
		SAU   & 84.83&2.40 & 85.52&8.00  & 80.08&1.82 & 85.31&3.10 & 85.75&6.80 & 80.68&\scalebox{1.0}{\colorbox[HTML]{FFCCCC}{12.10}} & 82.56&\scalebox{1.0}{\colorbox[HTML]{FFCCCC}{35.22}} & 85.86&10.30 &  83.82 & 9.96\\
		\hline
        BI-BAU (T) & 85.84&1.90 &  87.42&3.80  & 86.33&2.20 & 87.41&3.80  &  88.28&3.10  &  85.28&2.30 & 80.36&6.26 & 83.66&\scalebox{1.0}{\colorbox[HTML]{FFCCCC}{15.08}} & 85.57&4.80\\
		\hline
        BI-BAU (U) & 88.06&1.10 & 88.27&0.20 & 86.63&1.20 & 87.24&4.10 & 88.73&\scalebox{1.00}{\colorbox[HTML]{d9f2d9}{0.60}} & 85.33&\scalebox{1.00}{\colorbox[HTML]{d9f2d9}{1.40}} & 78.29&\scalebox{1.00}{\colorbox[HTML]{d9f2d9}{0.51}} & 82.30&2.24 & 85.60 & \scalebox{1.00}{\colorbox[HTML]{d9f2d9}{1.41}}\\
        \bottomrule
	\end{tabular}}
	\vspace{-3mm}
\end{table*} 

\subsection{Resistance to Low-Orthogonality Attacks}
\label{subsection:main_results}

To validate the effectiveness of the BI-BAU method, we summarize the experimental results on ResNet-18 in Table~\ref{table:backdoor_unlearning_on_CIFAR-10}.
As shown in the Table~\ref{table:backdoor_unlearning_on_CIFAR-10}, BI-BAU is the only method that consistently achieves a low ASR while maintaining a relatively 
high CA across all backdoor attacks. Furthermore, we observe that, apart from BI-BAU, all other methods largely fail against state-of-the-art clean-label attacks 
such as SIG and Narcissus. This can be attributed to the entanglement of clean and backdoor features caused by such attacks, which has been shown to make 
defense significantly more challenging~\cite{zeng2023narcissus, qi2023revisiting}. Similarly, most existing defenses also perform poorly against Refool, consistent 
with the findings of~\cite{zhang2024exploring}, where they demonstrated that existing defenses are particularly vulnerable to attacks characterized by low orthogonality, 
as exhibited by Refool, SIG and Narcissus. Although BI-BAU does not always achieve the highest CA, with an average drop of 4.31\% on CIFAR-10, this performance 
trade-off is both expected and justifiable. As discussed in Section~\ref{section:proposed_method}, the alignment between the unlearning objective and the 
backdoor objective inevitably interferes with the clean task, particularly when the attack exhibits low orthogonality.
To this end, BI-BAU is deliberately designed to prioritize robust backdoor removal, even at the expense of a moderate decrease in CA.

Specifically, among trigger synthesis based approaches, BTI-DBF demonstrates performance most comparable to our BI-BAU. These approaches attempt to synthesize 
the trigger patterns through model inversion, followed by unlearning. However, accurately recovering the ground-truth triggers remains challenging. BTI-DBF 
reconstructs highly similar triggers by decoupling benign features, but this strategy is not always effective against backdoor attacks where clean and backdoor 
features are highly entangled. For example, BTI-DBF fails against the Narcissus attack, with ASR as high as 80.88\%. Moreover, such methods typically 
incur high computational costs due to the need for training additional models, as demonstrated by the efficiency experiments in Section~\ref{subsection:efficiency_and_computational_cost}.

Moreover, model reconstruction-based methods, such as ANP, EP, RNP, and CLP, directly locate and remove hidden backdoors in the backdoor model by pruning suspicious neurons. 
These methods rely on the assumption that backdoor-related neurons exhibit significantly different sensitivities to clean and poisoned samples~\cite{wu2021adversarial}. 
However, reliably identifying backdoor-related parameters is not trivial, especially when there is no clear distinction between backdoor and clean neurons, as is 
the case with attacks such as Refool, SIG, and Narcissus. This ambiguity makes it difficult to balance model robustness and performance. For instance, CLP identifies 
channels with higher lipschitz constants as potential backdoor channels and prunes them in a data-free manner. As shown in Table~\ref{table:backdoor_unlearning_on_CIFAR-10}, 
for the Refool attack, CLP achieves only 24.28\% clean accuracy due to its inability to precisely distinguish between backdoor and clean-task neurons, a finding 
consistent with that reported by Lin~\etal~\cite{lin2024unveiling}. Furthermore, for low-orthogonality attacks, including SIG and Narcissus, such methods 
consistently perform poorly, exhibiting high ASRs. As a result, these approaches achieve only limited robustness in the absence of prior knowledge of the trigger 
pattern. Existing backdoor adversarial unlearning methods mainly focus on mitigating CA degradation from adversarial training (AT), but overlook the fundamental 
mechanisms of backdoor attacks and defenses. This limitation undermines their generalizability across diverse attack scenarios. For example, I-BAU demonstrates 
poor performance against most backdoor attacks, while SAU is largely ineffective against SIG. Nevertheless, our BI-BAU consistently achieves generalizable defense 
across various attacks and datasets, under both targeted and untargeted AT settings.

\subsection{Resistance to Backdoor Re-activation Attack}
\label{subsection:resistance_to_RA}

\renewcommand{\arraystretch}{1.5}
\begin{table*}[!ht]
	\centering
	\caption{Post-purification robustness (\%) of safety tuning methods under backdoor re-activation attacks(ResNet-18, CIFAR-10).}
	\label{table:post_purification_robustness_on_CIFAR-10}
	\resizebox{0.98\textwidth}{!}{\begin{tabular}{c|c|c|c|c|c|c|c|c|c|c|c}
		\toprule
		Attack $\to$ & \multicolumn{1}{c|}{BadNets} & \multicolumn{1}{c|}{Blended} & \multicolumn{1}{c|}{WaNet} & \multicolumn{1}{c|}{Refool} & \multicolumn{1}{c|}{IAD} & \multicolumn{1}{c|}{ISSBA} & \multicolumn{1}{c|}{SIG} & \multicolumn{1}{c|}{Narcissus} &  \multicolumn{1}{c|}{Adap-Blend} & \multicolumn{1}{c|}{Adap-Patch} & \multicolumn{1}{c}{Average}\\ 
		\midrule
		\multirow{2}{*}{Defense $\downarrow$} & CA/ASR & CA/ASR & CA/ASR & CA/ASR & CA/ASR & CA/ASR & CA/ASR & CA/ASR & CA/ASR & CA/ASR & CA/ASR \\
		
		                                     & R-ASR/Q-ASR & R-ASR/Q-ASR & R-ASR/Q-ASR & R-ASR/Q-ASR & R-ASR/Q-ASR & R-ASR/Q-ASR  & R-ASR/Q-ASR  & R-ASR/Q-ASR  & R-ASR/Q-ASR  & R-ASR/Q-ASR  & R-ASR/Q-ASR  \\
		\midrule
		\multirow{2}{*}{No Defense}     & 90.00/99.70 & 90.98/100.00 & 89.54/99.90 & 91.25/99.30 & 91.96/99.90 & 88.18/99.10 & 88.81/99.44 & 91.18/99.78  & 90.64/84.90 & 90.73/88.00 & 90.32/97.00   \\
		
								 		& -/-  & -/- & -/-  & -/- & -/- & -/- & -/-  & -/- & -/-  &-/- & -/-  \\   
		\hline
		\multirow{2}{*}{IBU}     & 87.95/0.10 & 88.42/0.00 & 87.57/0.30 & 91.00/2.60 & 87.05/1.00 & 85.03/0.00 & 85.89/0.00 & 86.28/0.48  & 86.96/0.00 & 87.35/0.00 & 87.35/0.44   \\
		
								 & 0.00/7.10  & 0.00/0.00 & 0.10/0.10  & 70.80/4.30 & 0.80/1.40  & 26.10/1.50 & 0.00/0.11  & 2.60/0.28   & 0.00/0.00  & 0.00/0.00 & 10.04/1.47  \\   
		\bottomrule
		\toprule

		\multirow{2}{*}{BTI-DBF} & 86.64/5.40 & \scalebox{1.00}{\colorbox[HTML]{d9f2d9}{88.87}}/82.60 & 85.67/8.00 & 85.10/20.60 & 86.95/\scalebox{1.00}{\colorbox[HTML]{d9f2d9}{3.80}} & 82.61/1.80 & \scalebox{1.00}{\colorbox[HTML]{d9f2d9}{84.27}}/83.26 & 85.78/1.44  & 85.46/30.40 & \scalebox{1.00}{\colorbox[HTML]{d9f2d9}{87.75}}/16.80 & 85.91/25.49   \\
		
		 					     & \scalebox{1.00}{\colorbox[HTML]{FFCCCC}{88.60}}/61.40  & 96.50/76.90 & 97.70/79.30  & \scalebox{1.00}{\colorbox[HTML]{FFCCCC}{88.70}}/\scalebox{1.00}{\colorbox[HTML]{FFCCCC}{81.40}} & 29.10/11.70  & 2.30/2.30 & \scalebox{1.00}{\colorbox[HTML]{FFCCCC}{92.86}}/\scalebox{1.00}{\colorbox[HTML]{FFCCCC}{69.76}} & \scalebox{1.00}{\colorbox[HTML]{FFCCCC}{95.42}}/\scalebox{1.00}{\colorbox[HTML]{FFCCCC}{92.16}} & \scalebox{1.00}{\colorbox[HTML]{FFCCCC}{61.20}}/\scalebox{1.00}{\colorbox[HTML]{FFCCCC}{61.40}} & \scalebox{1.00}{\colorbox[HTML]{FFCCCC}{31.00}}/\scalebox{1.00}{\colorbox[HTML]{FFCCCC}{25.60}} & \scalebox{1.00}{\colorbox[HTML]{FFCCCC}{68.33}}/\scalebox{1.00}{\colorbox[HTML]{FFCCCC}{56.19}} \\ 			 			 
		\hline
		\multirow{2}{*}{ANP}     & 86.25/0.60 & 85.04/4.30 & 85.93/\scalebox{1.00}{\colorbox[HTML]{d9f2d9}{0.60}} & 88.47/\scalebox{1.00}{\colorbox[HTML]{d9f2d9}{0.70}} & 88.06/4.20 & 88.12/0.70 & 83.87/11.95 & \scalebox{1.00}{\colorbox[HTML]{d9f2d9}{87.10}}/97.52 & \scalebox{1.00}{\colorbox[HTML]{d9f2d9}{87.27}}/17.30 & 83.32/32.90 & \scalebox{1.00}{\colorbox[HTML]{d9f2d9}{86.34}}/17.07   \\

								 & 3.60/\scalebox{1.00}{\colorbox[HTML]{d9f2d9}{2.40}} & 27.50/37.30 & \scalebox{1.00}{\colorbox[HTML]{d9f2d9}{6.90}}/2.50  & \scalebox{1.00}{\colorbox[HTML]{FFCCCC}{75.00}}/\scalebox{1.00}{\colorbox[HTML]{FFCCCC}{71.00}} & 51.60/20.60  & 44.90/37.80 & \scalebox{1.00}{\colorbox[HTML]{FFCCCC}{71.17}}/\scalebox{1.00}{\colorbox[HTML]{FFCCCC}{57.40}} & \scalebox{1.00}{\colorbox[HTML]{FFCCCC}{99.34}}/\scalebox{1.00}{\colorbox[HTML]{FFCCCC}{98.62}} & \scalebox{1.00}{\colorbox[HTML]{FFCCCC}{62.60}}/\scalebox{1.00}{\colorbox[HTML]{FFCCCC}{61.50}} & \scalebox{1.00}{\colorbox[HTML]{FFCCCC}{70.90}}/\scalebox{1.00}{\colorbox[HTML]{FFCCCC}{60.80}} & \scalebox{1.00}{\colorbox[HTML]{FFCCCC}{51.35}}/\scalebox{1.00}{\colorbox[HTML]{FFCCCC}{44.99}} \\
		\hline
		\multirow{2}{*}{\hl{RNP}}     & \scalebox{1.00}{\colorbox[HTML]{d9f2d9}{90.22}}/0.66 & 84.88/3.80 & \scalebox{1.00}{\colorbox[HTML]{d9f2d9}{87.72}}/3.10 & \scalebox{1.00}{\colorbox[HTML]{d9f2d9}{89.46}}/4.04 & \scalebox{1.00}{\colorbox[HTML]{d9f2d9}{88.86}}/\scalebox{1.00}{\colorbox[HTML]{d9f2d9}{3.04}} & \scalebox{1.00}{\colorbox[HTML]{d9f2d9}{90.67}}/\scalebox{1.00}{\colorbox[HTML]{d9f2d9}{0.40}} & 81.81/71.91 & 79.54/99.48 & 77.38/12.30 & 75.42/73.50 & 84.59/27.22   \\

								& \scalebox{1.00}{\colorbox[HTML]{d9f2d9}{1.66}}/\scalebox{1.00}{\colorbox[HTML]{d9f2d9}{0.94}} & 21.60/12.00 & 20.70/13.50 & 26.98/19.68 & 45.84/16.44  & \scalebox{1.00}{\colorbox[HTML]{FFCCCC}{92.50}}/\scalebox{1.00}{\colorbox[HTML]{FFCCCC}{76.80}} & \scalebox{1.00}{\colorbox[HTML]{FFCCCC}{76.46}}/\scalebox{1.00}{\colorbox[HTML]{FFCCCC}{69.47}} & \scalebox{1.00}{\colorbox[HTML]{FFCCCC}{99.84}}/\scalebox{1.00}{\colorbox[HTML]{FFCCCC}{98.02}} & \scalebox{1.00}{\colorbox[HTML]{FFCCCC}{54.50}}/\scalebox{1.00}{\colorbox[HTML]{FFCCCC}{53.70}} & \scalebox{1.00}{\colorbox[HTML]{FFCCCC}{89.80}}/\scalebox{1.00}{\colorbox[HTML]{FFCCCC}{87.90}} & \scalebox{1.00}{\colorbox[HTML]{FFCCCC}{52.98}}/\scalebox{1.00}{\colorbox[HTML]{FFCCCC}{44.84}} \\
		\hline
		\multirow{2}{*}{SAU}     & 85.21/3.00 & 85.52/8.00 & 83.71/1.80 & 85.31/3.10 & 85.75/6.80 & 80.81/2.70 & 82.56/35.22 & 85.86/10.30  & 84.10/5.50 & 85.42/31.90 & 84.42/10.83  \\
								 & 96.90/42.10  & 45.20/24.50 & 96.10/51.10 & \scalebox{1.00}{\colorbox[HTML]{FFCCCC}{88.50}}/\scalebox{1.00}{\colorbox[HTML]{FFCCCC}{47.30}} & 68.50/17.60  & 73.50/25.40 & \scalebox{1.00}{\colorbox[HTML]{FFCCCC}{88.06}}/\scalebox{1.00}{\colorbox[HTML]{FFCCCC}{68.44}} & \scalebox{1.00}{\colorbox[HTML]{FFCCCC}{93.78}}/\scalebox{1.00}{\colorbox[HTML]{FFCCCC}{71.10}} & \scalebox{1.00}{\colorbox[HTML]{FFCCCC}{72.90}}/\scalebox{1.00}{\colorbox[HTML]{FFCCCC}{69.40}} & \scalebox{1.00}{\colorbox[HTML]{FFCCCC}{53.90}}/\scalebox{1.00}{\colorbox[HTML]{FFCCCC}{40.50}} & \scalebox{1.00}{\colorbox[HTML]{FFCCCC}{77.73}} /\scalebox{1.00}{\colorbox[HTML]{FFCCCC}{45.74}} \\   
		\hline
		\multirow{2}{*}{PAM}     & 85.83/\scalebox{1.00}{\colorbox[HTML]{d9f2d9}{0.00}} & 88.10/0.40 & 80.10/20.30 & 82.68/6.70 & 77.61/8.30 & 78.10/5.80 & 80.89/4.04 & 86.42/\scalebox{1.00}{\colorbox[HTML]{d9f2d9}{0.00}} & 85.36/54.70 & 86.40/28.40 & 83.14/12.86  \\
		
								 & 58.10/11.00 & 20.00/\scalebox{1.00}{\colorbox[HTML]{d9f2d9}{0.20}} & 19.40/16.60  & \scalebox{1.00}{\colorbox[HTML]{FFCCCC}{82.50}}/\scalebox{1.00}{\colorbox[HTML]{FFCCCC}{68.20}} & 9.60/6.50 & \scalebox{1.00}{\colorbox[HTML]{d9f2d9}{1.90}}/2.10 & \scalebox{1.00}{\colorbox[HTML]{FFCCCC}{69.62}}/\scalebox{1.00}{\colorbox[HTML]{FFCCCC}{34.16}} & \scalebox{1.00}{\colorbox[HTML]{FFCCCC}{72.66}}/\scalebox{1.00}{\colorbox[HTML]{FFCCCC}{57.94}} & \scalebox{1.00}{\colorbox[HTML]{FFCCCC}{72.50}}/\scalebox{1.00}{\colorbox[HTML]{FFCCCC}{72.20}} & \scalebox{1.00}{\colorbox[HTML]{FFCCCC}{45.60}}/\scalebox{1.00}{\colorbox[HTML]{FFCCCC}{38.30}} & \scalebox{1.00}{\colorbox[HTML]{FFCCCC}{45.18}}/\scalebox{1.00}{\colorbox[HTML]{FFCCCC}{30.72}} \\ 
		\hline
		\multirow{2}{*}{BI-BAU(T)}     & 83.81/2.40 & 84.70/1.60 & 85.48/2.20 & 86.42/4.30 & 86.70/5.10 & 84.61/1.60 & 77.22/1.60 & 77.22/1.60  &  84.54/2.90  & 85.16/16.40 & 83.57/3.97   \\
		
								       & 3.70/2.70 & 8.90/3.60 & 12.30/4.20 & 11.80/3.90 & 13.30/8.30 & 2.80/\scalebox{1.00}{\colorbox[HTML]{d9f2d9}{1.50}} & \scalebox{1.00}{\colorbox[HTML]{FFCCCC}{64.24}}/14.47 & 23.38/\scalebox{1.00}{\colorbox[HTML]{d9f2d9}{6.70}} & 6.10/3.00 & 20.90/17.00 & 16.74/6.53  \\   
		\hline
		\multirow{2}{*}{BI-BAU(U)}     & 83.74/2.70 & 86.78/\scalebox{1.00}{\colorbox[HTML]{d9f2d9}{0.10}} & 86.05/1.30 & 86.38/1.40 & 86.25/3.80 & 83.74/2.10 & 82.29/\scalebox{1.00}{\colorbox[HTML]{d9f2d9}{0.46}} & 83.34/9.86 & 82.84/\scalebox{1.00}{\colorbox[HTML]{d9f2d9}{2.50}} & 80.45/\scalebox{1.00}{\colorbox[HTML]{d9f2d9}{14.90}} & 84.18/\scalebox{1.00}{\colorbox[HTML]{d9f2d9}{3.91}} \\
		 					           
									& 3.60/2.80  & \scalebox{1.00}{\colorbox[HTML]{d9f2d9}{4.30}}/4.70  & 12.80/\scalebox{1.00}{\colorbox[HTML]{d9f2d9}{0.10}} & \scalebox{1.00}{\colorbox[HTML]{d9f2d9}{11.60}}/\scalebox{1.00}{\colorbox[HTML]{d9f2d9}{2.90}} & \scalebox{1.00}{\colorbox[HTML]{d9f2d9}{7.00}}/\scalebox{1.00}{\colorbox[HTML]{d9f2d9}{4.60}}  & 7.00/14.80 & \scalebox{1.00}{\colorbox[HTML]{d9f2d9}{33.77}}/\scalebox{1.00}{\colorbox[HTML]{d9f2d9}{2.40}} & \scalebox{1.00}{\colorbox[HTML]{d9f2d9}{20.92}}/8.38 & \scalebox{1.00}{\colorbox[HTML]{d9f2d9}{4.90}}/\scalebox{1.00}{\colorbox[HTML]{d9f2d9}{2.70}} & \scalebox{1.00}{\colorbox[HTML]{d9f2d9}{19.00}}/\scalebox{1.00}{\colorbox[HTML]{d9f2d9}{14.20}} & \scalebox{1.00}{\colorbox[HTML]{d9f2d9}{12.48}}/\scalebox{1.00}{\colorbox[HTML]{d9f2d9}{5.75}} \\   
        \bottomrule
	\end{tabular}}
	\vspace{-3mm}
\end{table*} 

To evaluate the post-purification robustness of our BI-BAU method, we adopt the idealized backdoor unlearning (IBU) methods as a reference and make a comparison 
with state-of-the-art safety tuning strategies, including BTI-DBF, ANP, RNP, and SAU. The experimental results are presented in Table~\ref{table:post_purification_robustness_on_CIFAR-10}.

As shown in Table~\ref{table:post_purification_robustness_on_CIFAR-10}, since BI-BAU(U) consistently outperforms BI-BAU(T), we focus on BI-BAU(U) as a representative example.
Our BI-BAU method demonstrates strong post-purification robustness, and is the only defense approach that closely approximates the performance of the idealized 
backdoor unlearning (IBU) method. Compared to IBU, BI-BAU achieves a competitive post-purification robustness, with an average R-ASR of 12.48\% (vs. 10.04\%) 
and Q-ASR of 5.75\% (vs. 1.47\%). These results confirm that BI-BAU effectively and nearly thoroughly eliminates backdoor effects, making the purified model 
resistant to backdoor re-activation attacks.

In addition, we observe that, except for PAM, the other methods fail to exhibit post-purification robustness, confirming that they only offer superficial safety 
without truly eliminating the backdoor effects in depth. PAM is the defense method closest to our BI-BAU in terms of performance, achieving strong post-purification 
robustness against some conventional backdoor attacks, such as BadNets, WaNet, and IAD. However, for those backdoor attacks characterized by low orthogonality or 
low linearity, such as Refool, SIG, and Narcissus, PAM fails to provide reliable post-purification robustness. Moreover, we observe that the PAM method sacrifices
a certain degree of CA to achieve thorough backdoor removal. Its average CA (83.14\%) is slightly lower than that of our BI-BAU (84.18\%).
This trade-off highlights a fundamental challenge in backdoor removal: achieving thorough purification often comes at the cost of clean performance,
particularly when defending against more sophisticated attacks characterized by low orthogonality or linearity.

Moreover, \citet{qi2023revisiting} introduce simple yet effective adaptive backdoor poisoning attacks, including Adap-Blend and Adap-Patch, which challenge 
the fundamental assumptions (\eg latent separability), underlying existing defenses and significantly undermine their effectiveness. Remarkably, under 
the more challenging adaptive backdoor attack scenarios, BI-BAU remains the only defense method that achieves post-purification robustness, highlighting 
its superior generalizability and resilience. 

\renewcommand{\arraystretch}{1.5}
\begin{table*}[!ht]
	\centering
	\caption{Post-purification robustness performance (\%) of BI-BAU under backdoor re-activation attacks on ViT.}
	\label{table:post_purification_robustness_on_ViT}
	\resizebox{0.98\textwidth}{!}{\begin{tabular}{c|c|c|c|c|c|c|c|c|c|c|c}
		\toprule
		Attack $\to$ & \multicolumn{1}{c|}{BadNets} & \multicolumn{1}{c|}{Blended} & \multicolumn{1}{c|}{WaNet} & \multicolumn{1}{c|}{Refool} & \multicolumn{1}{c|}{IAD} & \multicolumn{1}{c|}{ISSBA} & \multicolumn{1}{c|}{SIG} & \multicolumn{1}{c|}{Narcissus} &  \multicolumn{1}{c|}{Adap-Blend} & \multicolumn{1}{c|}{Adap-Patch} & \multicolumn{1}{c}{Average}\\ 
		\midrule
		\multirow{2}{*}{Defense $\downarrow$} & CA/ASR & CA/ASR & CA/ASR & CA/ASR & CA/ASR & CA/ASR & CA/ASR & CA/ASR & CA/ASR & CA/ASR & CA/ASR \\
		
		                                     & R-ASR/Q-ASR  & R-ASR/Q-ASR  & R-ASR/Q-ASR  & R-ASR/Q-ASR  & R-ASR/Q-ASR  & R-ASR/Q-ASR  & R-ASR/Q-ASR  & R-ASR/Q-ASR  & R-ASR/Q-ASR  & R-ASR/Q-ASR  & R-ASR/Q-ASR  \\
		\midrule
		\multirow{2}{*}{No Defense}     & 95.20/96.98 & 95.32/100.00 & 94.16/98.30 & 95.60/96.90 & 95.34/99.24 & 95.04/81.70 & 94.83/98.66 & 95.48/61.40 & 95.43/64.90 & 95.34/55.70 & 95.17/85.37 \\
		
								 		& -/-  & -/- & -/-  & -/- & -/- & -/- & -/-  & -/- & -/-  &-/- & -/-  \\   
		\hline

		\multirow{2}{*}{BI-BAU(T)}     & 92.62/2.74 & 82.78/1.10 & 92.46/0.80 & 94.46/1.86 & 92.96/1.92  & 87.76/0.50 & 72.30/9.04 & 93.54/6.54 & 92.98/2.00 & 92.86/2.90  &  89.47/2.94  \\
		
								       & 0.76/0.62 & 8.80/2.50 & 4.10/1.10 & 26.74/1.58 & 47.70/9.64 & 1.10/0.20 & 36.17/18.29 & 11.12/4.32  & 23.20/6.60 & 24.90/6.60  & 18.45/5.08\\   
		\hline
		\multirow{2}{*}{BI-BAU(U)}     & 85.62/3.30 & 73.03/0.30 & 89.98/0.80 & 92.64/3.82 & 91.26/2.10 & 92.74/0.60 & 75.29/10.33 & 91.04/21.52 & 89.87/3.70 & 89.70/7.60 & 87.11/5.40 \\
		
								       & 54.98/28.30  & 6.90/1.60 & 4.70/1.30 & 26.22/5.78 & 10.38/1.88 & 1.70/0.60 & 38.08/33.18 & 20.10/14.86 & 8.40/3.10 & 21.60/12.10 & 19.30/10.27  \\   
        \bottomrule
	\end{tabular}}
	\vspace{-3mm}
\end{table*}

\subsection{Layer-wise Activation Analysis}
\label{appendix:layer-wise_activation_analysis}
To further verify whether ``complete unlearning'' is achieved beyond surface-level performance metrics, we employ neural-level evaluation metrics that directly 
probe the internal representations of the network. In particular, we adopt the \emph{Backdoor Existence Coefficient} (BEC), denoted as $\rho_{\text{BEC}}$, 
recently introduced by \citet{zhu2024breaking}, which quantifies the similarity of backdoor-related activations between a backdoored model and a target (purified) 
model across network layers. The implementation details of BEC are provided in Appendix~\ref{appendix:layer-wise_activation_analysis}.

An effective unlearning method should produce a purified model whose internal feature closely resemble those of a clean model, while substantially diverging from 
the backdoored model. However, the original $\rho_{\text{BEC}}$ may take negative values under certain conditions. To facilitate stable interpretation and comparison, 
we apply a sigmoid transformation $\sigma(\cdot)$ to $\rho_{\text{BEC}}$, mapping it into the range $[0,1]$. A larger value of $\sigma(\rho_{\text{BEC}})$ 
indicates a stronger presence of residual backdoor features in the model, whereas values closer to zero suggest effective removal of backdoor-related activations.
The quantitative results are reported in Table~\ref{table:BEC_for_post-purified_model_on_CIFAR-10}.
\begin{table}[!ht]
	\centering
	\caption{Backdoor Existence Coefficient (BEC) for the post-purified ResNet-18 model.}
	\label{table:BEC_for_post-purified_model_on_CIFAR-10}
	\resizebox{0.96\linewidth}{!}{
	\begin{tabular}{c|c|c|c|c|c|c|c}
		\toprule
		Attack $\to$ & \multicolumn{1}{c|}{BadNets} & \multicolumn{1}{c|}{Blended} & \multicolumn{1}{c|}{WaNet} & \multicolumn{1}{c|}{Refool} & \multicolumn{1}{c|}{IAD} & \multicolumn{1}{c|}{ISSBA} & \multicolumn{1}{c}{Narcissus} \\ 
		\midrule
		Defense $\downarrow$ & \multicolumn{1}{c|}{BEC} & \multicolumn{1}{c|}{BEC} & \multicolumn{1}{c|}{BEC}   & \multicolumn{1}{c|}{BEC}    & \multicolumn{1}{c|}{BEC} & \multicolumn{1}{c|}{BEC}   & \multicolumn{1}{c}{BEC} \\ 
		\midrule
		NC    		&  0.4592  & 0.1911 & 0.6672  & 0.4686 &  0.2055  & 0.4336  & 0.6489 \\
        \hline
		BTI-DBF    	&  0.1639  & 0.1639 & 0.2194  & 0.1428 &  0.0781  & \scalebox{1.00}{\colorbox[HTML]{d9f2d9}{0.0208}}  & 0.1616 \\
		\hline
		ANP    	    &  0.6174  & 0.5930 & 0.5849 &	0.3164  & 0.3000 &	0.2994	& 0.6817 \\
        \hline
        EP     	    &  0.7172  & 0.6206 & 0.7031  & 0.8454 &  0.1646  & 0.6893  & 0.8311 \\
		\hline
		I-BAU 	    &  0.3743  & 0.4898 & 0.6776  & 0.5759 &  0.1445  & 0.4330  & 0.6033 \\
		\hline
		SAU   	    &  0.1730  & \scalebox{1.00}{\colorbox[HTML]{d9f2d9}{0.1150}} & 0.1765  & 0.1018 &  0.1074  & 0.0365  & 0.0976 \\
		\hline				
        BI-BAU (U)  &  \scalebox{1.00}{\colorbox[HTML]{d9f2d9}{0.1501}}  & 0.1164 & \scalebox{1.00}{\colorbox[HTML]{d9f2d9}{0.1764}}  & \scalebox{1.00}{\colorbox[HTML]{d9f2d9}{0.0853}} &  \scalebox{1.00}{\colorbox[HTML]{d9f2d9}{0.0141}} & 0.1855 & \scalebox{1.00}{\colorbox[HTML]{d9f2d9}{0.0287}} \\
        \bottomrule
	\end{tabular}}
\end{table} 

From Table~\ref{table:BEC_for_post-purified_model_on_CIFAR-10}, we observe that BI-BAU consistently achieves the lowest or near-lowest BEC values across a wide 
range of backdoor attacks. This indicates that, at the neural representation level, the post-purified models produced by BI-BAU are significantly closer to 
the clean model than those obtained by other defenses. In contrast, several existing methods exhibit relatively high BEC values, suggesting that residual 
backdoor features remain embedded in their internal activations.
When combined with the robustness results reported in Table~\ref{table:post_purification_robustness_on_CIFAR-10}, these findings provide strong and complementary 
evidence that BI-BAU does not merely suppress backdoor behaviors at the output level, but instead fundamentally removes backdoor-related representations 
across the deep layers of the network. This neural-level consistency with the clean model aligns well with our theoretical definition of \emph{complete backdoor unlearning}, and 
further validates the effectiveness of BI-BAU as a principled and thorough post-purification defense.

\subsection{Experiments on the Vision Transformer Architecture}
\label{subsection:experiment_for_ViT_on_CIFAR-10}

With the widespread adoption of Transformer-based architectures, Vision Transformers (ViTs)~\cite{dosovitskiy2020image} have become a dominant paradigm in modern 
vision systems. To evaluate the effectiveness of our proposed BI-BAU method in defending against backdoors embedded in transformer-based architectures~\cite{dosovitskiy2020image}, 
we conduct experiments using the vision encoder of CLIP (ViT-B/32)~\cite{radford2021learning} as the backbone, followed by a linear classification head for image classification.
Specifically, we initialize the model with publicly available CLIP weights pretrained on 400M image-text pairs, inject 10\% poisoned samples into the CIFAR-10 training set, 
and perform full-model fine-tuning for 10 epochs to obtain a backdoored ViT model. The experimental results are summarized in Table~\ref{table:post_purification_robustness_on_ViT}.

From Table~\ref{table:post_purification_robustness_on_ViT}, we observe that BI-BAU still achieves strong post-purification robustness. For BI-BAU(T), the average 
Q-ASR is only 5.08\% with a corresponding CA drop of 5.70\%, while BI-BAU(U) achieves 10.27\% Q-ASR with an 8.06\% reduction in CA. This favorable trade-off indicates 
that BI-BAU effectively suppresses residual backdoor behaviors without substantially degrading benign performance.
Notably, these results are obtained on a large-scale Vision Transformer backbone (CLIP ViT-B/32), whose global self-attention and high-capacity representations 
can make backdoor persistence and reactivation more challenging. The strong performance of BI-BAU on this architecture demonstrates its generalizability beyond 
convolutional networks and highlights its applicability to a wide range of modern models, including large pretrained foundation models widely used in practice.

\subsection{Evaluation on Multi-Modal Contrastive Learning}
\label{section:evaluation_on_MCL}

To assess the robustness of BI-BAU in real-world multi-modal scenarios, we evaluate it against representative backdoor attacks in contrastive learning. Beyond 
the classical BadNets~\cite{gu2019badnets}, we consider the state-of-the-art BadCLIP~\cite{liang2024badclip}, which is specifically designed for vision-language 
models (\eg CLIP) to resist backdoor unlearning via fine-tuning strategies and demonstrates strong resilience against existing SoTA defenses. 
We compare BI-BAU with CleanCLIP~\cite{bansal2023cleanclip}, a defense method tailored for CLIP backdoor mitigation, and evaluate the purified models via zero-shot 
classification on the ImageNet-1K validation set, with results reported in Figure~\ref{fig:robustness_performance_for_MCL_model}.
\begin{figure}[!h]
	\graphicspath{{figures/re_activation/clip}}
	\centering
	\begin{subfigure}[b]{0.23\textwidth}
		\centering
		\includegraphics[width=\linewidth]{BadNets.pdf}
		\caption{BadNets}
	\end{subfigure}
	\begin{subfigure}[b]{0.23\textwidth}
		\centering
		\includegraphics[width=\linewidth]{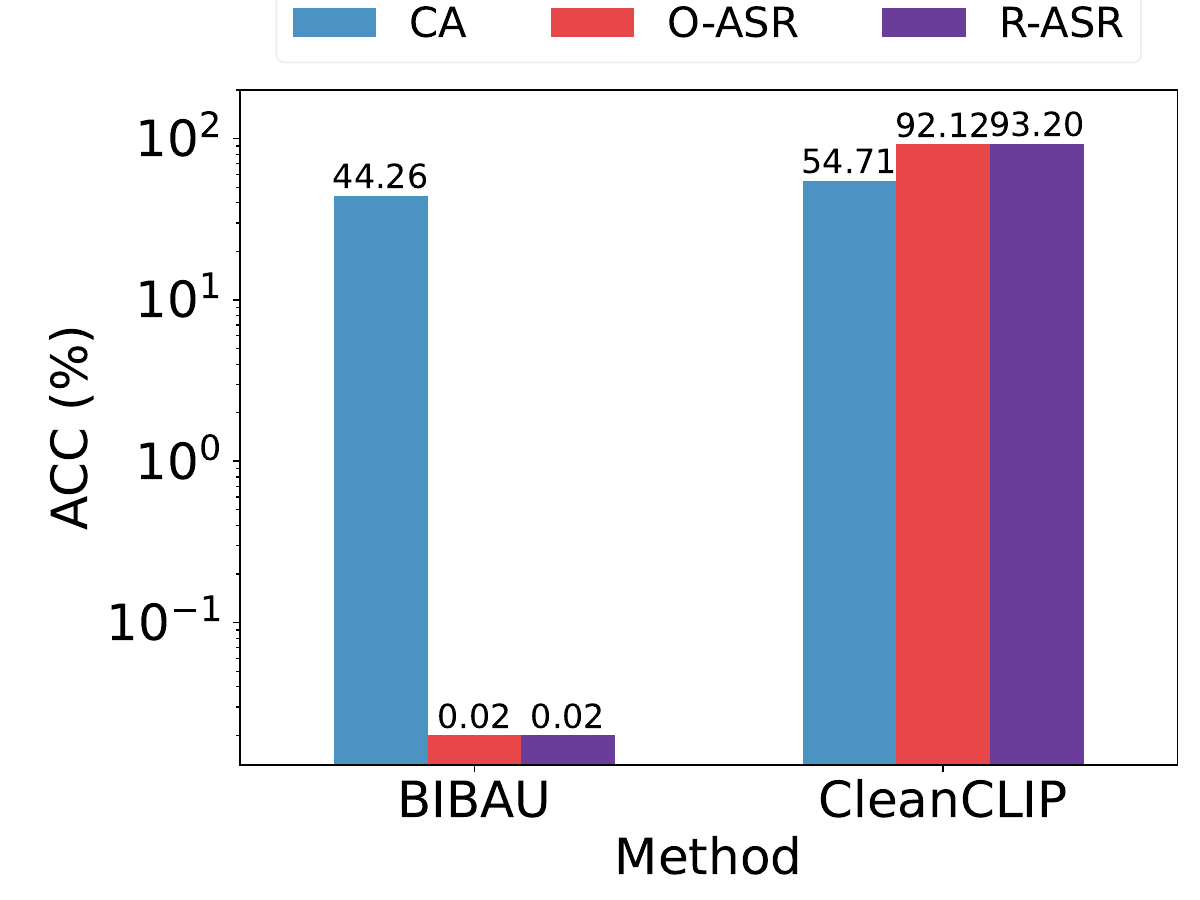}
		\caption{BadCLIP}
	\end{subfigure}
	\caption{\small Post-purification robustness performance (\%) of BI-BAU in Multi-Modal Contrastive Learning}
	\label{fig:robustness_performance_for_MCL_model}
	\vspace{-3mm}
\end{figure}

As shown in Figure~\ref{fig:robustness_performance_for_MCL_model}, although CleanCLIP can partially mitigate BadNets, it fails to fundamentally remove 
the backdoor. In particular, the residual backdoor can be readily reactivated by the Retuning Attack (RA), resulting in a high re-activation attack success
rate (R-ASR) of up to 42.70\%. More critically, CleanCLIP exhibits severe vulnerability to BadCLIP attack, reaching an ASR of 92.12\%, highlighting its limited 
robustness against advanced multimodal backdoors. In contrast, BI-BAU consistently achieves low ASR and R-ASR under both BadNets and BadCLIP attacks, demonstrating 
effective backdoor elimination. By maintaining strong resistance to backdoor reactivation, BI-BAU performs \emph{deep backdoor removal}, rather than merely 
suppressing malicious behaviors at inference time. Notably, the observed marginal drop in CA under these extreme attack scenarios reflects a natural trade-off 
between complete backdoor removal and maintaining model performance. These results validate BI-BAU as a principled and robust post-purification defense for 
multi-modal contrastive learning systems. 

\subsection{Robustness to Adaptive Attacks}

Qi~\etal~\cite{qi2023revisiting} introduce several simple yet effective adaptive backdoor attacks, namely Adap-Blend and Adap-Patch, as counter-examples 
to the latent separability assumption~\cite{geirhos2020shortcut,huang2022backdoor}, thereby circumventing most existing defenses~\cite{xu2024towards,wu2021adversarial,wei2023shared}.
Specifically, they employ a mixed-label strategy, where a fraction of trigger-injected samples are randomly retained with their correct semantic labels.
In addition, they use asymmetric trigger planting, in which weakened triggers are used during data poisoning while the original standard trigger is applied at 
test time to activate the backdoor. According to our analysis, the effectiveness of these attacks arises from their promotion of coupling between the clean and 
backdoor tasks.

Building upon Adap-Blend and Adap-Patch, we further design adaptive attacks tailored to our BI-BAU framework. Specifically, we introduce a feature coupling 
strategy that explicitly promotes the alignment between backdoor features and target-class features during poisoning training, for example by minimizing their 
Euclidean distance, thereby further increasing the coupling between the clean and backdoor tasks to evade BI-BAU. As illustrated in Figure~\ref{fig:t-SNE_visualization_for_adaptive_ttack}, 
the t-SNE visualization~\cite{van2008visualizing} shows that backdoor and clean features become highly entangled under such adaptive attacks.
\begin{figure}[!h]
	\graphicspath{{figures/adaptive_attack/}}
	\centering
	\begin{subfigure}[b]{0.23\textwidth}
		\centering
		\includegraphics[width=\linewidth]{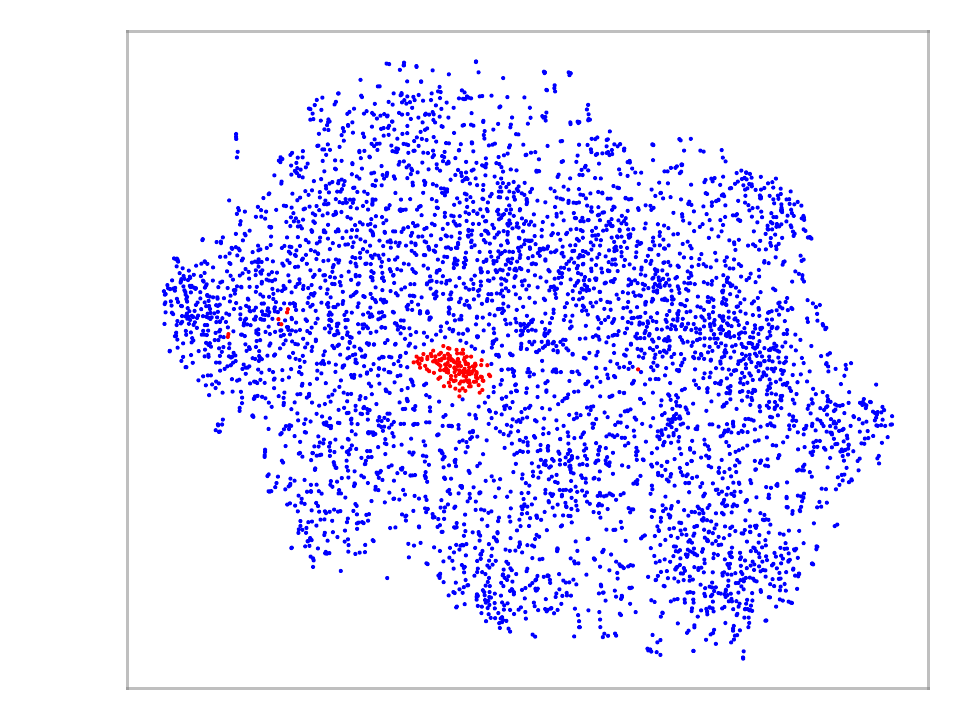}
		\caption{Enhanced Adap-Blend}
	\end{subfigure}
	\begin{subfigure}[b]{0.23\textwidth}
		\centering
		\includegraphics[width=\linewidth]{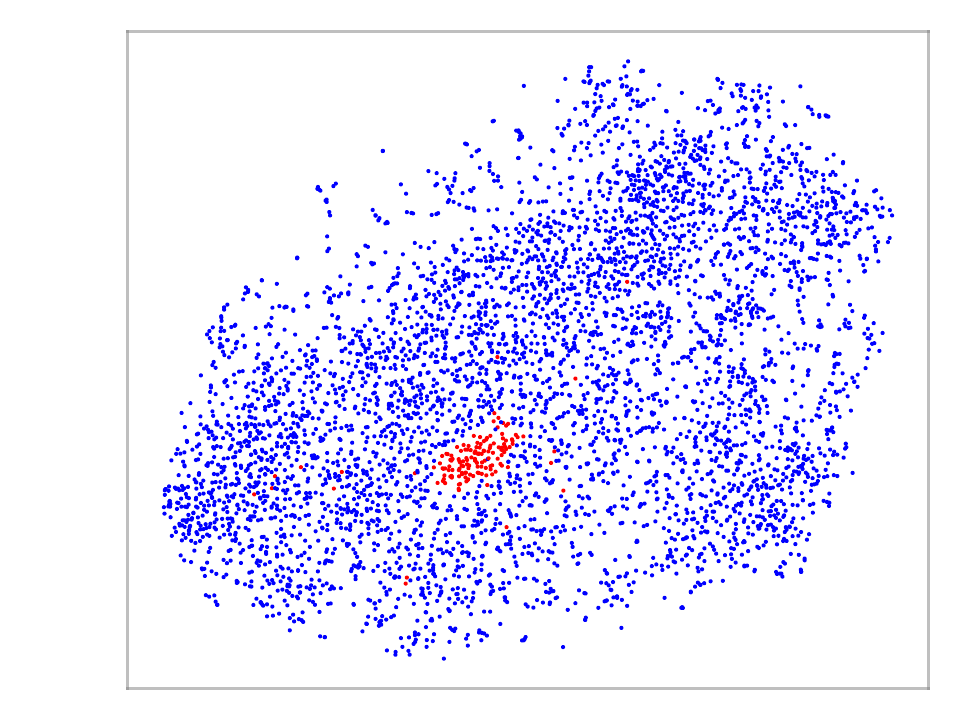}
		\caption{Enhanced Adap-Patch}
	\end{subfigure}
	\caption{\small t-SNE visualizations of latent representations on CIFAR-10 under the enhanced Adap-Blend and Adap-Patch attacks. The figure shows target-class 
	clean samples (blue) and triggered poisoned samples (red, originally from other classes), which are highly entangled in feature space.}
	\label{fig:t-SNE_visualization_for_adaptive_ttack}
	\vspace{-3mm}
\end{figure}

Subsequently, we evaluate the robustness of BI-BAU against such adaptive attacks, with the results summarized in the Table~\ref{table:robustness_against_adaptive_attack}.
We observe that under the enhanced Adap-Blend and Adap-Patch attacks, BI-BAU maintains strong robustness, achieving average R-ASR and Q-ASR of 7.70\% and 5.05\%, 
respectively, demonstrating its resilience against adaptive attacks. However, the strong entanglement between the clean and backdoor tasks causes the alignment 
between the unlearning task and the backdoor task to inevitably interfere with the clean task, resulting in a moderate degradation of model performance (with an 
average CA drop of 17.61\%). Given the practical threat posed by re-activation attacks, BI-BAU is deliberately designed to prioritize robust backdoor removal, 
even at the cost of a moderate reduction in clean accuracy.
\renewcommand{\arraystretch}{1.2}
\begin{table}[!ht]
	\centering
	\caption{Post-purification robustness (\%) of BI-BAU against adaptive attacks (ResNet-18, CIFAR-10).}
	\label{table:robustness_against_adaptive_attack}
	\resizebox{0.45\textwidth}{!}{\begin{tabular}{c|c|c|c}
		\toprule
		Attack $\to$ & \multicolumn{1}{c|}{Enhanced Adap-Blend} & \multicolumn{1}{c|}{Enhanced Adap-Patch} & \multicolumn{1}{c}{Average}\\ 
		\midrule
		\multirow{2}{*}{Defense $\downarrow$} & CA/ASR & CA/ASR & CA/ASR \\
		
		                                     & R-ASR/Q-ASR  & R-ASR/Q-ASR  & R-ASR/Q-ASR   \\
		\midrule
		\multirow{2}{*}{No Defense}     & 92.46/81.40 & 92.26/92.20 & 94.16/98.30 \\
		
								 		& -/-  & -/- & -/- \\   
		\hline
		\multirow{2}{*}{BI-BAU(U)}     & 74.76/2.30 & 74.75/9.60 & 74.75/5.95 \\
		
								       & 6.60/2.30 & 8.80/7.80 & 7.70/5.05  \\   
        \bottomrule
	\end{tabular}}
	\vspace{-3mm}
\end{table}

\subsection{Efficiency Evaluation}
\label{subsection:efficiency_and_computational_cost}
We compare the computational efficiency of BI-BAU against all baseline defenses, as summarized in Figure~\ref{fig:computational_cost}.
\begin{figure}[htbp]
	\vspace{-3mm}
    \graphicspath{{figures/computing_time/}}
    \centering  
    \includegraphics[width=0.35\textwidth]{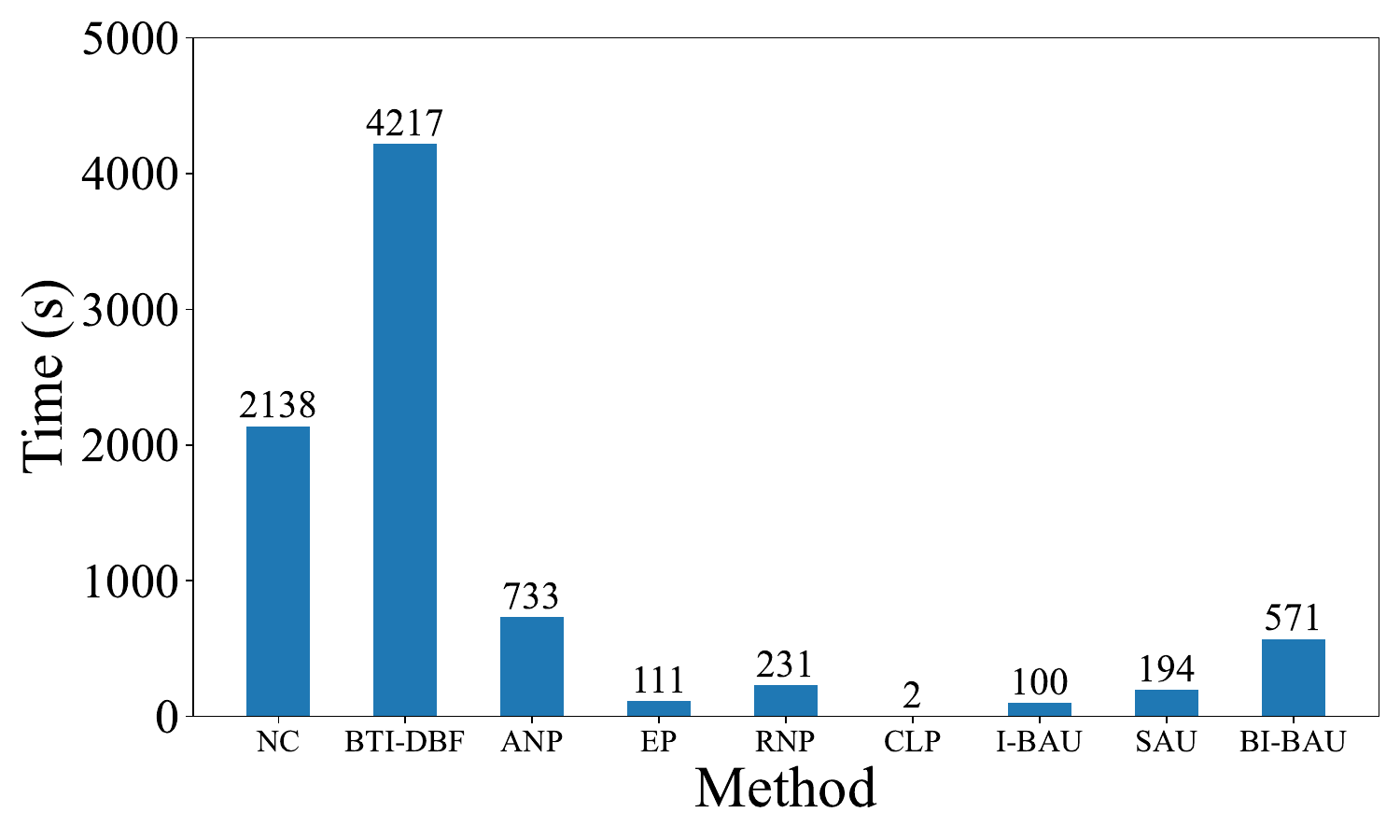}
	\caption{Computational Cost}
	\label{fig:computational_cost}
	\vspace{-3mm}
\end{figure}

We observe that trigger synthesis-based methods, including NC and BTI-DBF, incur the highest computational overhead, primarily due to the need to train additional 
models for backdoor trigger reconstruction. In contrast, model reconstruction-based methods (\ie EP, ANP, RNP, and CLP) are typically the most efficient—
for example, CLP requires only around 2 seconds—since they directly identify and prune suspicious backdoor-related parameters. However, such approaches often face 
a challenging trade-off between model robustness and performance, as they cannot precisely disentangle clean and backdoor parameters, particularly under attacks 
with low orthogonality or low linearity.

Backdoor adversarial unlearning methods (\ie I-BAU, SAU, and BI-BAU) avoid both trigger synthesis and explicit parameter-space pruning, thereby achieving 
a more favorable trade-off between efficiency and robustness. BI-BAU requires slightly more time than SAU and I-BAU under the same fine-tuning settings, due to 
additional inner optimization iterations, but remains within the same order of magnitude.

Furthermore, since BI-BAU is built upon standard adversarial training, we provide a detailed analysis of its computational overhead under different model architectures, 
including ResNet-18 and ViT-B/32. On CIFAR-10, we train for 20 epochs, using a batch size of 256 for ResNet-18 and 64 for ViT-B/32, respectively. The results show 
that BI-BAU incurs a modest additional overhead of 110 seconds for ResNet-18 (308s vs 418s) and 369 seconds for ViT-B/32 (578s vs 947s) compared to standard training. 
We argue that this modest overhead is justified by the substantially improved thoroughness of backdoor removal, especially in light of the practical threat posed 
by re-activation attacks.

\subsection{Ablation Study}
\label{subsection:ablation_study}
In this section, we conduct ablation studies to examine the roles of individual components in our BI-BAU framework. As presented in Section~\ref{section:proposed_method}, 
BI-BAU augments standard adversarial training (AT) with regularization constraints applied to both the inner maximization and outer minimization steps.
Specifically, for the inner maximization step, we impose two constraints on adversarial example generation: \ding{172} a data likelihood term $\ell_{kd}$ and \ding{173} 
a prior regularization term $\ell_{prior}$, weighted by hyperparameters $\lambda_1$ and $\lambda_2$. For the outer minimization step, which targets the backdoor unlearning, 
we introduce two regularization terms on model parameters: \ding{174} $\ell_{ewc}$ and \ding{175} $\ell_{inv\text{-}ewc}$, with their contributions balanced by 
the hyperparameters $\gamma_1$ and $\gamma_2$, respectively.

These regularizers are grouped into two sets: Constraint Set A, which includes terms \ding{172} and \ding{174}, which aim to preserve clean accuracy; 
and Constraint Set B, which includes terms~\ding{173} and~\ding{175}, which focuses on mitigating backdoor effects. If we denote the original adversarial training 
paradigm as ``Base AT'',  then BI-BAU can be expressed as: ``Base AT + A + B''. We evaluate four configurations, including ``Base AT'', ``Base AT + A + B'', 
``Base AT + A'', and ``Base AT + B'', to empirically validate the contributions of each key component. The experimental results are summarized in Table~\ref{table:ablation_study}, 
where \scalebox{0.75}{[\colorbox[HTML]{FFCCCC}{red}]} and \scalebox{0.75}{[\colorbox[HTML]{d9f2d9}{green}]} indicate results with poor defense performance (\ie high ASR) 
and those with favorable defense performance (\ie high CA or low ASR), respectively.
\begin{table}[!ht]
	\centering
	\caption{Post-purification robustness performance (\%) of Based AT and different variants of BI-BAU}
	\label{table:ablation_study}
	\resizebox{0.48\textwidth}{!}{\begin{tabular}{c|c|c|c|c|c}
		\toprule
		Attack $\to$ & \multicolumn{1}{c|}{BadNets} & \multicolumn{1}{c|}{Refool} & \multicolumn{1}{c|}{IAD} & \multicolumn{1}{c|}{ISSBA}  & \multicolumn{1}{c}{Average}\\ 
		\midrule
		\multirow{2}{*}{Condition $\downarrow$} & CA/ASR & CA/ASR & CA/ASR & CA/ASR & CA/ASR  \\
		
		                                        & R-ASR/Q-ASR & R-ASR/Q-ASR & R-ASR/Q-ASR & R-ASR/Q-ASR & R-ASR/Q-ASR \\
		\midrule
		\multirow{2}{*}{No Defense}     & 90.00/99.70 & 91.25/99.30 & 91.96/99.90 & 88.18/99.10 & 90.34/99.50     \\
		
								 		&     -/-     &      -/-     &      -/-    &      -/-    &  -/-  \\   
		\hline
		\multirow{2}{*}{Base AT}        & 83.96/17.30   & 86.91/0.30  & 86.54/2.70   & 83.07/1.00  & 85.12/5.32  \\
	
								        & \scalebox{1.00}{\colorbox[HTML]{FFCCCC}{96.70}}/\scalebox{1.00}{\colorbox[HTML]{FFCCCC}{95.10}} & \scalebox{1.00}{\colorbox[HTML]{FFCCCC}{50.50}}/4.10  & \scalebox{1.00}{\colorbox[HTML]{FFCCCC}{37.00}}/\scalebox{1.00}{\colorbox[HTML]{FFCCCC}{14.00}} & \scalebox{1.00}{\colorbox[HTML]{FFCCCC}{78.50}}/\scalebox{1.00}{\colorbox[HTML]{FFCCCC}{18.10}} & \scalebox{1.00}{\colorbox[HTML]{FFCCCC}{65.67}}/\scalebox{1.00}{\colorbox[HTML]{FFCCCC}{32.82}} \\   
		
		\hline
		\multirow{1}{*}{Base AT + A + B} & \scalebox{1.00}{\colorbox[HTML]{d9f2d9}{83.31}}/\scalebox{1.00}{\colorbox[HTML]{d9f2d9}{1.90}} & \scalebox{1.00}{\colorbox[HTML]{d9f2d9}{86.21}}/\scalebox{1.00}{\colorbox[HTML]{d9f2d9}{0.80}} & \scalebox{1.00}{\colorbox[HTML]{d9f2d9}{85.53}}/\scalebox{1.00}{\colorbox[HTML]{d9f2d9}{1.10}} & \scalebox{1.00}{\colorbox[HTML]{d9f2d9}{81.60}}/\scalebox{1.00}{\colorbox[HTML]{d9f2d9}{2.60}} & \scalebox{1.00}{\colorbox[HTML]{d9f2d9}{84.16}}/\scalebox{1.00}{\colorbox[HTML]{d9f2d9}{1.60}}  \\

		\multirow{1}{*}{(BI-BAU)}	  & \scalebox{1.00}{\colorbox[HTML]{d9f2d9}{4.80}}/\scalebox{1.00}{\colorbox[HTML]{d9f2d9}{2.20}} & \scalebox{1.00}{\colorbox[HTML]{d9f2d9}{7.90}}/\scalebox{1.00}{\colorbox[HTML]{d9f2d9}{1.40}} & \scalebox{1.00}{\colorbox[HTML]{d9f2d9}{6.80}}/\scalebox{1.00}{\colorbox[HTML]{d9f2d9}{2.90}} & \scalebox{1.00}{\colorbox[HTML]{d9f2d9}{2.50}}/\scalebox{1.00}{\colorbox[HTML]{d9f2d9}{1.80}} & \scalebox{1.00}{\colorbox[HTML]{d9f2d9}{5.50}}/\scalebox{1.00}{\colorbox[HTML]{d9f2d9}{2.07}} \\

		\hline
		\multirow{2}{*}{Base AT + A}     & \scalebox{1.00}{\colorbox[HTML]{d9f2d9}{87.27}}/40.10 & \scalebox{1.00}{\colorbox[HTML]{d9f2d9}{88.68}}/0.50 & \scalebox{1.00}{\colorbox[HTML]{d9f2d9}{87.65}}/0.00  & \scalebox{1.00}{\colorbox[HTML]{d9f2d9}{84.73}}/0.00 & \scalebox{1.00}{\colorbox[HTML]{d9f2d9}{87.08}}/10.15 \\

		                                 & \scalebox{1.00}{\colorbox[HTML]{FFCCCC}{98.00}}/\scalebox{1.00}{\colorbox[HTML]{FFCCCC}{98.20}} & \scalebox{1.00}{\colorbox[HTML]{FFCCCC}{51.70}}/\scalebox{1.00}{\colorbox[HTML]{FFCCCC}{29.60}} & \scalebox{1.00}{\colorbox[HTML]{FFCCCC}{49.90}}/\scalebox{1.00}{\colorbox[HTML]{FFCCCC}{13.80}} & \scalebox{1.00}{\colorbox[HTML]{FFCCCC}{98.90}}/\scalebox{1.00}{\colorbox[HTML]{FFCCCC}{22.40}} & \scalebox{1.00}{\colorbox[HTML]{FFCCCC}{74.62}}/\scalebox{1.00}{\colorbox[HTML]{FFCCCC}{41.00}}  \\
		\hline
		\multirow{2}{*}{Base AT + B}     & 79.12/2.10 & 81.73/1.50  & 81.92/6.30 & 80.53/2.10  & 80.82/3.00    \\

										 & \scalebox{1.00}{\colorbox[HTML]{d9f2d9}{3.40}}/\scalebox{1.00}{\colorbox[HTML]{d9f2d9}{2.00}} & \scalebox{1.00}{\colorbox[HTML]{d9f2d9}{10.60}}/\scalebox{1.00}{\colorbox[HTML]{d9f2d9}{2.20}} & \scalebox{1.00}{\colorbox[HTML]{d9f2d9}{11.50}}/\scalebox{1.00}{\colorbox[HTML]{d9f2d9}{6.90}} & \scalebox{1.00}{\colorbox[HTML]{d9f2d9}{1.80}}/\scalebox{1.00}{\colorbox[HTML]{d9f2d9}{0.30}} & \scalebox{1.00}{\colorbox[HTML]{d9f2d9}{6.82}}/\scalebox{1.00}{\colorbox[HTML]{d9f2d9}{2.85}} \\
        \bottomrule
	\end{tabular}}
	\vspace{-3mm}
\end{table}

From Table~\ref{table:ablation_study}, we observe that although ``Base AT'' can alleviate backdoor effects to a certain extent, it remains insufficient for their 
complete elimination. By comparing it with the ``Base AT + B'' setting, we conclude that Constraint Set B plays a critical role in completely removing backdoor 
effects, albeit with a slight drop in clean accuracy (CA). In contrast, the comparison between ``Base AT + A'' and ``Base AT'' reveals that Constraint Set A 
contributes to the preservation of the model's CA, but results in a slightly higher ASR. Overall, the ``Base AT + A + B'' setting, \ie our proposed BI-BAU, achieves 
the best performance, striking a desirable balance between maintaining model performance on clean task and effectively eliminates backdoor effects.

\subsection{Failure Analysis and Limitations}

Our method is based on the principle of aligning the unlearning task with the backdoor task while remaining orthogonal to the clean task. Consequently, when 
the clean and backdoor tasks exhibit high similarity, achieving this balance becomes challenging. Although our experiments demonstrate that the method exhibits 
a degree of robustness against the entanglement of clean and backdoor features, in extreme scenarios where this coupling is very high, adaptive attacks may 
circumvent our defense, thereby limiting its effectiveness. Additionally, BI-BAU is currently tailored for visual models, including ResNet and ViT architectures,
as well as multi-modal contrastive learning models such as CLIP. While the framework is theoretically extensible to other domains, including large language models (LLMs), 
practical adaptation for the principle of task-aligned backdoor unlearning may depend on the model architecture and task specifics.

%% file: section/Conclusion.tex
\section{Conclusion}
To the best of our knowledge, we are the first to formally define complete backdoor unlearning from the perspective of catastrophic forgetting within a continual 
learning framework, and derive the necessary conditions for achieving it. Building on this foundation, we propose BI-BAU method, which integrates a bi-level 
optimization process of adversarial training within an EM algorithm framework to approximate the complete backdoor unlearning objective. Extensive experiments 
demonstrate that our method is broadly applicable across a wide range of backdoor threats, and can effectively and thoroughly eliminate the backdoor effects. 
In the future, further exploration is excepted to extend this framework to generative models, including large language models (LLMs) and diffusion models, to 
further enhance its generalizability.

%% file: section/Acknowledgment.tex
\newpage
\section*{Acknowledgements}
We thank the anonymous reviewers for their valuable comments and suggestions. 
This paper was edited for grammar, clarity, and minor stylistic improvements using ChatGPT. 
All technical content, experimental design, and conclusions are solely the responsibility of the authors.

%% file: section/Supplementary.tex

\appendix

\section*{Open Science}
\label{appendix:open_science}
To support the evaluation of this paper's core contributions, we provide the following artifacts:

\begin{itemize}
    \item \textbf{Source Code:} Full implementation of the proposed BI-BAU defense method and evaluation pipeline.
    \item \textbf{Configuration Files:} Model settings, hyperparameters, and experimental setups.
    \item \textbf{Scripts and Documentation:} Scripts for running experiments, reproducing results, and accompanying README with usage instructions.
    \item \textbf{Environment:} Conda environment file (`requirements.yml`) specifying dependencies for replication.
\end{itemize}

All artifacts are accessible to the program committee via the anonymous repository: \url{https://anonymous.4open.science/r/BI-BAU-0634/}.

All artifacts contain no identifying information to ensure compliance with double-blind review.  
No restricted or sensitive data are included. The provided materials allow the program committee to fully reproduce the main experimental results reported in the paper.

\section*{Organization of the Appendix}
For clarity and ease of navigation, the Appendix is organized as follows:\\\
\begin{itemize}
	\item \textbf{Section~\ref{appendix:summary_of_symbols}} provides a summary of the symbols used throughout the paper. 
	\item \textbf{Section~\ref{appendix:theoretical_proof}} presents the theoretical proofs included in this work. 
	\item \textbf{Section~\ref{appendix:extending_BI-BAU_to_multi_modal_contrastive_learning}} details the implementation of extending BI-BAU to multi-modal 
	contrastive learning.
	\item \textbf{Section~\ref{appendix:additional_implementation_details}} describes additional implementation details for the experiments.
	\item \textbf{Section~\ref{appendix:additional_experiments}} presents complementary experimental results.
\end{itemize}

\section{Summary of Symbols}
\label{appendix:summary_of_symbols}
The comprehensive list of all notations is summarized in Table~\ref{table:all_notations}.

\begin{table}[!h]
\centering
\caption{Glossary of Notations.}
\label{table:all_notations}
\setlength{\tabcolsep}{6pt}
\renewcommand{\arraystretch}{1.0}
\resizebox{0.98\linewidth}{!}{ 
\begin{tabular}{p{0.22\linewidth} p{0.70\linewidth}}
\toprule

\multicolumn{2}{l}{\textbf{Spaces and Datasets}} \\
\midrule
$\mathbb{R}^{n_0}$ & Input space \\
$\mathbb{R}^{K}$ & Label space with $K$ classes \\
$\mathcal{X}$ & Input set \\
$\mathcal{Y}$ & Label set \\
$\mathcal{D} = D_c \cup D_p$ & Training dataset composed of clean and poisoned samples \\
$D_c$ & Clean subset of the training data \\
$D_p$ & Poisoned (backdoor) subset of the training data \\
$(x,y)$ & Input sample $x$ with ground-truth label $y$ \\
$y_t$ & Target label of the backdoor attack \\
$\tilde{x}$ & Adversarial example corresponding to $x$ \\

\midrule
\multicolumn{2}{l}{\textbf{Model Architecture}} \\
\midrule
$ \scriptstyle f = \text{FC} \circ f_{\text{backbone}}$ & Feed-forward network with $L$ hidden layers  \\
$f_{\text{backbone}}$ & Feature extractor (backbone network) \\
$\text{FC}$ & Fully connected classification head \\
$f_l$ & The $l$-th layer of the network \\
$\phi$ & Pointwise activation function \\
$h^l(x)$ & Pre-activation values at layer $l$ \\
$z^l(x)$ & Post-activation values at layer $l$ \\
$W^{l+1}$ & Weight matrix of layer $l+1$ \\
$b^{l+1}$ & Bias vector of layer $l+1$ \\
$\scriptstyle z(x) = f_{\text{backbone}}(x)$ & Feature representation of input $x$ \\

\midrule
\multicolumn{2}{l}{\textbf{Tasks and Learning Process}} \\
\midrule
$\mathcal{T}_{\tau}$ & A supervised learning task indexed by $\tau$ \\
$X^{\tau}$ & Dataset associated with task $\mathcal{T}_{\tau}$ \\
$\tau_c$ & Clean task \\
$\tau_b$ & Backdoor task \\
$\tau_u$ & Unlearning task \\

\midrule
\multicolumn{2}{l}{\textbf{Model Parameters and Optimization}} \\
\midrule
$\theta$ & Model parameters of $f$ \\
$\theta_{\tau}$ & Parameters during training on task $\tau$ \\
$\theta_{\tau}^{*}$ & Learned parameters after completing task $\tau$ \\
$\langle \cdot, \cdot \rangle$ & Euclidean inner product \\
$\|\cdot\|_2$ & Euclidean norm of a vector (or spectral norm of a matrix) \\

\bottomrule
\end{tabular}}
\end{table}

\section{Theoretical Proof}
\label{appendix:theoretical_proof}
\subsection{Proof of Proposition~\ref{proposition:proposition_1}}
\label{proof:proof_of_proposition_1}

\begin{proof}
For one of the backdoor unlearning objective $\Delta^{\tau_c \to \tau_u}(X^{\tau_c}) = 0$, we have:
\begin{equation}
\resizebox{0.90\linewidth}{!}{$
\begin{aligned}
    \Delta^{\tau_c->\tau_u}(X^{\tau_c}) &= \left \| \sigma^{\tau_c \to \tau_u} (X^{\tau_c})\right \|^2_2 \\
    & = \sum_{(x_c,y)\in X^{\tau_c}}(f^{*}_{\tau_u}(x_c)-f^{*}_{\tau_c}(x_c))^2 \\
    & = \sum_{(x_c,y)\in X^{\tau_c}}(f^{*}_{\tau_u}(x_c)-f^{*}_{\tau_b}(x_c) + f^{*}_{\tau_b}(x_c) - f^{*}_{\tau_c}(x_c))^2 \\
    & = \sum_{(x_c,y)\in X^{\tau_c}}[\nabla_{\theta} f^{*}_{\tau_b}(x_c)(\theta_{\tau_u}^{*}-\theta_{\tau_b}^{*}) + \nabla_{\theta} f^{*}_{\tau_c}(x_c) (\theta_{\tau_b}^{*}-\theta_{\tau_c}^{*})]^2.
\end{aligned}
$}
\end{equation}
Here, we used constant NTK assumption, \ie $\nabla_{\theta} f^{*}_{\tau}(x) = \nabla_{\theta} f_{0}(x)$,  Consequently, we have
\begin{equation}
	\resizebox{0.90\linewidth}{!}{$
    \Delta ^{\tau_c->\tau_u}(X^{\tau_c}) = \sum_{(x_c,y)\in X^{\tau_c}}(\nabla_{\theta} f_{0}(x_c)[(\theta_{\tau_u}^{*}-\theta_{\tau_b}^{*}) + (\theta_{\tau_b}^{*}-\theta_{\tau_c}^{*})])^2.
	$}
\end{equation}
Due to the orthogonality between the backdoor task $\tau_{b}$ and clean task $\tau_{c}$, namely $\nabla_{\theta}f_{0}(x_c) \perp (\theta_{\tau_b}^{*}-\theta_{\tau_c}^{*})$,
we have 
\begin{equation}
    \left \langle  \nabla_{\theta}f_{0}(x_c), (\theta_{\tau_b}^{*}-\theta_{\tau_c}^{*}) \right \rangle = 0,  \forall x_c \in X^{\tau_c}.
\end{equation}
Thus, if $\Delta ^{\tau_c->\tau_u}(X^{\tau_c}) = 0$, we have 
\begin{equation}
   \left \langle \nabla_{\theta} f_{0}(x_c), \theta_{\tau_u}^{*}-\theta_{\tau_b}^{*} \right \rangle = 0, 
\end{equation}
\ie
\begin{equation}
	\nabla_{\theta} f_{0}(x_c) \perp (\theta_{\tau_u}^{*}-\theta_{\tau_b}^{*}), \forall x_c \in X^{\tau_c}.
\end{equation}
Given that both $\theta_{\tau_u}^{*}-\theta_{\tau_b}^{*}$ and $\theta_{\tau_{c}}^{*} - \theta_{0}$ can be regarded as approximate gradient directions for 
the unlearning task $\tau_{u}$ and the clean task $\tau_{c}$, respectively, it follows that
\begin{equation}
    \nabla_{\theta}f_{0}(x_u,\theta) \perp (\theta_{\tau_{c}}^{*} - \theta_{0}), \forall x_u \in X^{\tau_u}.
\end{equation}
Similarly, for $\Delta^{\tau_c \to \tau_u}(X^{\tau_b}) = 0$, we have
\begin{align}
    \Delta ^{\tau_c->\tau_u}(X^{\tau_b}) & = \left \| \sigma^{\tau_c \to \tau_u} (X^{\tau_b})\right \|^2_2 \\
    & = \sum_{(x_b,y)\in X^{\tau_b}}(\nabla_{\theta} f_{0}(x_b)[(\theta_{\tau_{u}}^{*}-\theta_{\tau_{b}}^{*}) + (\theta_{\tau_{b}}^{*}-\theta_{\tau_{c}}^{*})])^2.
\end{align}
Thus, we obtain $(\theta_{\tau_{u}}^{*}-\theta_{\tau_{b}}^{*}) = -(\theta_{\tau_{b}}^{*}-\theta_{\tau_{c}}^{*})$, \ie
\begin{equation}
    \nabla_{\theta}f_{0}(x_u,\theta) \parallel (\theta_{\tau_{b}}^{*}-\theta_{\tau_{c}}^{*}), \forall x_u \in X_{\tau_u}. \\
\end{equation}
\end{proof}

\subsection{Proof of Proposition~\ref{proposition:proposition_2}}
\label{proof:proof_of_proposition_2}

\begin{proof}

For $\nabla_{\theta}f_{0}(x,\theta)$,based on Equation~\ref{eq:recurrence}, we have
\begin{align}
    \frac{\partial f(x,w)}{\partial w^{l+1}} &=  \frac{\partial f(x,w)}{\partial h^{l+1}} \cdot \frac{\partial h^{l+1}}{\partial w^{l+1}} \nonumber \\
    & = \frac{\partial f(x,w)}{\partial h^{l+1}} \cdot  z^{l}  \nonumber \\
    & = \frac{\partial f(x,w)}{z^{l+1}} \cdot \frac{\partial z^{l+1}}{\partial h^{l+1}} \cdot  z^{l}.
\end{align}
If the gradients of the backdoor unlearning task $\tau_u$ are aligned with those of the backdoor task $\tau_b$, \ie 
\begin{equation}
\nabla_{\theta}f_{0}(x_u,\theta) \parallel \nabla_{\theta}f_{0}(x_b,\theta), \forall x_{u} \in X_{\tau_u}, \forall x_{b} \in X_{\tau_b}.
\end{equation}
Then, the outputs of each layer of the neural network, $z^{l+1}_{u}$ and $z^{l+1}_{b}$ should also be aligned accordingly.

Given the characteristics of backdoor attacks, for the poisoned sample $\forall x_b \in X^{\tau_b}$, the backdoor model predicts the poisoned sample $x_b$ into 
the target label, whereas the clean model tends to classify it as its real source label $y$. Therefore, we can conclude that for $\forall x_b \in X^{\tau_b}$, 
we have $f^{*}_{\tau_b}(x_b) = y_t$ and $f^{*}_{\tau_c}(x_b) = y$. Consequently, For $\forall \tilde{x} \in X^{\tau_u}$, the following conditions must be satisfied:
\begin{equation}
    \left\{\begin{matrix}
        f^{*}_{\tau_c}(\tilde{x};\theta_{\tau_c}) =  f^{*}_{\tau_c}(x_b;\theta_{\tau_c}) = y, \\
        f^{*}_{\tau_b}(\tilde{x};\theta_{\tau_b}) =  f^{*}_{\tau_b}(x_b;\theta_{\tau_b}) = y_t.\\
    \end{matrix}\right.
\end{equation}

\end{proof}

\subsection{Derivation of \text{MAP$_{\theta}$}}
\label{proof:proof_of_MAP_theta}
For $\log p(\theta|y)$, we have
\begin{equation}
\resizebox{0.90\linewidth}{!}{$
\begin{aligned}
	& \log p(\theta|y) = \int q(\tilde{x}) \cdot \log(p(\theta|y)) d\tilde{x} \\
	& = \int q(\tilde{x}) \log(\frac{p(\theta,\tilde{x}|y)}{p(\tilde{x}|\theta,y)})d\tilde{x} \\
	& = \int q(\tilde{x}) \left [\log(\frac{p(\theta,\tilde{x}|y)}{q(\tilde{x})} )-\log(\frac{p(\tilde{x}|\theta,y)}{q(\tilde{x})})\right ]d\tilde{x} \\
	& = \underbrace{\int q(\tilde{x}) \left [\log(p(\theta,\tilde{x}|y))- \log(q(\tilde{x}) ) \right ]d\tilde{x}}_{L(q,\theta)}
	+ \underbrace{\int q(\tilde{x}) \log(\frac{q(\tilde{x})}{p(\tilde{x}|\theta,y)})d\tilde{x}}_{KL(q(\tilde{x})||p(\tilde{x}|\theta,y))}. 
\end{aligned}
$}
\end{equation}

\section{Extending BI-BAU to Multi-Modal Contrastive Learning}
\label{appendix:extending_BI-BAU_to_multi_modal_contrastive_learning}
Moreover, our BI-BAU framework can be easily extended to multi-modal contrastive learning scenarios. Taking the Contrastive Language-Image Pre-training (CLIP) model 
as an example, CLIP~\cite{radford2021learning} is a vision-language pre-training framework that enables cross-modal understanding by aligning image and text modalities 
within a shared feature space. A pre-trained CLIP model typically consists of an visual encoder $\mathcal{I}(\cdot; \theta_{\mathcal{I}})$ and a text encoder 
$\mathcal{T}(\cdot; \theta_{\mathcal{T}})$.

For a $K$-class classification task, given an image-text pair $(I,T)$, the visual encoder (\eg a Vision Transformer (ViT)~\cite{dosovitskiy2020image}) produces 
the image feature representation $ z(I) = \mathcal{I}(I;\theta_{\mathcal{I}})$.
For the text branch, each class name $t_i$ is inserted into a predefined template prompt (\eg ``a photo of a {class}''), forming the text $T_i$. 
The text encoder then produces the corresponding textual representation $z(T_i) = \mathcal{T}(T_i,\theta_{\mathcal{T}})$.

During training, CLIP employs contrastive learning to align image and text representations. Given a batch of $n$ pairs $(I_j, T_j)$, the bidirectional loss 
including image-to-text matching loss and Text-to-Image matching loss, is computed as:
\begin{equation}
    \mathcal{L}_{\text{img2txt}} = -\frac{1}{n}\sum_{j=1}^{n} \log \frac{\exp(\cos(z(I_j), z(T_j))/\tau)}{\sum_{k=1}^{n} \exp(\cos(z(I_j), z(T_k))/\tau)}, 
\end{equation}

\begin{equation}
    \mathcal{L}_{\text{txt2img}} = -\frac{1}{n}\sum_{j=1}^{n} \log \frac{\exp(\cos(z(T_j), z(I_j))/\tau)}{\sum_{k=1}^{n} \exp(\cos(z(T_j), z(I_k))/\tau)}.
\end{equation}
where $cos(·, ·)$ denotes the cosine similarity score and $\tau$ is the temperature parameter. The final training objective is the average of the two directions:
\begin{equation}
   \mathcal{L}_{\text{CLIP}} = \frac{1}{2}(\mathcal{L}_{\text{img2txt}} +  \mathcal{L}_{\text{txt2img}}).
\end{equation}

\textbf{Extend BI-BAU to CLIP Model.} Following the core principles established in the main paper, BI-BAU can be naturally extended to the CLIP model. 
As before, the framework consists of two main steps: an Inner Maximization step and an Outer Minimization step.

\textbf{In the Inner Maximization step.} For multi-modal contrastive learning, the cross-entropy loss $\ell_{ce}$ in Eq.~\ref{eq:ell_in} is replaced by the CLIP 
contrastive loss $\mathcal{L}_{\text{CLIP}}(\tilde{I}, T)$, where $\tilde{I}$ denotes adversarial images and $T$ their corresponding textual captions.

For the knowledge distillation term $\ell_{kd}$, inspired by Relational Knowledge Distillation (RKD)~\cite{park2019relational,tung2019similarity}, which transfers 
structural knowledge by leveraging mutual relations of data instances in the teacher's output space. we adopt a relational distillation loss:
\begin{equation}
 	\label{eq:rkd}
	\mathcal{L}_{\text{rkd}} = \text{KL}\big(\cos(Z_{\tilde{I}}, Z_{T})||\cos(Z_I, Z_T)\big)
\end{equation}
where $Z_I = \left\{z(I_j) \right\}_{j=1}^B$, $Z_{T} = \left\{z(T_j)\right\}_{j=1}^B$ denote natural image-text feature sets, while $Z_{\tilde{I}}$ represents 
the corresponding adversarial features. The pairs $(Z_I, Z_T)$ and $(Z_{\tilde{I}}, Z_T)$ are extracted from the original backdoored CLIP model and the updated 
model, respectively. The relational potential function $\cos(\cdot, \cdot)$ encodes the pairwise similarity between image-text features within a batch of size $B$. 
Accordingly, Kullback-Leibler (KL) divergence $\mathcal{L}_{\text{rkd}}$ measures the discrepancy between the relational structures induced by adversarial samples,
$\cos(Z_{\tilde{I}}, Z_T)$, and those of the natural samples, $\cos(Z_I, Z_T)$. 

By maximizing $-\mathcal{L}_{\text{rkd}}$, we enforce the adversarial image-text representations to preserve the pairwise relational structure induced by the 
original model, thereby retaining core semantic information while suppressing backdoor effects.

In the case of the CLIP model, the prior loss $\ell_{prior}$ is defined as:
\begin{equation}
    \label{eq:ell_clip_prior}
    \mathcal{L}_{prior} = \text{MSE} \big(Z_{\tilde{I}}, Z_{I}\big) - \cos(\beta + k \cdot \frac{\pi}{2})
\end{equation}
where $Z_{\tilde{I}}$ and $Z_I$ are visual encoder outputs of adversarial and clean images, respectively, and $\beta$ follows the definition in Section~\ref{section:proposed_method}.

Combining Eqs.~\ref{eq:rkd} and~\ref{eq:ell_clip_prior}, the resulting inner maximization objective is:
\begin{equation}
    \mathcal{L}_{adv-in} = -\mathcal{L}_{\text{CLIP}}  -\lambda_{1}\mathcal{L}_{kd} + \lambda_{2} \mathcal{L}_{prior}
\end{equation}

\textbf{In the Outer Minimization step.} the cross-entropy loss term in Eq.~\ref{eq:ell_out} is similarly replaced with the CLIP contrastive loss $\mathcal{L}{\text{CLIP}}(\tilde{I}, T)$.
The Elastic Weight Consolidation (EWC) loss and the inverse EWC loss retain the same formulations, given by
\begin{equation}
\scalebox{0.90}{$
\begin{aligned}
\mathcal{L}_{ewc} &= (\theta_{\tau_b}^{*}(\mathcal{I}) - \theta(\mathcal{I})) ^{T} F_{\tau_c} (\theta_{\tau_b}^{*}(\mathcal{I}) - \theta(\mathcal{I})), \\
\mathcal{L}_{inv-ewc} &= 1.0 / \left [ 1.0 + (\theta_{\tau_b}^{*}(\mathcal{I}) - \theta(\mathcal{I})) ^{T} F_{\tau_u} (\theta_{\tau_b}^{*}(\mathcal{I}) - \theta(\mathcal{I}))  \right ].
\end{aligned}
$}
\end{equation}
Here, $\theta(\mathcal{I})$ denotes the parameters of the visual encoder. Notably, in the CLIP-based unlearning setting, we update only the visual encoder 
$\mathcal{I}(\cdot;\theta_{\mathcal{I}})$ during optimization, while keeping the text encoder $\mathcal{T}(\cdot;\theta_{\mathcal{T}})$ fixed.

Accordingly, the overall outer optimization objective becomes:
\begin{equation}
    \label{eq:ell_clip_out}
    \mathcal{L}_{adv-out} = \mathcal{L}_{\text{CLIP}}(f_{\theta}(\tilde{x};\theta), y) + \gamma_{1} \mathcal{L}_{inv-ewc} + \gamma_{2} \mathcal{L}_{ewc}.
\end{equation}

This extension preserves the core BI-BAU principles while adapting the framework to multi-modal contrastive learning, enabling effective backdoor unlearning for 
CLIP-based models.

\section{Additional Implementation Details}
\label{appendix:additional_implementation_details}

\subsection{Datasets and Network Architectures}
\label{subsection:datasets_and_network}

To comprehensively evaluate the effectiveness of our proposed BI-BAU method, we conduct experiments on both standard image classification and multi-modal 
contrastive learning tasks.

\textbf{Supervised Learning.} For image classification, we use three benchmark datasets: CIFAR-10~\cite{krizhevsky2009learning}, Tiny-ImageNet-200~\cite{wu2017tiny}, 
and GTSRB~\cite{stallkamp2011german}, and three neural network architectures: ResNet-18~\cite{he2016deep}, VGG19~\cite{simonyan2014very}, and Vision Transformer (ViT)~\cite{dosovitskiy2020image}. 
Dataset statistics and model assignments are summarized in Table~\ref{tab:dataset_and_model}. 
\begin{table}[!h]
    \centering
    \caption{Summary of datasets and model architectures used in supervised learning tasks.}
    \label{tab:dataset_and_model}
    \resizebox{\linewidth}{!}{
    \begin{tabular}{ccccc}
        \toprule
        Dataset & Labels & Input Size & Training Images & DNN \\ \hline
        CIFAR-10 & 10 & 32$\times$32$\times$3 & 50,000 & ResNet-18 / ViT \\
        Tiny-ImageNet-200 & 200 & 32$\times$32$\times$3 & 100,000 & ResNet-18 \\
        GTSRB & 43 & 32$\times$32$\times$3 & 39,252 & VGG19 \\
        \bottomrule
    \end{tabular}}
	\vspace{-3mm}
\end{table}

For all classification tasks, we perform backdoor attacks on either the ResNet-18 or VGG19 models over 200 training epochs, with a batch size of 128. 
We employ the stochastic gradient descent (SGD) optimizer with an initial learning rate of 0.1, a momentum of 0.9, and a weight decay of $5 \times 10^{-4}$. 
The learning rate is reduced by a factor of 10 at the 100th and 150th epochs, respectively.

\textbf{Multi-Modal Contrastive Learning.} For multi-modal evaluation, we adopt OpenAI's open-source CLIP~\cite{radford2021learning} as the pre-trained backbone, 
trained on 400 million image-text pairs. Specifically, the ``ViT-B/32'' variant is used as the visual encoder, and a Transformer serves as the text encoder. 
A backdoor is first implanted into the pre-trained CLIP model on ImageNet-1K. Then, a fine-tuning dataset is constructed by sampling 50,000 image-text pairs from 
Conceptual Captions (CC3M)~\cite{sharma2018conceptual}, corresponding to only $\sim$0.6\% of the full dataset. This setup allows us to evaluate the effectiveness 
of BI-BAU under limited-data backdoor mitigation scenarios.

\subsection{Attack Settings}

\label{subsection:attack_settings}

\textbf{Backdoor Attacks.} We conduct a comprehensive study by adopting 10 representative backdoor attacks to evaluate the effectiveness of our method, which can 
be grouped into the following five categories: (1) Patch-based attacks: BadNets~\cite{gu2019badnets} and Blended~\cite{chen2017targeted};
(2) Invisible backdoor attacks: WaNet~\cite{nguyen2021wanet} and Refool~\cite{liu2020reflection}; (3) Sample-specific backdoor attacks: IAD~\cite{nguyen2020input} and ISSBA~\cite{li2021invisible};
(4) Clean-label attacks: SIG~\cite{barni2019new} and Narcissus~\cite{zeng2023narcissus}, as well as (5) two adaptive attacks~\cite{qi2023revisiting}: Adaptive-Blend and Adaptive-Patch,
which are specifically designed to evade latent separation-based defenses. For a fair and reliable comparison, we use the implementation provided by BackdoorBench~\cite{wu2022backdoorbench} 
and follow default configuration specified in the original papers. By default, the poisoning ratio $\rho$ is set to 10\% and the target label $y_t$ set to 0 for all attacks, 
expect for SIG and Narcissus. For the SIG attack, the poisoning ratio is set to $2.5\%$ (\ie $25\%$ of training samples with the target label), and the target label 
$y_t = 3$. For the Narcissus, the poisoning ratio is set to $0.05\%$ of the entire dataset, with the target label $y_t$ set to 0.

Furthermore, to verify that our proposed BI-BAU method can effectively and thoroughly eliminate backdoor effects from compromised models, we implement two backdoor 
re-activation attacks introduced by \citet{min2024uncovering}, namely the Retuning Attack (RA) and the Query-based Reactivation Attack (QRA).
Following its original implementation~\cite{min2024uncovering}, we provide the configuration details as follows.

\textbf{RA Configuration.} For all three datasets, we use 5 samples for the BadNet, WaNet, IAD, and ISSBA attacks, and 1 sample for the Blended attack. 
For the Refool, SIG, Narcissus, Adaptive-Blend, and Adaptive-Patch attacks, we adopt 2 samples for CIFAR-10 and Tiny-ImageNet-200, while using 5 samples for GTSRB.
To mitigate the risk that directly fine-tuning the purified model with poisoned samples may degrade clean accuracy, we incorporate an additional 0.02\% of 
the overall training dataset as benign samples during the fine-tuning process.

\textbf{QRA Configuration.} We employ a three-layer Multilayer Perceptron (MLP) to generate reversed perturbations in the input space. Each hidden layer contains 
1024 neurons and is followed by a Rectified Linear Unit (ReLU) activation function. To train the generator, we use 500 benign examples and 500 backdoored examples 
from the CIFAR-10 dataset. The training is conducted for 50 epochs using a fixed learning rate of 0.1 and a batch size of 128. The balancing coefficient $\alpha$ 
is set to 0.2 across all attack types to ensure high Q-ASR on purified models while simultaneously reducing the attack performance on clean models.
Additionally, the perturbation budget is fixed at 16/255 for all experiments.

\subsection{Defense Settings}
\label{subsection:defense_settings}

In addition, as a comparative reference, we construct an idealized backdoor unlearning scenario (termed IBU), in which we assume that the defender fine-tunes 
the model using real backdoor poisoned samples with original labels, alongside a small subset of clean data $\mathcal{D}_{cl}$, to maintain the model’s performance 
on clean tasks, as follows: 
\begin{equation}
    \label{eq:ideal_backdoor_unlearning}
    \min_{\theta} \mathcal{R}_{cl} + \lambda \mathcal{R}_{bu}
\end{equation}
where
\begin{equation}
	\left\{\begin{matrix}
		\mathcal{R}_{cl} = \mathbb{E}_{(x,y)\sim\mathcal{D}_{cl}} (\ell(f(x;\theta),y)),\\
		\mathcal{R}_{bu} = \mathbb{E}_{(x,y)\sim\mathcal{D}_{-y_t}} (\ell(f((1-m) \odot x + m \odot \Delta ;\theta),y)).
	\end{matrix}\right.
\end{equation}
A poisoned sample is defined as $(1-m) \odot x + m \odot \Delta $, where $\odot$ denotes the element-wise product, $\Delta$ is the backdoor trigger, and 
the binary mask $m \in \{0,1\}^n$ specifies its location and intensity. The coefficients $\lambda$ balance clean task performance and the backdoor unlearning objective. 
In our experiments, we set $\lambda =1.0$.

BI-BAU employs projected gradient descent (PGD) attack~\cite{cubuk2017intriguing} to generate adversarial examples in the inner maximization step and subsequently unlearns them to mitigate the backdoor effects in the outer minimization step.
\textbf{(i) In the inner optimization step}, we set the parameters involved in $\mathcal{L}_{adv\text{-}in}$ as $\lambda_1 = 1.0$ and $\lambda_2 = 1.0$.
For the data likelihood loss $\ell_{kd}$, we use $\alpha = 0.2$ and $T = 1.0$. Regarding the parameter $k$ in the prior loss term $\ell_{prior}$, we set $k = 0.8$ for general attacks.
However, for attacks characterized by low orthogonality or low linearity, we set $k = 0.4$ for Refool, and $k = 0.2$ for SIG, Narcissus, Adaptive-Blend, and Adaptive-Patch.
\textbf{(ii) In the outer optimization step}, for the loss $\mathcal{L}_{adv\text{-}out}$, we set the parameters $\gamma_1$ and $\gamma_2$ as follows. 
we use $\gamma_1 = 1.0$, and $\gamma_2 = 1.0$, with a learning rate of $0.0001$ for general backdoor attacks. 
In contrast, for attacks with low orthogonality or low linearity, we adopt $\gamma_1 = 1.0$, and $\gamma_2 = 0.01$ and set the learning rate to $0.0005$.
In our method, adversarial examples are generated using Projected Gradient Descent (PGD) under the  $L_{\infty }$ norm constraint. 
Specifically, we employ 20 steps with a step size of 2.0 and a perturbation bound of 25.0 across all experiments.
All perturbations are applied in the input pixel space, where pixel values range from 0 to 255.

\subsection{Layer-wise Activation Analysis}
\label{appendix:layer-wise_activation_analysis}
To further verify whether ``complete unlearning'' is truly achieved, we adopt and extend additional neural-level evaluation metrics, such as the Backdoor Existence 
Coefficient (BEC) introduced by \citet{zhu2024breaking}.
An effective indicator should be able to quantify the similarity of backdoor effects between the backdoored model and the target defense model across the entire network, 
thereby validating that BI-BAU completely eliminates the backdoor effects embedded in the original backdoored model and achieves the claimed backdoor defense goal 
of ``complete unlearning''. The implementation details of the Backdoor Existence Coefficient (BEC) are described as follows. Backdoor Existence Coefficient (BEC), 
which is calculated through the following three steps~\cite{zhu2024breaking}:
\begin{enumerate}
    \item \textbf{Backdoor neuron identification}: Firstly, we need to identify backdoor-related neurons. \citet{zheng2022data} proposed Trigger-activated Change (TAC) 
	to quantify the correlation between backdoor impact and neurons.
    With this metric, backdoor-related neurons in $f_{\vtheta_\text{A}}$ are identified for each layer. Thus, the feature maps corresponding to these neuron indices 
	are selected for each model, denoted as $\tilde{m}^{(l)}_{\text{A}}(\x_{\bm{\xi}})$, $\tilde{m}^{(l)}_{\text{D}}(\x_{\bm{\xi}})$, and $\tilde{m}^{(l)}_{\text{C}}(\x_{\bm{\xi}})$, 
	respectively. Denote the feature maps across dataset $\mathcal{D}_p$ as $\tilde{m}^{(l)}(\mathcal{D}_p)\in \mathbb{R}^{n_p \times (\tilde{c}_l\times h_l \times w_l)}$.
    \item \textbf{Backdoor effect similarity metric}: In order to measure the backdoor effect similarity between models, we employ Centered Kernel Alignment (CKA)~\cite{kornblith2019similarity}
	to quantify the similarity between these matrices. 
	The similarity in backdoor effects between $f_{\vtheta_\text{D}}$ and $f_{\vtheta_\text{A}}$, calculated through the use of corresponding features, can be computed as:
    \begin{equation}
    S_{\text{D},\text{A}}^{(l)}(\mathcal{D}_p) = \text{CKA}\left(\tilde{m}_{\text{D}}^{(l)}(\mathcal{D}_p), \tilde{m}_{\text{A}}^{(l)}(\mathcal{D}_p)\right),
    \end{equation}
     and $S_{\text{C},\text{A}}^{(l)}(\mathcal{D}_p)$ is computed accordingly.
     \item \textbf{Backdoor existence coefficient computation}: The BEC is the average of normalized backdoor effect similarity across all layers. By assigning the BEC 
	 of $f_{\vtheta_\text{A}}$ a value of 1 and $f_{\vtheta_\text{C}}$ a value of 0, the computation can proceed as follows:
    \begin{equation}
    	\rho_{\text{BEC}} (f_{\vtheta_{\text{D}}},f_{\vtheta_{\text{A}}},f_{\vtheta_{\text{C}}};\mathcal{D}_p) = \frac{1}{N}\sum_{l=1}^{N} \frac{S_{\text{D},\text{A}}^{(l)}(\mathcal{D}_{p})-S_{\text{C},\text{A}}^{(l)}(\mathcal{D}_{p})}{S_{\text{A},\text{A}}^{(l)}(\mathcal{D}_{p})-S_{\text{C},\text{A}}^{(l)}(\mathcal{D}_{p})}.
	\end{equation}
\end{enumerate}
Denote $\rho_{\text{BEC}} (f_{\vtheta_{\text{D}}},f_{\vtheta_{\text{A}}},f_{\vtheta_{\text{C}}};\mathcal{D}_p)$ as $\rho_{\text{BEC}} (f_{\vtheta_{\text{D}}})$ for simplicity. 
\citet{zhu2024breaking} utilize $\text{BEC}$ to signify backdoor existence of the target defense model (\ie the purified model) $f_{\vtheta_\text{D}}$.

However, $\rho_{\text{BEC}}$ may occasionally take values below 0. To address this, we apply a sigmoid transformation:
\begin{equation}
	\sigma(\rho_{\text{BEC}}) = \frac{1}{1 + \text{exp}(-k(\rho_{\text{BEC}}-\rho_{c}))}
\end{equation}
where $k > 0$ controls the steepness of the curve, $\rho_{c}$ represents the center of the curve, and $\sigma(\rho_{\text{BEC}}) = 0.5$ when $\rho_{\text{BEC}} = \rho_{c}$.
The function $\sigma$ maps $\rho_{\text{BEC}}$ into the range $[0,1.0]$, where a higher value of $\sigma(\rho_{\text{BEC}})$ indicates a stronger presence of backdoors 
in the model.

In the experiments, we set $k = 4.0$ and $\rho_{c} = 0.5$. Considering that the shallow layers in neural networks generally capture more general features and are 
more susceptible to noise~\cite{semenova2022understanding}, we only used the last 10 convolutional layers (20 layers in total for ResNet-18).
This metric serves as a complementary evaluation to ASR, enabling quantification of residual backdoor effects at the layer-wise level within the neural network.

\section{Additional Experiments}
\label{appendix:additional_experiments}
\subsection{Validating Assumption~\ref{assumption:backdoor_unlearning_as_continual_learning}}
\label{subsection: validating_assumption_1}

Prior study~\cite{li2021anti} views backdoor learning as a composition of a clean task $\tau_c$ and a backdoor task $\tau_b$. 
\citet{zhang2024exploring} further characterize this process as a two-stage task flow: $\mathcal{T}=\{\tau_b, \tau_c\}$, 
comprising two independent learning tasks:\ding{172} an initial rapid learning phase in which the backdoor task is acquired within just a few training epochs, followed by
\ding{173} a subsequent phase of gradual adaptation to the clean task.

\begin{figure*}[ht]
	\graphicspath{{figures/continual_learning/}}
	\centering
	\begin{subfigure}[b]{0.30\textwidth}
		\centering
		\includegraphics[width=\linewidth]{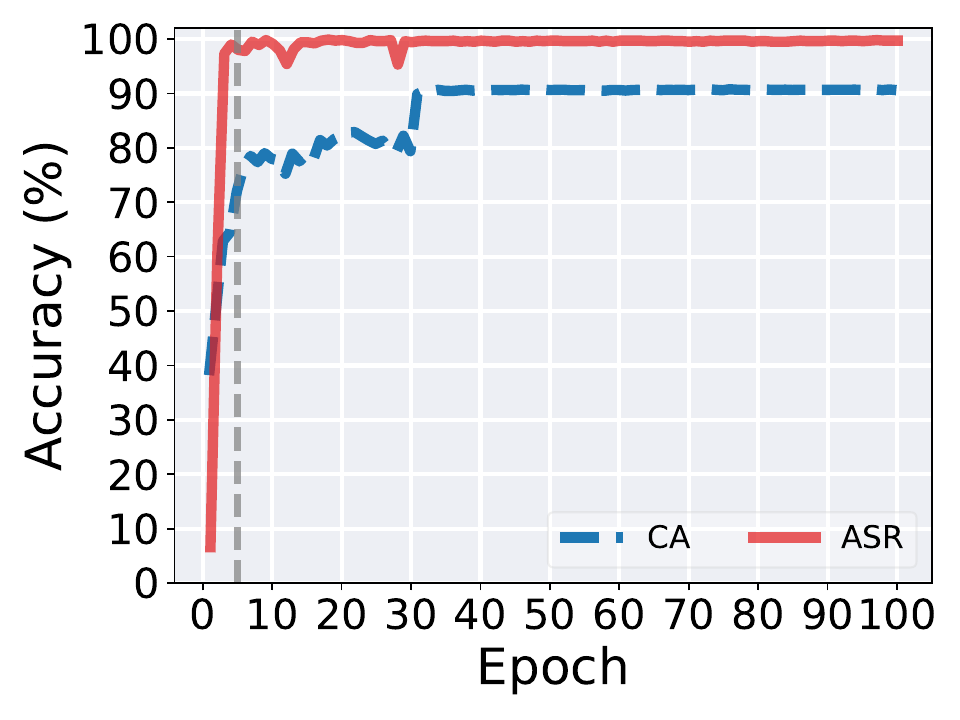}
		\caption{Joint Learning}
	\end{subfigure}
	\begin{subfigure}[b]{0.30\textwidth}
		\centering
		\includegraphics[width=\linewidth]{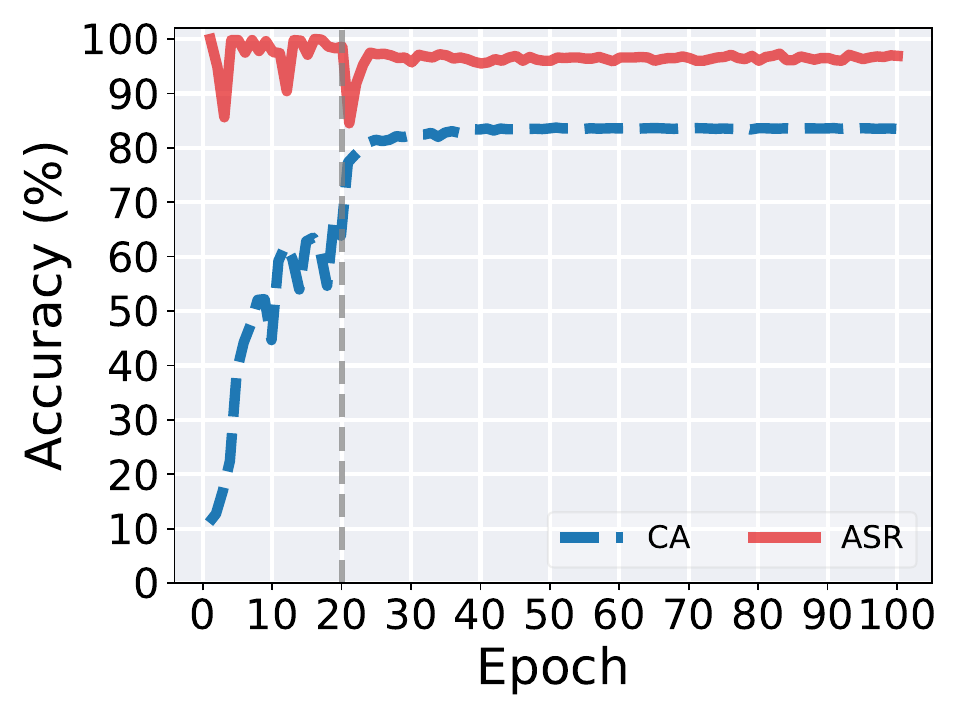}
		\caption{$\mathcal{T}=\{\tau_b, \tau_c\}$}
	\end{subfigure}
	\begin{subfigure}[b]{0.30\textwidth}
		\centering
		\includegraphics[width=\linewidth]{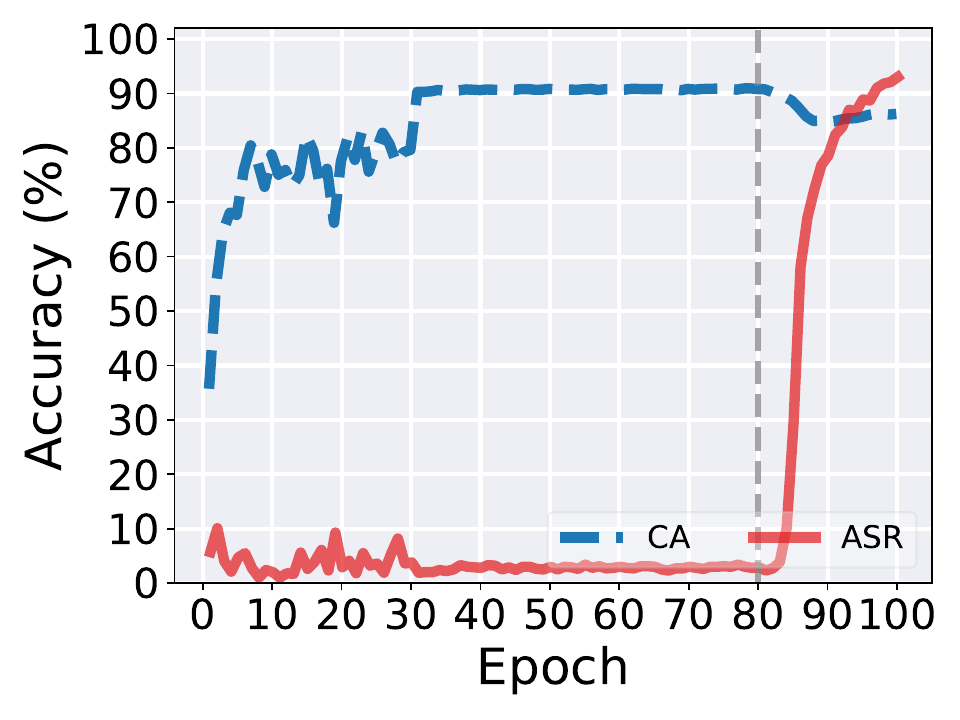}
		\caption{$\mathcal{T}=\{\tau_c,\tau_b\}$}
	\end{subfigure}
	\caption{\small Illustration of backdoor learning as a task flow.}
	\label{fig:backdoor_learning_as_CL}
\end{figure*}
To validate the plausibility of Assumption~\ref{assumption:backdoor_unlearning_as_continual_learning}, we compare three training settings: 
(i) the standard backdoor learning setup, (ii) a task stream $\mathcal{T} = \{\tau_b, \tau_c\}$, and (iii) a task stream $\mathcal{T} = \{\tau_c, \tau_b\}$.
Following the continual learning paradigm, we simulate the task streams using a ResNet-18 architecture on the CIFAR-10 dataset, with BadNets as a representative 
attack method. Since the clean task $\tau_c$ and backdoor task $\tau_b$ share the same classifier (\ie the fully connected (FC) classification head), their parameters 
inevitably overlap, making complete isolation between tasks is impractical. To approximate independent training on $D_c$ and $D_p$, we adopt a rehearsal strategy 
by introducing a few samples from the complementary task. This approach is supported by prior work~\cite{rolnick2019experience} in continual learning, which helps 
mitigate interference. Notably, the inclusion of a small number of samples is insufficient to achieve high performance, and is intended solely to approximate 
task independence.
\begin{enumerate}[(\roman*)]
    \item In the standard backdoor learning setting, the model is trained jointly on the poisoned dataset $D_p$ and clean dataset $D_c$, with backdoor and clean tasks 
    learned simultaneously, as illustrated in Figure~\ref{fig:backdoor_learning_as_CL} (a).
    \item For the task stream $\mathcal{T} = \{\tau_b, \tau_c\}$,  the training process is divided into two sequential phases. In the first phase, the model is 
    trained for 20 epochs based on a mixed dataset composed of the full poisoned dataset $D_p$ and a small fraction (10\%) of clean data.
    In the second stage, the model is further trained on the full clean dataset $D_c$ with only 1\% of poisoned subset $D_p$, continuing for 80 epochs, as illustrated in Figure~\ref{fig:backdoor_learning_as_CL} (b).
    \item For the task stream $\mathcal{T} = \{\tau_c, \tau_b\}$, we adopt a similar two-stage setup. The first stage involves training on the full clean dataset $D_c$ 
    along with a small set of poisoned samples (10 instances, approximately 0.2\% of $D_p$). These few samples are included only to approximate task independence 
	and have negligible impact on the overall results. In the second stage, the model is exposed to the full poisoned dataset $D_p$ and 10\% of the clean data, 
	continuing training until a total of 100 epochs is reached, as illustrated in Figure~\ref{fig:backdoor_learning_as_CL} (c).
\end{enumerate}

As shown in Figure~\ref{fig:backdoor_learning_as_CL} (a), in the standard backdoor learning setting, the backdoor is rapidly injected into the model, evidenced 
by the attack success rate (ASR) peaking within the first 5 training epochs. Subsequently, the clean task $\tau_c$ is gradually learned as training progresses 
until convergence, while the ASR remains nearly unchanged. This observation is consistent with the findings reported by \citet{zhang2024exploring}.
In Figure~\ref{fig:backdoor_learning_as_CL} (b), we manually partition the dataset to further decouple the backdoor task $\tau_b$ and the clean task $\tau_c$. 
Following a two-stage learning procedure, we observe similar dynamics: the backdoor task is quickly learned in the initial phase, and upon switching to the clean task $\tau_c$, 
the clean accuracy (CA) gradually improves while the ASR remains stable. Finally, model achieves a comparable CA and ASR to that of the standard backdoor learning setup.
In contrast, Figure~\ref{fig:backdoor_learning_as_CL} (c) presents the reversed task stream $\mathcal{T} = \{\tau_c, \tau_b\}$, where the clean task $\tau_c$ is 
gradually learned in the first stage, followed by a rapid acquisition of the backdoor task $\tau_b$, without significant degradation in CA.
This again results in a similar convergence outcome, suggesting that tasks $\tau_b$ and $\tau_c$ can be interchanged without loss of generality, 
further supporting the plausibility of Assumption~\ref{assumption:backdoor_unlearning_as_continual_learning}.

Overall, these experiments empirically validate that the clean and backdoor tasks can be modeled as a sequential task flow, supporting our continual learning 
formulation for backdoor unlearning.

\renewcommand{\arraystretch}{1.5}
\begin{table*}[!h]
	\centering
	\caption{Performance (\%) of safety tuning strategies for eliminating backdoor effects on the Tiny-ImageNet-200 dataset.}
	\label{table:backdoor_unlearning_on_Tiny-ImageNet}
	\resizebox{0.98\textwidth}{!}{\begin{tabular}{c|cc|cc|cc|cc|cc|cc|cc|cc|cc}
		\toprule
		Attack $\to$ & \multicolumn{2}{c|}{BadNets} & \multicolumn{2}{c|}{Blended} & \multicolumn{2}{c|}{WaNet} & \multicolumn{2}{c|}{Refool} & \multicolumn{2}{c|}{IAD} & \multicolumn{2}{c|}{ISSBA} & \multicolumn{2}{c|}{SIG} & \multicolumn{2}{c|}{Narcissus} & \multicolumn{2}{c}{Average}\\ 
		\midrule
		Defense $\downarrow$ & \multicolumn{1}{c}{CA}&\multicolumn{1}{c|}{ASR} & \multicolumn{1}{c}{CA}&\multicolumn{1}{c|}{ASR} & \multicolumn{1}{c}{CA}&\multicolumn{1}{c|}{ASR} & \multicolumn{1}{c}{CA}&\multicolumn{1}{c|}{ASR} & \multicolumn{1}{c}{CA}&\multicolumn{1}{c|}{ASR} & \multicolumn{1}{c}{CA}&\multicolumn{1}{c|}{ASR} & \multicolumn{1}{c}{CA}&\multicolumn{1}{c|}{ASR} & \multicolumn{1}{c}{CA}&\multicolumn{1}{c|}{ASR} & \multicolumn{1}{c}{CA}&\multicolumn{1}{c}{ASR} \\ 
		\midrule
		
        No Defense  & 48.80&99.08 & 50.08&100.00 & 45.42&99.90 & 49.80&99.90 & 40.53&100.00 & 40.50&99.80 & 51.80&96.14 & 50.68&100.00 & 47.20&99.35\\
        \hline
		NC          & 45.86&1.02 & \scalebox{1.00}{\colorbox[HTML]{d9f2d9}{47.50}}&4.00 & 43.12&1.00 & 47.52&\scalebox{1.00}{\colorbox[HTML]{FFCCCC}{70.40}} & 39.55&1.00 & 36.26&0.90 & 48.15&\scalebox{1.00}{\colorbox[HTML]{FFCCCC}{93.78}} & 47.08&\scalebox{1.00}{\colorbox[HTML]{d9f2d9}{0.00}} &  44.38&21.51 \\
        \hline
		BTI-DBF     & 43.72&\scalebox{1.00}{\colorbox[HTML]{d9f2d9}{0.12}} & 45.68&\scalebox{1.00}{\colorbox[HTML]{d9f2d9}{0.00}} & 41.80&\scalebox{1.00}{\colorbox[HTML]{d9f2d9}{0.10}} & 46.47&\scalebox{1.00}{\colorbox[HTML]{d9f2d9}{0.00}}& 37.46&\scalebox{1.00}{\colorbox[HTML]{d9f2d9}{0.00}} & 31.53&\scalebox{1.00}{\colorbox[HTML]{d9f2d9}{0.00}} & 40.17&11.01 & 42.40&\scalebox{1.00}{\colorbox[HTML]{d9f2d9}{0.00}} &  41.15&\scalebox{1.00}{\colorbox[HTML]{d9f2d9}{1.40}}\\
		\hline
		ANP         & 46.28&0.18 & 47.66&17.80 & 41.44&\scalebox{1.00}{\colorbox[HTML]{FFCCCC}{65.20}} & 47.17&\scalebox{1.00}{\colorbox[HTML]{FFCCCC}{20.20}} & \scalebox{1.00}{\colorbox[HTML]{d9f2d9}{40.53}}&1.10 & 39.45&2.80 & 48.03&\scalebox{1.00}{\colorbox[HTML]{FFCCCC}{91.83}}& \scalebox{1.00}{\colorbox[HTML]{d9f2d9}{48.24}}&\scalebox{1.00}{\colorbox[HTML]{FFCCCC}{98.80}} & 44.85&37.23    \\
        \hline
        EP          & 48.56&0.44 & 49.76&8.20 & \scalebox{1.00}{\colorbox[HTML]{d9f2d9}{44.23}}&1.40 &\scalebox{1.00}{\colorbox[HTML]{d9f2d9}{48.63}}&\scalebox{1.00}{\colorbox[HTML]{FFCCCC}{92.80}} & 40.01&0.30 & 39.78&0.30 & \scalebox{1.00}{\colorbox[HTML]{d9f2d9}{50.12}}&\scalebox{1.00}{\colorbox[HTML]{FFCCCC}{96.52}} & \scalebox{1.00}{\colorbox[HTML]{d9f2d9}{48.24}}&\scalebox{1.00}{\colorbox[HTML]{FFCCCC}{99.96}} &  \scalebox{1.00}{\colorbox[HTML]{d9f2d9}{46.16}}&37.49\\
		\hline
		RNP    		& \scalebox{1.00}{\colorbox[HTML]{d9f2d9}{53.98}}&0.20 & 45.90&1.20 & 43.67&\scalebox{1.00}{\colorbox[HTML]{d9f2d9}{0.00}} & 48.05&\scalebox{1.00}{\colorbox[HTML]{d9f2d9}{0.00}} & 36.52&\scalebox{1.00}{\colorbox[HTML]{d9f2d9}{0.10}} & 18.98&0.20 & 39.34&\scalebox{1.00}{\colorbox[HTML]{FFCCCC}{96.82}} & 41.42&\scalebox{1.00}{\colorbox[HTML]{FFCCCC}{28.66}} &  40.98&15.89 \\
        \hline
        CLP    		& 50.52&1.74 & 46.71&\scalebox{1.00}{\colorbox[HTML]{FFCCCC}{94.40}} & 30.66&\scalebox{1.00}{\colorbox[HTML]{FFCCCC}{78.30}} & 41.30&\scalebox{1.00}{\colorbox[HTML]{FFCCCC}{99.80}} & 35.07&4.20 & 35.24&1.40 & 48.45&\scalebox{1.00}{\colorbox[HTML]{FFCCCC}{98.15}} & 36.20&\scalebox{1.00}{\colorbox[HTML]{FFCCCC}{100.00}} & 40.51&59.74 \\
		\hline
		I-BAU       & 35.94&\scalebox{1.00}{\colorbox[HTML]{FFCCCC}{38.86}} & 37.23&\scalebox{1.00}{\colorbox[HTML]{FFCCCC}{74.50}} & 37.36&\scalebox{1.00}{\colorbox[HTML]{FFCCCC}{78.70}} & 37.02&\scalebox{1.00}{\colorbox[HTML]{FFCCCC}{31.80}}& 35.74&1.20 & 34.10&0.10 & 37.75&\scalebox{1.00}{\colorbox[HTML]{FFCCCC}{86.19}} & 40.86&\scalebox{1.00}{\colorbox[HTML]{FFCCCC}{60.26}}&  37.00&46.45\\
		\hline
		SAU         & 40.02&3.30 & 40.94&\scalebox{1.00}{\colorbox[HTML]{FFCCCC}{20.30}} & 35.92&\scalebox{1.00}{\colorbox[HTML]{FFCCCC}{20.70}} & 40.11&\scalebox{1.00}{\colorbox[HTML]{FFCCCC}{22.80}} & 31.28&16.90 & 31.41&5.50 & 42.56&\scalebox{1.00}{\colorbox[HTML]{FFCCCC}{84.90}} & 40.54&1.30 &  37.84&21.96\\
		\hline
        BI-BAU (T)  & 41.14&4.34 & 32.80&0.50 & 36.80&6.80 & 32.36&0.40 & 36.70&11.40 & 35.97&1.50 & 30.68&\scalebox{1.00}{\colorbox[HTML]{FFCCCC}{21.06}} & 30.50&0.06 & 34.61&5.75 \\
		\hline
        BI-BAU (U) & 44.60&1.18 & 46.73&4.00 & 40.55&8.00 & 45.28&0.50 & 39.18&1.10 & \scalebox{1.00}{\colorbox[HTML]{d9f2d9}{40.75}}&\scalebox{1.00}{\colorbox[HTML]{d9f2d9}{0.00}} & 35.70&\scalebox{1.00}{\colorbox[HTML]{d9f2d9}{8.50}} & 35.80&11.80 & 39.78&1.90\\
        \bottomrule
	\end{tabular}}
\end{table*} 

\subsection{Investigation of the Hyperparameter $k$ in the Prior Loss $\ell_{prior}$}
\label{appendix:Investigation_of_hyperparameter_k}

\begin{figure}[H]
	\graphicspath{{figures/Hyperparameter_k/}}
	\centering
	\begin{subfigure}[b]{0.23\textwidth}
		\centering
		\includegraphics[width=\linewidth]{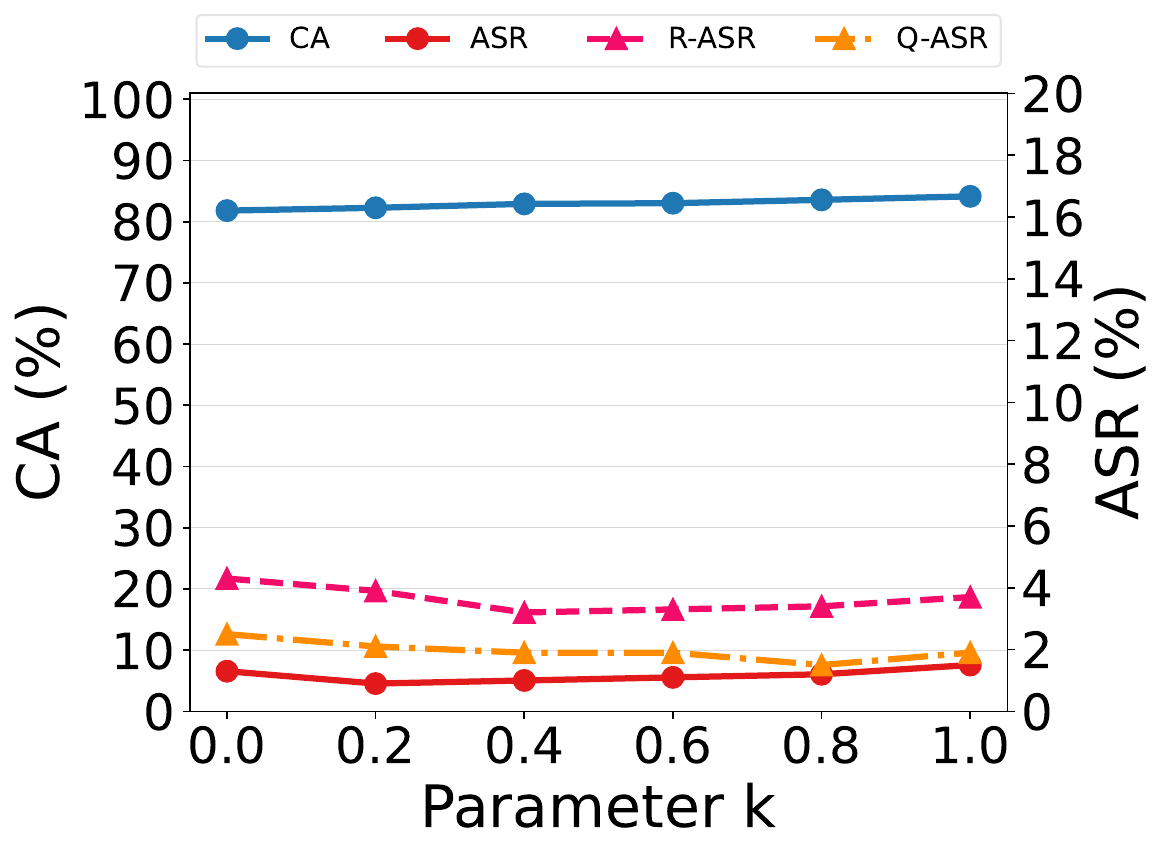}
		\caption{BadNets}
	\end{subfigure}
	\begin{subfigure}[b]{0.23\textwidth}
		\centering
		\includegraphics[width=\linewidth]{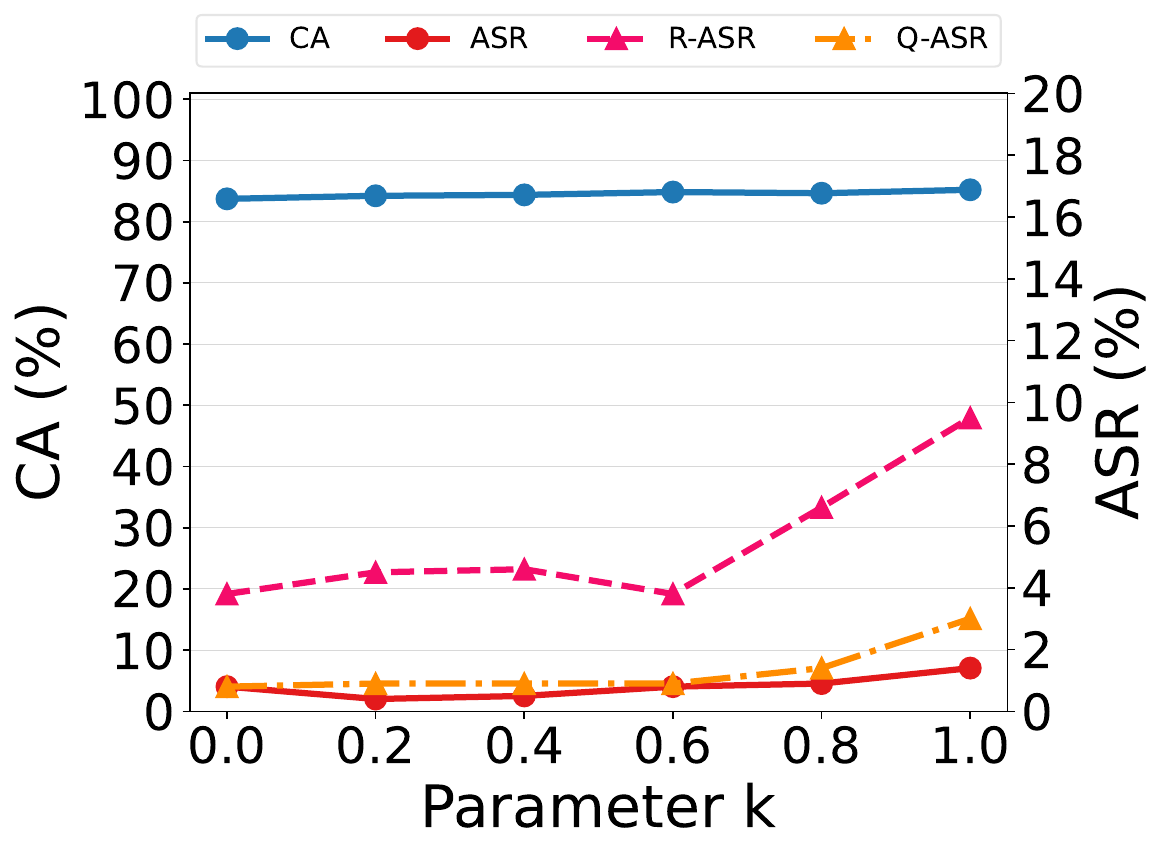}
		\caption{Refool}
	\end{subfigure}
	\begin{subfigure}[b]{0.23\textwidth}
		\centering
		\includegraphics[width=\linewidth]{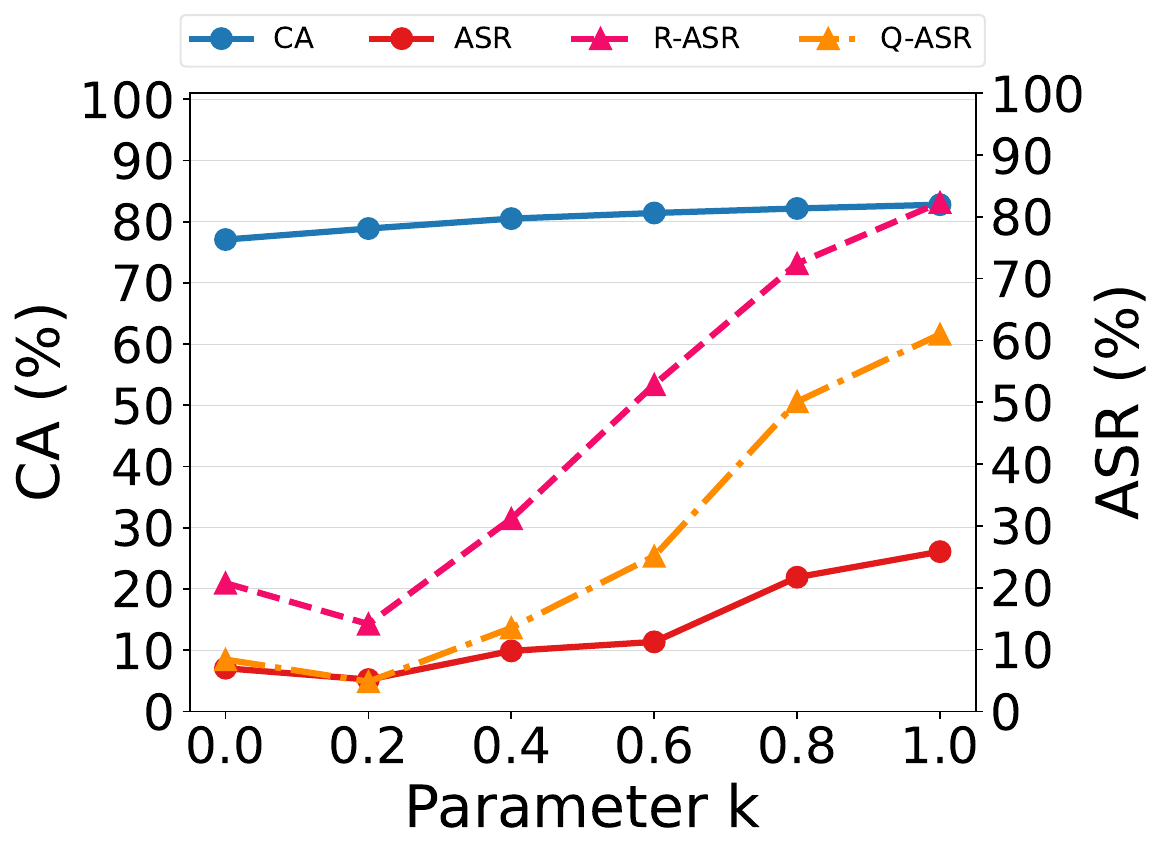}
		\caption{Narcissus}
	\end{subfigure}
    \begin{subfigure}[b]{0.23\textwidth}
		\centering
		\includegraphics[width=\linewidth]{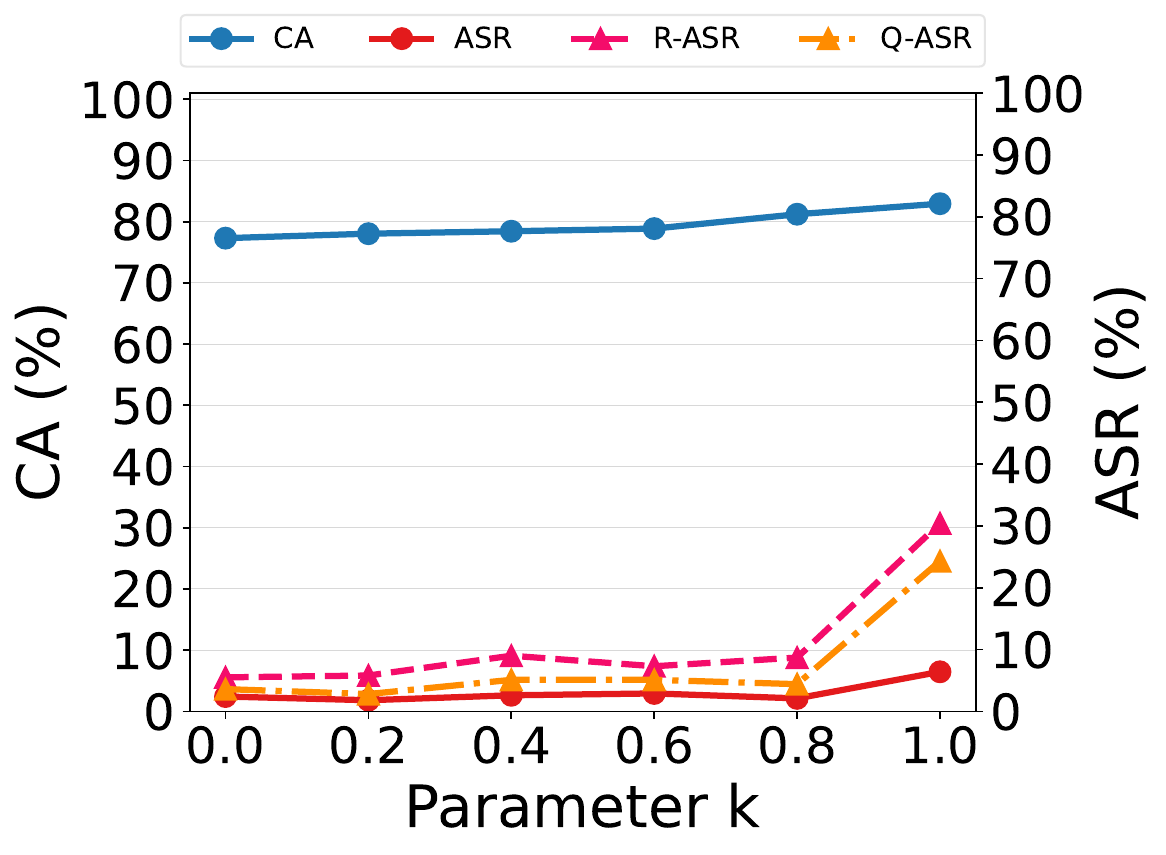}
		\caption{Adaptive-Blend}
	\end{subfigure}
	\caption{\small Evolution of CA and ASR with varying hyperparameter $k$ in the prior loss $\ell_{prior}$.}
	\label{fig:Investigation_Hyperparameter_k}
\end{figure}

As discussed in Section~\ref{section:proposed_method}, the hyperparameter $k$ in the prior loss $\ell_{prior}$ controls the trade-off between preserving clean features 
and aligning with backdoor features. Specifically, smaller values of $k$ bias $\tilde{x}$ toward backdoor-like representations, potentially at the expense of retaining 
clean semantics, while larger values favor clean feature preservation. To empirically validate this trade-off, we conduct experiments across multiple backdoor attacks.

Figure~\ref{fig:Investigation_Hyperparameter_k} shows that increasing $k$ from 0.0 to 1.0 gradually improves clean accuracy (CA), but simultaneously tends to 
increase attack success rates (ASR, R-ASR, QASR), in agreement with our theoretical analysis. This trade-off is especially pronounced for attacks exhibiting 
low orthogonality or linearity, such as Narcissus and Adaptive-Blend, where embedded backdoors are intrinsically more challenging to remove. In such scenarios, 
smaller $k$ values are preferable, as they enhance the suppression of backdoor-related features while maintaining robust defense performance.

\subsection{Additional Experiments on Diverse Datasets and Models}
\label{appendix:experiments_on_diverse_datasets_and_models}
To further evaluate the effectiveness of our proposed BI-BAU method, we conduct additional experiments on two distinct settings: ResNet-18 on the Tiny-ImageNet-200 
dataset and VGG19 on the GTSRB dataset. The experimental results for Tiny-ImageNet-200 are summarized in Tables~\ref{table:backdoor_unlearning_on_Tiny-ImageNet} and~\ref{table:post_purification_robustness_on_Tiny-ImageNet},
while the results for GTSRB are presented in Tables~\ref{table:backdoor_mitigation_on_GTSRB} and~\ref{table:resistance_to_backdoor_re-activation_attack_on_GTSRB}.

\renewcommand{\arraystretch}{1.5}
\begin{table*}[htbp]
	\centering
	\caption{Post-purification robustness (\%) of safety tuning methods under backdoor re-activation attacks (ResNet-18, Tiny-ImageNet-200).}
	\label{table:post_purification_robustness_on_Tiny-ImageNet}
	\resizebox{0.98\textwidth}{!}{\begin{tabular}{c|c|c|c|c|c|c|c|c|c|c|c}
		\toprule
		Attack $\to$ & \multicolumn{1}{c|}{BadNets} & \multicolumn{1}{c|}{Blended} & \multicolumn{1}{c|}{WaNet} & \multicolumn{1}{c|}{Refool} & \multicolumn{1}{c|}{IAD} & \multicolumn{1}{c|}{ISSBA} & \multicolumn{1}{c|}{SIG} & \multicolumn{1}{c|}{Narcissus} &  \multicolumn{1}{c|}{Adap-Blend} & \multicolumn{1}{c|}{Adap-Patch} & \multicolumn{1}{c}{Average}\\ 
		\midrule
		\multirow{2}{*}{Defense $\downarrow$} & CA/ASR & CA/ASR & CA/ASR & CA/ASR & CA/ASR & CA/ASR & CA/ASR & CA/ASR & CA/ASR & CA/ASR & CA/ASR \\
		
		                                     & R-ASR/Q-ASR  & R-ASR/Q-ASR  & R-ASR/Q-ASR  & R-ASR/Q-ASR  & R-ASR/Q-ASR  & R-ASR/Q-ASR  & R-ASR/Q-ASR  & R-ASR/Q-ASR  & R-ASR/Q-ASR  & R-ASR/Q-ASR  & R-ASR/Q-ASR  \\
		\midrule
		\multirow{2}{*}{No Defense}     & 49.82/98.34 & 50.08/100.00 & 45.42/99.90 & 49.80/99.90 & 40.53/100.00 & 40.50/99.80 & 51.80/96.14 & 49.80/99.74 & 50.13/93.70 & 50.33/98.20 & 47.82/98.57   \\
		
								 		& -/-  & -/- & -/-  & -/- & -/- & -/- & -/-  & -/- & -/-  &-/- & -/-  \\   
		\hline
		\multirow{2}{*}{IBU}     & 32.12/0.00 & 32.90/0.00 & 40.11/0.00 & 43.42/0.00 & 36.88/0.00 & 34.04/0.00 & 46.36/0.00 & 31.38/0.00  & 44.64/0.00 & 44.91/0.00 & 38.67/0.00   \\
		
								 & 0.00/0.00  & 42.30/4.10 & 0.00/0.00  & 0.10/0.00 & 0.00/0.00  & 0.00/0.00 & 0.00/0.02  & 0.00/0.00  & 0.00/0.00  & 0.00/0.00 & 4.24/0.41  \\   
		\bottomrule
		\toprule
		\multirow{2}{*}{BTI-DBF} 
			& 39.44/0.20 
			& 45.70/0.00 
			& \scalebox{1.00}{\colorbox[HTML]{d9f2d9}{41.90}}/\scalebox{1.00}{\colorbox[HTML]{d9f2d9}{0.10}} 
			& 46.47/\scalebox{1.00}{\colorbox[HTML]{d9f2d9}{0.00}} 
			& 37.47/\scalebox{1.00}{\colorbox[HTML]{d9f2d9}{0.00}} 
			& 31.53/\scalebox{1.00}{\colorbox[HTML]{d9f2d9}{0.00}} 
			& 40.15/11.05 
			& 42.42/\scalebox{1.00}{\colorbox[HTML]{d9f2d9}{0.00}} 
			& 39.17/5.60 
			& 35.48/7.50 
			& 39.97/2.44   \\
		
			& 96.36/36.28 
			& 86.60/0.00 
			& \scalebox{1.00}{\colorbox[HTML]{FFCCCC}{88.60}}/\scalebox{1.00}{\colorbox[HTML]{FFCCCC}{36.90}} 
			& \scalebox{1.00}{\colorbox[HTML]{FFCCCC}{99.50}}/\scalebox{1.00}{\colorbox[HTML]{FFCCCC}{16.90}} 
			& 1.90/0.10  
			& \scalebox{1.00}{\colorbox[HTML]{d9f2d9}{0.10}}/0.10 
			& \scalebox{1.00}{\colorbox[HTML]{FFCCCC}{98.69}}/\scalebox{1.00}{\colorbox[HTML]{FFCCCC}{85.83}} 
			& \scalebox{1.00}{\colorbox[HTML]{FFCCCC}{96.80}}/\scalebox{1.00}{\colorbox[HTML]{FFCCCC}{0.00}} 
			& \scalebox{1.00}{\colorbox[HTML]{FFCCCC}{97.90}}/\scalebox{1.00}{\colorbox[HTML]{FFCCCC}{91.70}} 
			& \scalebox{1.00}{\colorbox[HTML]{FFCCCC}{93.10}}/\scalebox{1.00}{\colorbox[HTML]{FFCCCC}{69.70}} 
			& \scalebox{1.00}{\colorbox[HTML]{FFCCCC}{75.95}}/\scalebox{1.00}{\colorbox[HTML]{FFCCCC}{33.75}} \\   
		\hline

		\multirow{2}{*}{ANP}     &\scalebox{1.00}{\colorbox[HTML]{d9f2d9}{46.78}}/2.14 & \scalebox{1.00}{\colorbox[HTML]{d9f2d9}{47.66}}/17.80 & 41.56/65.20 & 47.17/20.20 & \scalebox{1.00}{\colorbox[HTML]{d9f2d9}{40.53}}/1.10 & \scalebox{1.00}{\colorbox[HTML]{d9f2d9}{39.45}}/2.80 & \scalebox{1.00}{\colorbox[HTML]{d9f2d9}{48.03}}/91.83 & \scalebox{1.00}{\colorbox[HTML]{d9f2d9}{48.24}}/98.86 & \scalebox{1.00}{\colorbox[HTML]{d9f2d9}{46.45}}/6.30 & \scalebox{1.00}{\colorbox[HTML]{d9f2d9}{47.36}}/41.10 & \scalebox{1.00}{\colorbox[HTML]{d9f2d9}{45.32}}/34.73 \\

								 & 91.50/11.82 & 99.60/48.30 & \scalebox{1.00}{\colorbox[HTML]{FFCCCC}{99.60}}/\scalebox{1.00}{\colorbox[HTML]{FFCCCC}{79.90}} & \scalebox{1.00}{\colorbox[HTML]{FFCCCC}{98.60}}/\scalebox{1.00}{\colorbox[HTML]{FFCCCC}{55.00}} & 96.40/31.00  & 4.80/6.80 & \scalebox{1.00}{\colorbox[HTML]{FFCCCC}{97.26}}/\scalebox{1.00}{\colorbox[HTML]{FFCCCC}{92.60}} & \scalebox{1.00}{\colorbox[HTML]{FFCCCC}{99.30}}/\scalebox{1.00}{\colorbox[HTML]{FFCCCC}{99.40}} & \scalebox{1.00}{\colorbox[HTML]{FFCCCC}{59.10}}/\scalebox{1.00}{\colorbox[HTML]{FFCCCC}{24.10}} & \scalebox{1.00}{\colorbox[HTML]{FFCCCC}{74.20}}/\scalebox{1.00}{\colorbox[HTML]{FFCCCC}{59.30}} & \scalebox{1.00}{\colorbox[HTML]{FFCCCC}{82.03}}/\scalebox{1.00}{\colorbox[HTML]{FFCCCC}{50.82}}  \\   
		\hline

		\multirow{2}{*}{RNP}     & \scalebox{1.00}{\colorbox[HTML]{d9f2d9}{53.98}}/\scalebox{1.00}{\colorbox[HTML]{d9f2d9}{0.20}} & 45.90/ \scalebox{1.00}{\colorbox[HTML]{d9f2d9}{1.20}} & 43.67/\scalebox{1.00}{\colorbox[HTML]{d9f2d9}{0.00}} &  \scalebox{1.00}{\colorbox[HTML]{d9f2d9}{48.05}}/\scalebox{1.00}{\colorbox[HTML]{d9f2d9}{0.00}} & 36.52/0.10 & 18.98/0.20 & 39.34/96.82 & 41.42/28.66 & 44.26/83.80 & 41.07/1.60 & 41.31/21.25 \\

								 & 3.18/\scalebox{1.00}{\colorbox[HTML]{d9f2d9}{0.32}} & 5.10/ \scalebox{1.00}{\colorbox[HTML]{d9f2d9}{0.60}} & 18.50/ \scalebox{1.00}{\colorbox[HTML]{d9f2d9}{0.20}} & 60.90/ \scalebox{1.00}{\colorbox[HTML]{d9f2d9}{0.00}} & 14.10/ \scalebox{1.00}{\colorbox[HTML]{d9f2d9}{0.00}}  & 0.90/0.20 & \scalebox{1.00}{\colorbox[HTML]{FFCCCC}{87.17}}/\scalebox{1.00}{\colorbox[HTML]{FFCCCC}{85.69}} & \scalebox{1.00}{\colorbox[HTML]{FFCCCC}{100.00}}/\scalebox{1.00}{\colorbox[HTML]{FFCCCC}{99.96}} & \scalebox{1.00}{\colorbox[HTML]{FFCCCC}{80.50}}/\scalebox{1.00}{\colorbox[HTML]{FFCCCC}{78.00}} & 9.10/1.70 & 37.94/\scalebox{1.00}{\colorbox[HTML]{FFCCCC}{26.66}} \\   
		\hline
		
		\multirow{2}{*}{SAU}     & 40.98/1.94 & 40.94/20.30 & 36.07/20.70 & 40.10/22.80 & 31.28/16.90 & 31.40/5.50 & 42.58/84.90 & 40.54/1.30  & 41.62/1.30 & 41.44/33.50 & 39.69/20.91  \\
		
								 & 93.86/13.12 & 96.60/61.60 & \scalebox{1.00}{\colorbox[HTML]{FFCCCC}{97.90}}/\scalebox{1.00}{\colorbox[HTML]{FFCCCC}{96.40}} & \scalebox{1.00}{\colorbox[HTML]{FFCCCC}{98.40}}/\scalebox{1.00}{\colorbox[HTML]{FFCCCC}{63.80}} & 94.80/54.20  & 0.60/7.00 & \scalebox{1.00}{\colorbox[HTML]{FFCCCC}{98.13}}/\scalebox{1.00}{\colorbox[HTML]{FFCCCC}{87.78}} & \scalebox{1.00}{\colorbox[HTML]{FFCCCC}{99.76}}/\scalebox{1.00}{\colorbox[HTML]{FFCCCC}{76.52}} & \scalebox{1.00}{\colorbox[HTML]{FFCCCC}{83.90}}/\scalebox{1.00}{\colorbox[HTML]{FFCCCC}{57.60}} & \scalebox{1.00}{\colorbox[HTML]{FFCCCC}{91.20}}/\scalebox{1.00}{\colorbox[HTML]{FFCCCC}{73.20}} & \scalebox{1.00}{\colorbox[HTML]{FFCCCC}{85.51}}/\scalebox{1.00}{\colorbox[HTML]{FFCCCC}{59.12}}  \\   
		\hline
		\multirow{2}{*}{PAM}     & 43.66/2.24 & 25.05/34.20 & 39.93/5.30  & 31.73/89.90 & 32.38/4.10 & 33.06/4.70 & 35.56/\scalebox{1.00}{\colorbox[HTML]{d9f2d9}{1.16}} & 42.72/\scalebox{1.00}{\colorbox[HTML]{d9f2d9}{0.00}} & 42.87/6.50 & 39.21/\scalebox{1.00}{\colorbox[HTML]{d9f2d9}{0.10}} & 36.61/14.82  \\
								 & 1.34/2.52 & 76.00/38.80  & \scalebox{1.00}{\colorbox[HTML]{FFCCCC}{97.80}}/\scalebox{1.00}{\colorbox[HTML]{FFCCCC}{52.20}} & \scalebox{1.00}{\colorbox[HTML]{FFCCCC}{82.00}}/\scalebox{1.0}{\colorbox[HTML]{FFCCCC}{75.70}} & 9.80/5.40 & 0.70/5.10 & \scalebox{1.00}{\colorbox[HTML]{FFCCCC}{97.44}}/\scalebox{1.00}{\colorbox[HTML]{FFCCCC}{23.90}} & \scalebox{1.00}{\colorbox[HTML]{FFCCCC}{100.00}}/\scalebox{1.00}{\colorbox[HTML]{FFCCCC}{62.16}} & \scalebox{1.00}{\colorbox[HTML]{FFCCCC}{95.60}}/\scalebox{1.00}{\colorbox[HTML]{FFCCCC}{88.30}} & \scalebox{1.00}{\colorbox[HTML]{FFCCCC}{85.50}}/\scalebox{1.00}{\colorbox[HTML]{FFCCCC}{20.90}} & \scalebox{1.00}{\colorbox[HTML]{FFCCCC}{64.61}}/\scalebox{1.00}{\colorbox[HTML]{FFCCCC}{37.49}} \\   
		\hline
		\multirow{2}{*}{BI-BAU(T)} & 41.14/4.34 & 32.80/ \scalebox{1.00}{\colorbox[HTML]{d9f2d9}{0.50}} & 36.80/6.80 & 32.36/0.40 & 36.70/11.40 & 35.97/1.50 & 30.68/21.06 & 30.50/0.06  & 29.87/0.60 & 30.67/6.30 & 33.74/5.29   \\
		                           & 2.54/3.98 & \scalebox{1.00}{\colorbox[HTML]{d9f2d9}{4.00}}/\scalebox{1.00}{\colorbox[HTML]{d9f2d9}{0.70}} & \scalebox{1.00}{\colorbox[HTML]{d9f2d9}{14.00}}/5.80 & \scalebox{1.00}{\colorbox[HTML]{d9f2d9}{9.20}}/0.80 & \scalebox{1.00}{\colorbox[HTML]{d9f2d9}{9.00}}/11.80 & 0.70/0.90 & \scalebox{1.00}{\colorbox[HTML]{d9f2d9}{65.90}}/17.25 & \scalebox{1.00}{\colorbox[HTML]{d9f2d9}{1.18}}/\scalebox{1.00}{\colorbox[HTML]{d9f2d9}{0.04}} & \scalebox{1.00}{\colorbox[HTML]{d9f2d9}{9.50}}/\scalebox{1.00}{\colorbox[HTML]{d9f2d9}{0.40}} & \scalebox{1.00}{\colorbox[HTML]{d9f2d9}{19.30}}/\scalebox{1.00}{\colorbox[HTML]{d9f2d9}{5.20}} & \scalebox{1.00}{\colorbox[HTML]{d9f2d9}{13.53}}/\scalebox{1.00}{\colorbox[HTML]{d9f2d9}{4.68}}  \\   
		\hline
		\multirow{2}{*}{BI-BAU(U)}  & 39.56/0.34 & 1.60 & 38.78/1.30 & 40.28/\scalebox{1.00}{\colorbox[HTML]{d9f2d9}{0.00}} & 38.44/2.20 & 38.91/1.30 & 32.89/26.57 & 31.16/2.94 & 31.50/\scalebox{1.00}{\colorbox[HTML]{d9f2d9}{0.20}} & 30.34/8.00 & 36.29/\scalebox{1.00}{\colorbox[HTML]{d9f2d9}{4.44}} \\
		                            & \scalebox{1.00}{\colorbox[HTML]{d9f2d9}{0.98}}/0.52 & 36.20/12.20 & 25.40/0.50 & 25.70/0.10 & 11.70/2.50 & 0.60/\scalebox{1.00}{\colorbox[HTML]{d9f2d9}{0.00}} & \scalebox{1.00}{\colorbox[HTML]{FFCCCC}{96.72}}/\scalebox{1.00}{\colorbox[HTML]{d9f2d9}{13.95}} &  66.68/34.88 & 24.60/0.80 & 30.60/8.90 & 31.91/7.43 \\   
        \bottomrule
	\end{tabular}}
\end{table*}

From Tables~\ref{table:backdoor_unlearning_on_Tiny-ImageNet} and~\ref{table:backdoor_mitigation_on_GTSRB}, we observe that although BI-BAU does not achieve the 
highest clean accuracy (CA), it maintains competitive performance. Specifically, BI-BAU (U) attains CAs of 39.78\% and 93.87\% on Tiny-ImageNet-200 and GTSRB, 
respectively, corresponding to modest drops of approximately 7.42\% and 2.79\%. Notably, BI-BAU is the only method that consistently achieves a low attack success rate (ASR)
across all evaluated backdoor attacks. This includes the challenging SIG attack, against which all other methods either perform poorly or fail entirely. These findings 
are consistent with the results shown in Tables~\ref{table:backdoor_unlearning_on_CIFAR-10}, further demonstrating 
the robustness and generalizability of our BI-BAU method under varying model architectures and data distributions.

Furthermore, as shown in Tables~\ref{table:post_purification_robustness_on_Tiny-ImageNet} and~\ref{table:resistance_to_backdoor_re-activation_attack_on_GTSRB}, 
both PAM and BI-BAU achieve strong post-purification robustness performance. However, under attacks characterized by low orthogonality or low linearity, including
Refool, SIG, Adaptive-Blend, and Adaptive-Patch, PAM exhibits a high R-ASR, indicating that the backdoor effects can still be easily reactivated by the Retuning Attack (RA). 
In contrast, BI-BAU consistently maintains a low ASR, even when subjected to both RA and QRA. These observations suggest that BI-BAU is significantly more effective 
than PAM at thoroughly eliminating backdoor effects from backdoored models, ensuring more reliable and comprehensive defense performance.

\renewcommand{\arraystretch}{1.5}
\begin{table*}[!ht]
	\centering
	\caption{Performance (\%) of safety tuning strategies for eliminating backdoor effects on VGG19. 
	}
    \label{table:backdoor_mitigation_on_GTSRB}
	\resizebox{0.98\textwidth}{!}{\begin{tabular}{c|cc|cc|cc|cc|cc|cc|cc|cc}
		\toprule
		Attack $\to$ & \multicolumn{2}{c|}{BadNets} & \multicolumn{2}{c|}{Blended} & \multicolumn{2}{c|}{WaNet} & \multicolumn{2}{c|}{Refool} & \multicolumn{2}{c|}{IAD} & \multicolumn{2}{c|}{ISSBA} & \multicolumn{2}{c|}{SIG} & \multicolumn{2}{c}{Average}\\ 
		\midrule
		Defense $\downarrow$ & \multicolumn{1}{c}{CA}&\multicolumn{1}{c|}{ASR} & \multicolumn{1}{c}{CA}&\multicolumn{1}{c|}{ASR} & \multicolumn{1}{c}{CA}&\multicolumn{1}{c|}{ASR} & \multicolumn{1}{c}{CA}&\multicolumn{1}{c|}{ASR} & \multicolumn{1}{c}{CA}&\multicolumn{1}{c|}{ASR} & \multicolumn{1}{c}{CA}&\multicolumn{1}{c|}{ASR} & \multicolumn{1}{c}{CA}&\multicolumn{1}{c|}{ASR} & \multicolumn{1}{c}{CA}&\multicolumn{1}{c}{ASR} \\ 
		\hline
        No Defense  & 96.07&96.92 & 95.24&99.36 & 97.33&99.68 & 96.37&93.98 & 97.41&99.65 & 97.40&51.93 & 96.81&66.12 & 96.66&86.80\\
        \hline
		NC          &96.37&\scalebox{1.00}{\colorbox[HTML]{FFCCCC}{93.63}} & 95.90&\scalebox{1.00}{\colorbox[HTML]{FFCCCC}{23.75}} & 95.53&\scalebox{1.00}{\colorbox[HTML]{FFCCCC}{55.50}} & 96.24&\scalebox{1.00}{\colorbox[HTML]{FFCCCC}{64.30}} & 96.95&0.31 & 97.29&\scalebox{1.00}{\colorbox[HTML]{d9f2d9}{0.00}} & 95.99&\scalebox{1.00}{\colorbox[HTML]{FFCCCC}{52.10}} & 96.32&41.37   \\
		\hline
		BTI-DBF     &\scalebox{1.00}{\colorbox[HTML]{d9f2d9}{97.37}}&\scalebox{1.00}{\colorbox[HTML]{d9f2d9}{0.00}}  &\scalebox{1.00}{\colorbox[HTML]{d9f2d9}{96.01}}&0.15 &96.82&\scalebox{1.00}{\colorbox[HTML]{d9f2d9}{0.00}} &\scalebox{1.00}{\colorbox[HTML]{d9f2d9}{97.59}}&\scalebox{1.00}{\colorbox[HTML]{d9f2d9}{1.74}} &97.39&0.00 &\scalebox{1.00}{\colorbox[HTML]{d9f2d9}{97.58}}&0.95 &97.30&\scalebox{1.00}{\colorbox[HTML]{FFCCCC}{45.58}} &\scalebox{1.00}{\colorbox[HTML]{d9f2d9}{97.15}}&6.91  \\
		\hline
		ANP         & 93.34&0.01 & 94.74&18.21 &\scalebox{1.00}{\colorbox[HTML]{d9f2d9}{97.54}}&\scalebox{1.00}{\colorbox[HTML]{d9f2d9}{0.00}} & 95.99&6.25 & 96.92&\scalebox{1.00}{\colorbox[HTML]{d9f2d9}{0.00}} & 95.94&3.80 & \scalebox{1.00}{\colorbox[HTML]{d9f2d9}{97.52}}&\scalebox{1.00}{\colorbox[HTML]{FFCCCC}{56.17}} & 95.99&12.06   \\
        \hline
        EP          & 96.29&5.74 & 95.07&\scalebox{1.00}{\colorbox[HTML]{FFCCCC}{88.04}} & 96.78&\scalebox{1.00}{\colorbox[HTML]{FFCCCC}{81.07}} & 96.12&\scalebox{1.00}{\colorbox[HTML]{FFCCCC}{94.10}} & \scalebox{1.00}{\colorbox[HTML]{d9f2d9}{97.49}}&\scalebox{1.00}{\colorbox[HTML]{FFCCCC}{67.30}} & 97.57&0.23 & 96.77&\scalebox{1.00}{\colorbox[HTML]{FFCCCC}{65.79}} & 96.58&57.46\\
		\hline
		RNP    		& 87.04&8.44 & 89.31&0.00 & 96.34&0.00 & 94.33&9.24 & 95.12&\scalebox{1.00}{\colorbox[HTML]{d9f2d9}{0.00}} & 94.92&\scalebox{1.00}{\colorbox[HTML]{d9f2d9}{0.00}} & 89.80&54.72 & 92.40&10.34   \\
        \hline
        CLP		    & 85.36&96.26 & 73.92&79.57 & 94.34&4.35 & 18.11&99.82 & 92.88&2.70 & 85.88&0.23 & 71.13&45.43 & 74.51&46.90\\
		\hline
		I-BAU       & 94.10&8.06 & 95.73&\scalebox{1.00}{\colorbox[HTML]{FFCCCC}{68.48}} & 95.93&\scalebox{1.00}{\colorbox[HTML]{FFCCCC}{66.98}} & 96.34&\scalebox{1.00}{\colorbox[HTML]{FFCCCC}{50.75}} & 94.71&14.01 & 96.55&5.77 & 95.96&\scalebox{1.00}{\colorbox[HTML]{FFCCCC}{65.48}} & 95.61&39.93 \\
		\hline
		SAU         & 95.24&1.36 & 93.51&0.39 & 95.24&0.00 & 96.02&\scalebox{1.00}{\colorbox[HTML]{FFCCCC}{50.60}} & 94.67&0.00 & 95.74&0.00 & 93.74&\scalebox{1.00}{\colorbox[HTML]{FFCCCC}{32.38}} & 94.88&12.10\\
		\hline
        BI-BAU (T)  & 96.78&0.26 & 91.27&\scalebox{1.00}{\colorbox[HTML]{FFCCCC}{67.77}} & 93.70&0.00 & 94.72&3.67 & 95.02 &0.23 & 94.85&0.00 & 90.73&\scalebox{1.00}{\colorbox[HTML]{FFCCCC}{21.60}} & 93.86&13.36 \\
		\hline 
		BI-BAU (U) & 96.16&0.03 & 95.13& \scalebox{1.00}{\colorbox[HTML]{d9f2d9}{0.00}} & 91.54&\scalebox{1.00}{\colorbox[HTML]{d9f2d9}{0.00}} & 93.42&6.80 & 95.13&0.01 & 95.50&\scalebox{1.00}{\colorbox[HTML]{d9f2d9}{0.00}} & 90.27&\scalebox{1.00}{\colorbox[HTML]{d9f2d9}{17.88}} & 93.87&\scalebox{1.00}{\colorbox[HTML]{d9f2d9}{3.53}} \\
        \bottomrule
	\end{tabular}}
\end{table*} 
\renewcommand{\arraystretch}{1.5}
\begin{table*}[htbp]
	\centering
	\caption{Post-purification robustness (\%) of Safety Tuning methods under backdoor re-activation attacks (GTSRB,VGG19).
	}
    \label{table:resistance_to_backdoor_re-activation_attack_on_GTSRB} 
	\resizebox{0.98\textwidth}{!}{\begin{tabular}{c|c|c|c|c|c|c|c|c|c|c}
		\toprule
		Attack $\to$ & \multicolumn{1}{c|}{BadNets} & \multicolumn{1}{c|}{Blended} & \multicolumn{1}{c|}{WaNet} & \multicolumn{1}{c|}{Refool} & \multicolumn{1}{c|}{IAD} & \multicolumn{1}{c|}{ISSBA} & \multicolumn{1}{c|}{SIG} & \multicolumn{1}{c|}{Adaptive-Blend} & \multicolumn{1}{c|}{Adaptive-Patch} & \multicolumn{1}{c}{Average}\\ 
		\midrule
		\multirow{2}{*}{Defense $\downarrow$} & CA/ASR & CA/ASR & CA/ASR & CA/ASR & CA/ASR & CA/ASR & CA/ASR & CA/ASR & CA/ASR & CA/ASR \\
		
		                                     & R-ASR/Q-ASR & R-ASR/Q-ASR & R-ASR/Q-ASR & R-ASR/Q-ASR & R-ASR/Q-ASR & R-ASR/Q-ASR & R-ASR/Q-ASR & R-ASR/Q-ASR & R-ASR/Q-ASR & R-ASR/Q-ASR \\
		\midrule
		\multirow{2}{*}{No Defense}     & 96.00/95.94 & 95.24/99.36 & 97.33/99.68 & 96.37/93.98 & 97.41/99.65 & 97.40/51.93 & 96.81/66.12 & 96.21/51.54 & 97.44/60.64 & 96.76/71.88   \\
		
								 		& -/-  & -/- & -/-  & -/- & -/- & -/- & -/-  & -/- & -/- & -/-  \\   
		\hline
		\multirow{2}{*}{IBU}     & 99.49/0.00 & 97.55/0.00  & 97.39/0.00 & 98.66/0.00 & 99.09/0.00 & 98.64/0.00 & 98.85/0.00 & 98.46/0.00  & 98.75/0.00 & 98.54/0.00   \\
		
								 & 3.30/0.00  & 0.00/0.00 & 0.00/0.00  & 0.01/0.00 & 0.01/0.00  & 0.00/0.00 & 0.72/0.36  & 0.00/0.16  & 0.00/0.00  & 0.40/0.05  \\   
		\bottomrule
		\toprule
		\multirow{2}{*}{BTI-DBF} & \scalebox{1.00}{\colorbox[HTML]{d9f2d9}{97.41}}/0.03 & \scalebox{1.00}{\colorbox[HTML]{d9f2d9}{96.01}}/0.15 & 96.59/10.37 & \scalebox{1.00}{\colorbox[HTML]{d9f2d9}{97.02}}/57.41 & \scalebox{1.00}{\colorbox[HTML]{d9f2d9}{97.27}}/0.12 & 96.52/0.00 & \scalebox{1.00}{\colorbox[HTML]{d9f2d9}{96.68}}/45.23 & \scalebox{1.00}{\colorbox[HTML]{d9f2d9}{96.54}}/0.79 & \scalebox{1.00}{\colorbox[HTML]{d9f2d9}{97.65}}/10.60 & \scalebox{1.00}{\colorbox[HTML]{d9f2d9}{96.85}}/13.85 \\

								 & 0.23/0.17  & 0.00/2.30 & 99.84/56.37 & \scalebox{1.00}{\colorbox[HTML]{FFCCCC}{86.82}}/\scalebox{1.00}{\colorbox[HTML]{FFCCCC}{81.14}} & 16.54/0.46  & 1.18/0.00  &  \scalebox{1.00}{\colorbox[HTML]{FFCCCC}{50.88}}/\scalebox{1.00}{\colorbox[HTML]{FFCCCC}{40.56}} & \scalebox{1.00}{\colorbox[HTML]{FFCCCC}{30.95}}/\scalebox{1.00}{\colorbox[HTML]{FFCCCC}{11.08}} &  \scalebox{1.00}{\colorbox[HTML]{FFCCCC}{33.49}}/\scalebox{1.00}{\colorbox[HTML]{FFCCCC}{23.83}} & \scalebox{1.00}{\colorbox[HTML]{FFCCCC}{35.54}}/\scalebox{1.00}{\colorbox[HTML]{FFCCCC}{23.99}} \\   
		\hline
		\multirow{2}{*}{ANP}     & 95.37/0.01 & 89.99/6.49 & \scalebox{1.00}{\colorbox[HTML]{d9f2d9}{97.25}}/\scalebox{1.00}{\colorbox[HTML]{d9f2d9}{0.00}} & 95.56/6.20 & 96.76/\scalebox{1.00}{\colorbox[HTML]{d9f2d9}{0.00}} & \scalebox{1.00}{\colorbox[HTML]{d9f2d9}{97.20}}/\scalebox{1.00}{\colorbox[HTML]{d9f2d9}{0.00}} & 96.23/42.70 & 93.58/28.50  & 95.47/11.90 & 95.26/10.64   \\
		
								 & 60.68/25.53   & 84.71/47.19 & 99.76/75.30  & \scalebox{1.00}{\colorbox[HTML]{FFCCCC}{66.66}}/\scalebox{1.00}{\colorbox[HTML]{FFCCCC}{19.89}} & 96.31/24.80 & 28.10/28.03 & \scalebox{1.00}{\colorbox[HTML]{FFCCCC}{55.84}}/\scalebox{1.00}{\colorbox[HTML]{FFCCCC}{47.37}} & \scalebox{1.00}{\colorbox[HTML]{FFCCCC}{44.57}}/\scalebox{1.00}{\colorbox[HTML]{FFCCCC}{42.99}} & \scalebox{1.00}{\colorbox[HTML]{FFCCCC}{35.23}}/\scalebox{1.00}{\colorbox[HTML]{FFCCCC}{32.22}} & \scalebox{1.00}{\colorbox[HTML]{FFCCCC}{63.54}}/\scalebox{1.00}{\colorbox[HTML]{FFCCCC}{38.14}} \\   
		\hline
		\multirow{2}{*}{RNP}     & 87.04/8.44 & 89.31/\scalebox{1.00}{\colorbox[HTML]{d9f2d9}{0.00}} & 96.34/\scalebox{1.00}{\colorbox[HTML]{d9f2d9}{0.00}} & 94.33/9.24 & 95.12/\scalebox{1.00}{\colorbox[HTML]{d9f2d9}{0.00}} & 94.92/\scalebox{1.00}{\colorbox[HTML]{d9f2d9}{0.00}} & 89.80/54.72 & 86.87/62.23 & 87.04/8.44 & 91.19/15.89   \\
		
								 & \scalebox{1.00}{\colorbox[HTML]{FFCCCC}{94.96}}/\scalebox{1.00}{\colorbox[HTML]{FFCCCC}{27.36}} & 98.02/\scalebox{1.00}{\colorbox[HTML]{d9f2d9}{0.00}} & \scalebox{1.00}{\colorbox[HTML]{FFCCCC}{63.34}}/\scalebox{1.00}{\colorbox[HTML]{d9f2d9}{0.00}} & \scalebox{1.00}{\colorbox[HTML]{FFCCCC}{72.57}}/\scalebox{1.00}{\colorbox[HTML]{FFCCCC}{13.59}} & 89.20/2.69 & 30.48/6.10 & \scalebox{1.00}{\colorbox[HTML]{FFCCCC}{55.32}}/\scalebox{1.00}{\colorbox[HTML]{FFCCCC}{50.71}} & \scalebox{1.00}{\colorbox[HTML]{FFCCCC}{60.80}}/\scalebox{1.00}{\colorbox[HTML]{FFCCCC}{65.40}} & \scalebox{1.00}{\colorbox[HTML]{FFCCCC}{94.96}}/\scalebox{1.00}{\colorbox[HTML]{FFCCCC}{27.36}} & \scalebox{1.00}{\colorbox[HTML]{FFCCCC}{73.29}}/\scalebox{1.00}{\colorbox[HTML]{FFCCCC}{21.46}} \\   
		\hline
		\multirow{2}{*}{SAU}     & 95.50/\scalebox{1.00}{\colorbox[HTML]{d9f2d9}{0.00}} & 93.51/0.39 & 95.50/\scalebox{1.00}{\colorbox[HTML]{d9f2d9}{0.00}} & 96.02/50.60 & 94.67/\scalebox{1.00}{\colorbox[HTML]{d9f2d9}{0.00}} & 95.74/0.00 & 93.74/32.38 & 95.68/0.00  & 95.43/0.15 & 95.08/9.28 \\
		                         & \scalebox{1.00}{\colorbox[HTML]{d9f2d9}{0.01}}/\scalebox{1.00}{\colorbox[HTML]{d9f2d9}{0.00}} & 83.68/0.40 & 74.34/19.95 & \scalebox{1.00}{\colorbox[HTML]{FFCCCC}{92.35}}/\scalebox{1.00}{\colorbox[HTML]{FFCCCC}{76.36}} & 79.12/0.00  & 1.66/0.00 & \scalebox{1.00}{\colorbox[HTML]{FFCCCC}{69.01}}/\scalebox{1.00}{\colorbox[HTML]{FFCCCC}{47.47}} & \scalebox{1.00}{\colorbox[HTML]{d9f2d9}{0.00}}/\scalebox{1.00}{\colorbox[HTML]{d9f2d9}{0.00}} & 10.60/0.24 & \scalebox{1.00}{\colorbox[HTML]{FFCCCC}{45.64}}/\scalebox{1.00}{\colorbox[HTML]{FFCCCC}{16.04}} \\   
		\hline
		\multirow{2}{*}{PAM}     & 95.75/\scalebox{1.00}{\colorbox[HTML]{d9f2d9}{0.00}} 
		& 95.83/\scalebox{1.00}{\colorbox[HTML]{d9f2d9}{0.00}}
		& 96.15/\scalebox{1.00}{\colorbox[HTML]{d9f2d9}{0.00}} 
		& 90.65/\scalebox{1.00}{\colorbox[HTML]{d9f2d9}{0.01}}
		& 93.49/\scalebox{1.00}{\colorbox[HTML]{d9f2d9}{0.00}}
		& 96.69/\scalebox{1.00}{\colorbox[HTML]{d9f2d9}{0.00}} 
		& 89.11/\scalebox{1.00}{\colorbox[HTML]{d9f2d9}{0.00}}
		& 94.70/\scalebox{1.00}{\colorbox[HTML]{d9f2d9}{0.00}} 
		& 95.25/\scalebox{1.00}{\colorbox[HTML]{d9f2d9}{0.00}} 
		& 94.18/\scalebox{1.00}{\colorbox[HTML]{d9f2d9}{0.00}} \\
		
		& \scalebox{1.00}{\colorbox[HTML]{d9f2d9}{0.01}}/\scalebox{1.00}{\colorbox[HTML]{d9f2d9}{0.00}}
		& \scalebox{1.00}{\colorbox[HTML]{d9f2d9}{0.00}}/\scalebox{1.00}{\colorbox[HTML]{d9f2d9}{0.00}}
		& \scalebox{1.00}{\colorbox[HTML]{d9f2d9}{0.00}}/\scalebox{1.00}{\colorbox[HTML]{d9f2d9}{0.00}}
		& \scalebox{1.00}{\colorbox[HTML]{FFCCCC}{91.29}}/\scalebox{1.00}{\colorbox[HTML]{d9f2d9}{0.05}}
		& \scalebox{1.00}{\colorbox[HTML]{FFCCCC}{78.81}}/\scalebox{1.00}{\colorbox[HTML]{d9f2d9}{0.00}}  
		& \scalebox{1.00}{\colorbox[HTML]{d9f2d9}{0.00}}/\scalebox{1.00}{\colorbox[HTML]{d9f2d9}{0.00}}
		& \scalebox{1.00}{\colorbox[HTML]{FFCCCC}{51.60}}/\scalebox{1.00}{\colorbox[HTML]{d9f2d9}{0.00}}	
		& \scalebox{1.00}{\colorbox[HTML]{FFCCCC}{30.95}}/\scalebox{1.00}{\colorbox[HTML]{d9f2d9}{0.00}}
		& \scalebox{1.00}{\colorbox[HTML]{FFCCCC}{34.36}}/\scalebox{1.00}{\colorbox[HTML]{d9f2d9}{0.00}}
		& \scalebox{1.00}{\colorbox[HTML]{FFCCCC}{31.89}}/\scalebox{1.00}{\colorbox[HTML]{d9f2d9}{0.00}} \\  
		
		\hline
		\multirow{2}{*}{BI-BAU(T)}  & 90.94/\scalebox{1.00}{\colorbox[HTML]{d9f2d9}{0.00}} & 87.21/44.02 & 93.96/\scalebox{1.00}{\colorbox[HTML]{d9f2d9}{0.00}} & 94.34/4.32 & 95.12/0.03 & 94.90/\scalebox{1.00}{\colorbox[HTML]{d9f2d9}{0.00}} & 90.73/21.60 &  92.17/0.07 & 91.01/\scalebox{1.00}{\colorbox[HTML]{d9f2d9}{0.00}} & 92.26/7.78   \\
								    & 0.15/0.03 & \scalebox{1.00}{\colorbox[HTML]{FFCCCC}{57.48}}/\scalebox{1.00}{\colorbox[HTML]{FFCCCC}{56.53}} & 0.07/\scalebox{1.00}{\colorbox[HTML]{d9f2d9}{0.00}} & 25.57/8.79 & 1.64/0.05 & \scalebox{1.00}{\colorbox[HTML]{d9f2d9}{0.00}}/\scalebox{1.00}{\colorbox[HTML]{d9f2d9}{0.00}} & 27.38/24.84 & 0.15/0.08& 0.23/\scalebox{1.00}{\colorbox[HTML]{d9f2d9}{0.00}} & 11.26/9.03  \\   
		\hline
		\multirow{2}{*}{BI-BAU(U)}  
		& 96.32/0.03 
		& 87.23/0.15 
		& 94.68/\scalebox{1.00}{\colorbox[HTML]{d9f2d9}{0.00}}
		& 91.24/0.55 
		& 95.10/0.03 
		& 94.15/\scalebox{1.00}{\colorbox[HTML]{d9f2d9}{0.00}}
		& 86.07/10.78 & 91.34/0.39 & 93.79/1.10 & 92.21/1.44   \\
		
		& 0.39/0.03 
		& 0.71/0.24  
		& \scalebox{1.00}{\colorbox[HTML]{d9f2d9}{0.00}}/\scalebox{1.00}{\colorbox[HTML]{d9f2d9}{0.00}}
		& \scalebox{1.00}{\colorbox[HTML]{d9f2d9}{4.76}}/\scalebox{1.00}{\colorbox[HTML]{d9f2d9}{1.17}}
		& \scalebox{1.00}{\colorbox[HTML]{d9f2d9}{4.95}}/\scalebox{1.00}{\colorbox[HTML]{d9f2d9}{0.00}}
		& 0.23/1.11 
		& \scalebox{1.00}{\colorbox[HTML]{d9f2d9}{27.83}}/24.43 
		& \scalebox{1.00}{\colorbox[HTML]{d9f2d9}{1.42}}/0.79 
		& \scalebox{1.00}{\colorbox[HTML]{d9f2d9}{1.66}}/0.55 
		& \scalebox{1.00}{\colorbox[HTML]{d9f2d9}{4.66}}/3.14 \\   
        \bottomrule
	\end{tabular}}
\end{table*} 

%% file: main.bbl

\begin{thebibliography}{67}


\ifx \showCODEN    \undefined \def \showCODEN     #1{\unskip}     \fi
\ifx \showISBNx    \undefined \def \showISBNx     #1{\unskip}     \fi
\ifx \showISBNxiii \undefined \def \showISBNxiii  #1{\unskip}     \fi
\ifx \showISSN     \undefined \def \showISSN      #1{\unskip}     \fi
\ifx \showLCCN     \undefined \def \showLCCN      #1{\unskip}     \fi
\ifx \shownote     \undefined \def \shownote      #1{#1}          \fi
\ifx \showarticletitle \undefined \def \showarticletitle #1{#1}   \fi
\ifx \showURL      \undefined \def \showURL       {\relax}        \fi
\providecommand\bibfield[2]{#2}
\providecommand\bibinfo[2]{#2}
\providecommand\natexlab[1]{#1}
\providecommand\showeprint[2][]{arXiv:#2}

\bibitem[Bansal et~al\mbox{.}(2023)]%
        {bansal2023cleanclip}
\bibfield{author}{\bibinfo{person}{Hritik Bansal}, \bibinfo{person}{Nishad Singhi}, \bibinfo{person}{Yu Yang}, \bibinfo{person}{Fan Yin}, \bibinfo{person}{Aditya Grover}, {and} \bibinfo{person}{Kai-Wei Chang}.} \bibinfo{year}{2023}\natexlab{}.
\newblock \showarticletitle{Cleanclip: Mitigating data poisoning attacks in multimodal contrastive learning}. In \bibinfo{booktitle}{\emph{ICCV}}.
\newblock


\bibitem[Barni et~al\mbox{.}(2019)]%
        {barni2019new}
\bibfield{author}{\bibinfo{person}{Mauro Barni}, \bibinfo{person}{Kassem Kallas}, {and} \bibinfo{person}{Benedetta Tondi}.} \bibinfo{year}{2019}\natexlab{}.
\newblock \showarticletitle{A new backdoor attack in cnns by training set corruption without label poisoning}. In \bibinfo{booktitle}{\emph{ICIP}}.
\newblock


\bibitem[Bennani et~al\mbox{.}(2020)]%
        {bennani2020generalisation}
\bibfield{author}{\bibinfo{person}{Mehdi~Abbana Bennani}, \bibinfo{person}{Thang Doan}, {and} \bibinfo{person}{Masashi Sugiyama}.} \bibinfo{year}{2020}\natexlab{}.
\newblock \showarticletitle{Generalisation guarantees for continual learning with orthogonal gradient descent}.
\newblock \bibinfo{journal}{\emph{arXiv preprint arXiv:2006.11942}} (\bibinfo{year}{2020}).
\newblock


\bibitem[Chen et~al\mbox{.}(2017)]%
        {chen2017targeted}
\bibfield{author}{\bibinfo{person}{Xinyun Chen}, \bibinfo{person}{Chang Liu}, \bibinfo{person}{Bo Li}, \bibinfo{person}{Kimberly Lu}, {and} \bibinfo{person}{Dawn Song}.} \bibinfo{year}{2017}\natexlab{}.
\newblock \showarticletitle{Targeted backdoor attacks on deep learning systems using data poisoning}.
\newblock \bibinfo{journal}{\emph{arXiv preprint arXiv:1712.05526}} (\bibinfo{year}{2017}).
\newblock


\bibitem[Chen and Liu(2022)]%
        {chen2022lifelong}
\bibfield{author}{\bibinfo{person}{Zhiyuan Chen} {and} \bibinfo{person}{Bing Liu}.} \bibinfo{year}{2022}\natexlab{}.
\newblock \bibinfo{booktitle}{\emph{Lifelong machine learning}}.
\newblock \bibinfo{publisher}{Springer Nature}.
\newblock


\bibitem[Cubuk et~al\mbox{.}(2017)]%
        {cubuk2017intriguing}
\bibfield{author}{\bibinfo{person}{Ekin~D Cubuk}, \bibinfo{person}{Barret Zoph}, \bibinfo{person}{Samuel~S Schoenholz}, {and} \bibinfo{person}{Quoc~V Le}.} \bibinfo{year}{2017}\natexlab{}.
\newblock \showarticletitle{Intriguing properties of adversarial examples}.
\newblock \bibinfo{journal}{\emph{arXiv preprint arXiv:1711.02846}} (\bibinfo{year}{2017}).
\newblock


\bibitem[Dempster et~al\mbox{.}(1977)]%
        {dempster1977maximum}
\bibfield{author}{\bibinfo{person}{Arthur~P Dempster}, \bibinfo{person}{Nan~M Laird}, {and} \bibinfo{person}{Donald~B Rubin}.} \bibinfo{year}{1977}\natexlab{}.
\newblock \showarticletitle{Maximum likelihood from incomplete data via the EM algorithm}.
\newblock \bibinfo{journal}{\emph{Journal of the royal statistical society: series B (methodological)}} \bibinfo{volume}{39}, \bibinfo{number}{1} (\bibinfo{year}{1977}), \bibinfo{pages}{1--22}.
\newblock


\bibitem[Doan et~al\mbox{.}(2021)]%
        {doan2021theoretical}
\bibfield{author}{\bibinfo{person}{Thang Doan}, \bibinfo{person}{Mehdi~Abbana Bennani}, \bibinfo{person}{Bogdan Mazoure}, \bibinfo{person}{Guillaume Rabusseau}, {and} \bibinfo{person}{Pierre Alquier}.} \bibinfo{year}{2021}\natexlab{}.
\newblock \showarticletitle{A theoretical analysis of catastrophic forgetting through the ntk overlap matrix}. In \bibinfo{booktitle}{\emph{AISTATS}}.
\newblock


\bibitem[Dosovitskiy et~al\mbox{.}(2020)]%
        {dosovitskiy2020image}
\bibfield{author}{\bibinfo{person}{Alexey Dosovitskiy}, \bibinfo{person}{Lucas Beyer}, \bibinfo{person}{Alexander Kolesnikov}, \bibinfo{person}{Dirk Weissenborn}, \bibinfo{person}{Xiaohua Zhai}, \bibinfo{person}{Thomas Unterthiner}, \bibinfo{person}{Mostafa Dehghani}, \bibinfo{person}{Matthias Minderer}, \bibinfo{person}{Georg Heigold}, \bibinfo{person}{Sylvain Gelly}, {et~al\mbox{.}}} \bibinfo{year}{2020}\natexlab{}.
\newblock \showarticletitle{An image is worth 16x16 words: Transformers for image recognition at scale}.
\newblock \bibinfo{journal}{\emph{arXiv preprint arXiv:2010.11929}} (\bibinfo{year}{2020}).
\newblock


\bibitem[Gao et~al\mbox{.}(2021)]%
        {gao2021deepgem}
\bibfield{author}{\bibinfo{person}{Angela Gao}, \bibinfo{person}{Jorge Castellanos}, \bibinfo{person}{Yisong Yue}, \bibinfo{person}{Zachary Ross}, {and} \bibinfo{person}{Katherine Bouman}.} \bibinfo{year}{2021}\natexlab{}.
\newblock \showarticletitle{DeepGEM: Generalized expectation-maximization for blind inversion}. In \bibinfo{booktitle}{\emph{NeurIPS}}.
\newblock


\bibitem[Gao et~al\mbox{.}(2023)]%
        {gao2023effectiveness}
\bibfield{author}{\bibinfo{person}{Yinghua Gao}, \bibinfo{person}{Dongxian Wu}, \bibinfo{person}{Jingfeng Zhang}, \bibinfo{person}{Guanhao Gan}, \bibinfo{person}{Shu-Tao Xia}, \bibinfo{person}{Gang Niu}, {and} \bibinfo{person}{Masashi Sugiyama}.} \bibinfo{year}{2023}\natexlab{}.
\newblock \showarticletitle{On the effectiveness of adversarial training against backdoor attacks}.
\newblock \bibinfo{journal}{\emph{IEEE Transactions on Neural Networks and Learning Systems}} (\bibinfo{year}{2023}).
\newblock


\bibitem[Geirhos et~al\mbox{.}(2020)]%
        {geirhos2020shortcut}
\bibfield{author}{\bibinfo{person}{Robert Geirhos}, \bibinfo{person}{J{\"o}rn-Henrik Jacobsen}, \bibinfo{person}{Claudio Michaelis}, \bibinfo{person}{Richard Zemel}, \bibinfo{person}{Wieland Brendel}, \bibinfo{person}{Matthias Bethge}, {and} \bibinfo{person}{Felix~A Wichmann}.} \bibinfo{year}{2020}\natexlab{}.
\newblock \showarticletitle{Shortcut learning in deep neural networks}.
\newblock \bibinfo{journal}{\emph{Nature Machine Intelligence}} \bibinfo{volume}{2}, \bibinfo{number}{11} (\bibinfo{year}{2020}), \bibinfo{pages}{665--673}.
\newblock


\bibitem[Gu et~al\mbox{.}(2019)]%
        {gu2019badnets}
\bibfield{author}{\bibinfo{person}{Tianyu Gu}, \bibinfo{person}{Kang Liu}, \bibinfo{person}{Brendan Dolan-Gavitt}, {and} \bibinfo{person}{Siddharth Garg}.} \bibinfo{year}{2019}\natexlab{}.
\newblock \showarticletitle{Badnets: Evaluating backdooring attacks on deep neural networks}.
\newblock \bibinfo{journal}{\emph{IEEE Access}}  \bibinfo{volume}{7} (\bibinfo{year}{2019}), \bibinfo{pages}{47230--47244}.
\newblock


\bibitem[He et~al\mbox{.}(2016)]%
        {he2016deep}
\bibfield{author}{\bibinfo{person}{Kaiming He}, \bibinfo{person}{Xiangyu Zhang}, \bibinfo{person}{Shaoqing Ren}, {and} \bibinfo{person}{Jian Sun}.} \bibinfo{year}{2016}\natexlab{}.
\newblock \showarticletitle{Deep residual learning for image recognition}. In \bibinfo{booktitle}{\emph{CVPR}}.
\newblock


\bibitem[Hinton et~al\mbox{.}(2015)]%
        {hinton2015distilling}
\bibfield{author}{\bibinfo{person}{Geoffrey Hinton}, \bibinfo{person}{Oriol Vinyals}, {and} \bibinfo{person}{Jeff Dean}.} \bibinfo{year}{2015}\natexlab{}.
\newblock \showarticletitle{Distilling the knowledge in a neural network}.
\newblock \bibinfo{journal}{\emph{arXiv preprint arXiv:1503.02531}} (\bibinfo{year}{2015}).
\newblock


\bibitem[Huang et~al\mbox{.}(2022)]%
        {huang2022backdoor}
\bibfield{author}{\bibinfo{person}{Kunzhe Huang}, \bibinfo{person}{Yiming Li}, \bibinfo{person}{Baoyuan Wu}, \bibinfo{person}{Zhan Qin}, {and} \bibinfo{person}{Kui Ren}.} \bibinfo{year}{2022}\natexlab{}.
\newblock \showarticletitle{Backdoor defense via decoupling the training process}. In \bibinfo{booktitle}{\emph{ICLR}}.
\newblock


\bibitem[Jain(2022)]%
        {jain2022hugging}
\bibfield{author}{\bibinfo{person}{Shashank~Mohan Jain}.} \bibinfo{year}{2022}\natexlab{}.
\newblock \showarticletitle{Hugging face}.
\newblock In \bibinfo{booktitle}{\emph{Introduction to transformers for NLP: With the hugging face library and models to solve problems}}. \bibinfo{publisher}{Springer}, \bibinfo{pages}{51--67}.
\newblock


\bibitem[Kirkpatrick et~al\mbox{.}(2017)]%
        {kirkpatrick2017overcoming}
\bibfield{author}{\bibinfo{person}{James Kirkpatrick}, \bibinfo{person}{Razvan Pascanu}, \bibinfo{person}{Neil Rabinowitz}, \bibinfo{person}{Joel Veness}, \bibinfo{person}{Guillaume Desjardins}, \bibinfo{person}{Andrei~A Rusu}, \bibinfo{person}{Kieran Milan}, \bibinfo{person}{John Quan}, \bibinfo{person}{Tiago Ramalho}, \bibinfo{person}{Agnieszka Grabska-Barwinska}, {et~al\mbox{.}}} \bibinfo{year}{2017}\natexlab{}.
\newblock \showarticletitle{Overcoming catastrophic forgetting in neural networks}.
\newblock \bibinfo{journal}{\emph{Proceedings of the national academy of sciences}} \bibinfo{volume}{114}, \bibinfo{number}{13} (\bibinfo{year}{2017}), \bibinfo{pages}{3521--3526}.
\newblock


\bibitem[Kornblith et~al\mbox{.}(2019)]%
        {kornblith2019similarity}
\bibfield{author}{\bibinfo{person}{Simon Kornblith}, \bibinfo{person}{Mohammad Norouzi}, \bibinfo{person}{Honglak Lee}, {and} \bibinfo{person}{Geoffrey Hinton}.} \bibinfo{year}{2019}\natexlab{}.
\newblock \showarticletitle{Similarity of neural network representations revisited}. In \bibinfo{booktitle}{\emph{ICML}}. PMLR.
\newblock


\bibitem[Krizhevsky et~al\mbox{.}(2009)]%
        {krizhevsky2009learning}
\bibfield{author}{\bibinfo{person}{Alex Krizhevsky}, \bibinfo{person}{Geoffrey Hinton}, {et~al\mbox{.}}} \bibinfo{year}{2009}\natexlab{}.
\newblock \showarticletitle{Learning multiple layers of features from tiny images}.
\newblock \bibinfo{journal}{\emph{Technical Report}} (\bibinfo{year}{2009}).
\newblock


\bibitem[Lee et~al\mbox{.}(2019)]%
        {lee2019wide}
\bibfield{author}{\bibinfo{person}{Jaehoon Lee}, \bibinfo{person}{Lechao Xiao}, \bibinfo{person}{Samuel Schoenholz}, \bibinfo{person}{Yasaman Bahri}, \bibinfo{person}{Roman Novak}, \bibinfo{person}{Jascha Sohl-Dickstein}, {and} \bibinfo{person}{Jeffrey Pennington}.} \bibinfo{year}{2019}\natexlab{}.
\newblock \showarticletitle{Wide neural networks of any depth evolve as linear models under gradient descent}. In \bibinfo{booktitle}{\emph{NeurIPS}}.
\newblock


\bibitem[Li et~al\mbox{.}(2022)]%
        {li2022backdoor}
\bibfield{author}{\bibinfo{person}{Yiming Li}, \bibinfo{person}{Yong Jiang}, \bibinfo{person}{Zhifeng Li}, {and} \bibinfo{person}{Shu-Tao Xia}.} \bibinfo{year}{2022}\natexlab{}.
\newblock \showarticletitle{Backdoor learning: A survey}.
\newblock \bibinfo{journal}{\emph{IEEE Transactions on Neural Networks and Learning Systems}} \bibinfo{volume}{35}, \bibinfo{number}{1} (\bibinfo{year}{2022}), \bibinfo{pages}{5--22}.
\newblock


\bibitem[Li et~al\mbox{.}(2021a)]%
        {li2021invisible}
\bibfield{author}{\bibinfo{person}{Yuezun Li}, \bibinfo{person}{Yiming Li}, \bibinfo{person}{Baoyuan Wu}, \bibinfo{person}{Longkang Li}, \bibinfo{person}{Ran He}, {and} \bibinfo{person}{Siwei Lyu}.} \bibinfo{year}{2021}\natexlab{a}.
\newblock \showarticletitle{Invisible backdoor attack with sample-specific triggers}. In \bibinfo{booktitle}{\emph{ICCV}}.
\newblock


\bibitem[Li et~al\mbox{.}(2021b)]%
        {li2021anti}
\bibfield{author}{\bibinfo{person}{Yige Li}, \bibinfo{person}{Xixiang Lyu}, \bibinfo{person}{Nodens Koren}, \bibinfo{person}{Lingjuan Lyu}, \bibinfo{person}{Bo Li}, {and} \bibinfo{person}{Xingjun Ma}.} \bibinfo{year}{2021}\natexlab{b}.
\newblock \showarticletitle{Anti-backdoor learning: Training clean models on poisoned data}. In \bibinfo{booktitle}{\emph{NeurIPS}}.
\newblock


\bibitem[Li et~al\mbox{.}(2021c)]%
        {li2021neural}
\bibfield{author}{\bibinfo{person}{Yige Li}, \bibinfo{person}{Xixiang Lyu}, \bibinfo{person}{Nodens Koren}, \bibinfo{person}{Lingjuan Lyu}, \bibinfo{person}{Bo Li}, {and} \bibinfo{person}{Xingjun Ma}.} \bibinfo{year}{2021}\natexlab{c}.
\newblock \showarticletitle{Neural attention distillation: Erasing backdoor triggers from deep neural networks}.
\newblock \bibinfo{journal}{\emph{arXiv preprint arXiv:2101.05930}} (\bibinfo{year}{2021}).
\newblock


\bibitem[Li et~al\mbox{.}(2023)]%
        {li2023reconstructive}
\bibfield{author}{\bibinfo{person}{Yige Li}, \bibinfo{person}{Xixiang Lyu}, \bibinfo{person}{Xingjun Ma}, \bibinfo{person}{Nodens Koren}, \bibinfo{person}{Lingjuan Lyu}, \bibinfo{person}{Bo Li}, {and} \bibinfo{person}{Yu-Gang Jiang}.} \bibinfo{year}{2023}\natexlab{}.
\newblock \showarticletitle{Reconstructive neuron pruning for backdoor defense}. In \bibinfo{booktitle}{\emph{ICML}}.
\newblock


\bibitem[Liang et~al\mbox{.}(2024)]%
        {liang2024badclip}
\bibfield{author}{\bibinfo{person}{Siyuan Liang}, \bibinfo{person}{Mingli Zhu}, \bibinfo{person}{Aishan Liu}, \bibinfo{person}{Baoyuan Wu}, \bibinfo{person}{Xiaochun Cao}, {and} \bibinfo{person}{Ee-Chien Chang}.} \bibinfo{year}{2024}\natexlab{}.
\newblock \showarticletitle{Badclip: Dual-embedding guided backdoor attack on multimodal contrastive learning}. In \bibinfo{booktitle}{\emph{CVPR}}.
\newblock


\bibitem[Lin et~al\mbox{.}(2024)]%
        {lin2024unveiling}
\bibfield{author}{\bibinfo{person}{Weilin Lin}, \bibinfo{person}{Li Liu}, \bibinfo{person}{Shaokui Wei}, \bibinfo{person}{Jianze Li}, {and} \bibinfo{person}{Hui Xiong}.} \bibinfo{year}{2024}\natexlab{}.
\newblock \showarticletitle{Unveiling and mitigating backdoor vulnerabilities based on unlearning weight changes and backdoor activeness}.
\newblock \bibinfo{journal}{\emph{Advances in Neural Information Processing Systems}}  \bibinfo{volume}{37} (\bibinfo{year}{2024}), \bibinfo{pages}{42097--42122}.
\newblock


\bibitem[Liu et~al\mbox{.}(2018)]%
        {liu2018fine}
\bibfield{author}{\bibinfo{person}{Kang Liu}, \bibinfo{person}{Brendan Dolan-Gavitt}, {and} \bibinfo{person}{Siddharth Garg}.} \bibinfo{year}{2018}\natexlab{}.
\newblock \showarticletitle{Fine-pruning: Defending against backdooring attacks on deep neural networks}. In \bibinfo{booktitle}{\emph{RAID}}.
\newblock


\bibitem[Liu et~al\mbox{.}(2020)]%
        {liu2020reflection}
\bibfield{author}{\bibinfo{person}{Yunfei Liu}, \bibinfo{person}{Xingjun Ma}, \bibinfo{person}{James Bailey}, {and} \bibinfo{person}{Feng Lu}.} \bibinfo{year}{2020}\natexlab{}.
\newblock \showarticletitle{Reflection backdoor: A natural backdoor attack on deep neural networks}. In \bibinfo{booktitle}{\emph{ECCV}}.
\newblock


\bibitem[Lopez-Paz and Ranzato(2017)]%
        {lopez2017gradient}
\bibfield{author}{\bibinfo{person}{David Lopez-Paz} {and} \bibinfo{person}{Marc'Aurelio Ranzato}.} \bibinfo{year}{2017}\natexlab{}.
\newblock \showarticletitle{Gradient episodic memory for continual learning}. In \bibinfo{booktitle}{\emph{NeurIPS}}.
\newblock


\bibitem[McCloskey and Cohen(1989)]%
        {mccloskey1989catastrophic}
\bibfield{author}{\bibinfo{person}{Michael McCloskey} {and} \bibinfo{person}{Neal~J Cohen}.} \bibinfo{year}{1989}\natexlab{}.
\newblock \showarticletitle{Catastrophic interference in connectionist networks: The sequential learning problem}.
\newblock In \bibinfo{booktitle}{\emph{Psychology of learning and motivation}}. Vol.~\bibinfo{volume}{24}. \bibinfo{publisher}{Elsevier}, \bibinfo{pages}{109--165}.
\newblock


\bibitem[Min et~al\mbox{.}(2024)]%
        {min2024uncovering}
\bibfield{author}{\bibinfo{person}{Rui Min}, \bibinfo{person}{Zeyu Qin}, \bibinfo{person}{Nevin~L Zhang}, \bibinfo{person}{Li Shen}, {and} \bibinfo{person}{Minhao Cheng}.} \bibinfo{year}{2024}\natexlab{}.
\newblock \showarticletitle{Uncovering, Explaining, and Mitigating the Superficial Safety of Backdoor Defense}. In \bibinfo{booktitle}{\emph{NeurIPS}}.
\newblock


\bibitem[Nguyen and Tran(2021)]%
        {nguyen2021wanet}
\bibfield{author}{\bibinfo{person}{Anh Nguyen} {and} \bibinfo{person}{Anh Tran}.} \bibinfo{year}{2021}\natexlab{}.
\newblock \showarticletitle{Wanet--imperceptible warping-based backdoor attack}.
\newblock \bibinfo{journal}{\emph{arXiv preprint arXiv:2102.10369}} (\bibinfo{year}{2021}).
\newblock


\bibitem[Nguyen et~al\mbox{.}(2019)]%
        {nguyen2019toward}
\bibfield{author}{\bibinfo{person}{Cuong~V Nguyen}, \bibinfo{person}{Alessandro Achille}, \bibinfo{person}{Michael Lam}, \bibinfo{person}{Tal Hassner}, \bibinfo{person}{Vijay Mahadevan}, {and} \bibinfo{person}{Stefano Soatto}.} \bibinfo{year}{2019}\natexlab{}.
\newblock \showarticletitle{Toward understanding catastrophic forgetting in continual learning}.
\newblock \bibinfo{journal}{\emph{arXiv preprint arXiv:1908.01091}} (\bibinfo{year}{2019}).
\newblock


\bibitem[Nguyen and Tran(2020)]%
        {nguyen2020input}
\bibfield{author}{\bibinfo{person}{Tuan~Anh Nguyen} {and} \bibinfo{person}{Anh Tran}.} \bibinfo{year}{2020}\natexlab{}.
\newblock \showarticletitle{Input-aware dynamic backdoor attack}. In \bibinfo{booktitle}{\emph{NeurIPS}}.
\newblock


\bibitem[Parisi et~al\mbox{.}(2019)]%
        {parisi2019continual}
\bibfield{author}{\bibinfo{person}{German~I Parisi}, \bibinfo{person}{Ronald Kemker}, \bibinfo{person}{Jose~L Part}, \bibinfo{person}{Christopher Kanan}, {and} \bibinfo{person}{Stefan Wermter}.} \bibinfo{year}{2019}\natexlab{}.
\newblock \showarticletitle{Continual lifelong learning with neural networks: A review}.
\newblock \bibinfo{journal}{\emph{Neural networks}}  \bibinfo{volume}{113} (\bibinfo{year}{2019}), \bibinfo{pages}{54--71}.
\newblock


\bibitem[Park et~al\mbox{.}(2019)]%
        {park2019relational}
\bibfield{author}{\bibinfo{person}{Wonpyo Park}, \bibinfo{person}{Dongju Kim}, \bibinfo{person}{Yan Lu}, {and} \bibinfo{person}{Minsu Cho}.} \bibinfo{year}{2019}\natexlab{}.
\newblock \showarticletitle{Relational knowledge distillation}. In \bibinfo{booktitle}{\emph{CVPR}}.
\newblock


\bibitem[Qi et~al\mbox{.}(2023)]%
        {qi2023revisiting}
\bibfield{author}{\bibinfo{person}{Xiangyu Qi}, \bibinfo{person}{Tinghao Xie}, \bibinfo{person}{Yiming Li}, \bibinfo{person}{Saeed Mahloujifar}, {and} \bibinfo{person}{Prateek Mittal}.} \bibinfo{year}{2023}\natexlab{}.
\newblock \showarticletitle{Revisiting the assumption of latent separability for backdoor defenses}. In \bibinfo{booktitle}{\emph{ICLR}}.
\newblock


\bibitem[Qiao et~al\mbox{.}(2019)]%
        {qiao2019defending}
\bibfield{author}{\bibinfo{person}{Ximing Qiao}, \bibinfo{person}{Yukun Yang}, {and} \bibinfo{person}{Hai Li}.} \bibinfo{year}{2019}\natexlab{}.
\newblock \showarticletitle{Defending neural backdoors via generative distribution modeling}. In \bibinfo{booktitle}{\emph{NeurIPS}}.
\newblock


\bibitem[Radford et~al\mbox{.}(2021)]%
        {radford2021learning}
\bibfield{author}{\bibinfo{person}{Alec Radford}, \bibinfo{person}{Jong~Wook Kim}, \bibinfo{person}{Chris Hallacy}, \bibinfo{person}{Aditya Ramesh}, \bibinfo{person}{Gabriel Goh}, \bibinfo{person}{Sandhini Agarwal}, \bibinfo{person}{Girish Sastry}, \bibinfo{person}{Amanda Askell}, \bibinfo{person}{Pamela Mishkin}, \bibinfo{person}{Jack Clark}, {et~al\mbox{.}}} \bibinfo{year}{2021}\natexlab{}.
\newblock \showarticletitle{Learning transferable visual models from natural language supervision}. In \bibinfo{booktitle}{\emph{ICML}}.
\newblock


\bibitem[Rolnick et~al\mbox{.}(2019)]%
        {rolnick2019experience}
\bibfield{author}{\bibinfo{person}{David Rolnick}, \bibinfo{person}{Arun Ahuja}, \bibinfo{person}{Jonathan Schwarz}, \bibinfo{person}{Timothy Lillicrap}, {and} \bibinfo{person}{Gregory Wayne}.} \bibinfo{year}{2019}\natexlab{}.
\newblock \showarticletitle{Experience replay for continual learning}. In \bibinfo{booktitle}{\emph{NeurIPS}}.
\newblock


\bibitem[Salem et~al\mbox{.}(2022)]%
        {salem2022dynamic}
\bibfield{author}{\bibinfo{person}{Ahmed Salem}, \bibinfo{person}{Rui Wen}, \bibinfo{person}{Michael Backes}, \bibinfo{person}{Shiqing Ma}, {and} \bibinfo{person}{Yang Zhang}.} \bibinfo{year}{2022}\natexlab{}.
\newblock \showarticletitle{Dynamic backdoor attacks against machine learning models}. In \bibinfo{booktitle}{\emph{S\&P}}.
\newblock


\bibitem[Semenova et~al\mbox{.}(2022)]%
        {semenova2022understanding}
\bibfield{author}{\bibinfo{person}{Nadezhda Semenova}, \bibinfo{person}{Laurent Larger}, {and} \bibinfo{person}{Daniel Brunner}.} \bibinfo{year}{2022}\natexlab{}.
\newblock \showarticletitle{Understanding and mitigating noise in trained deep neural networks}.
\newblock \bibinfo{journal}{\emph{Neural Networks}}  \bibinfo{volume}{146} (\bibinfo{year}{2022}), \bibinfo{pages}{151--160}.
\newblock


\bibitem[Sharma et~al\mbox{.}(2018)]%
        {sharma2018conceptual}
\bibfield{author}{\bibinfo{person}{Piyush Sharma}, \bibinfo{person}{Nan Ding}, \bibinfo{person}{Sebastian Goodman}, {and} \bibinfo{person}{Radu Soricut}.} \bibinfo{year}{2018}\natexlab{}.
\newblock \showarticletitle{Conceptual captions: A cleaned, hypernymed, image alt-text dataset for automatic image captioning}. In \bibinfo{booktitle}{\emph{Proceedings of the 56th Annual Meeting of the Association for Computational Linguistics (Volume 1: Long Papers)}}. \bibinfo{pages}{2556--2565}.
\newblock


\bibitem[Simonyan and Zisserman(2014)]%
        {simonyan2014very}
\bibfield{author}{\bibinfo{person}{Karen Simonyan} {and} \bibinfo{person}{Andrew Zisserman}.} \bibinfo{year}{2014}\natexlab{}.
\newblock \showarticletitle{Very deep convolutional networks for large-scale image recognition}.
\newblock \bibinfo{journal}{\emph{arXiv preprint arXiv:1409.1556}} (\bibinfo{year}{2014}).
\newblock


\bibitem[Stallkamp et~al\mbox{.}(2011)]%
        {stallkamp2011german}
\bibfield{author}{\bibinfo{person}{Johannes Stallkamp}, \bibinfo{person}{Marc Schlipsing}, \bibinfo{person}{Jan Salmen}, {and} \bibinfo{person}{Christian Igel}.} \bibinfo{year}{2011}\natexlab{}.
\newblock \showarticletitle{The German traffic sign recognition benchmark: a multi-class classification competition}. In \bibinfo{booktitle}{\emph{IEEE IJCNN}}.
\newblock


\bibitem[Thrun(1995)]%
        {thrun1995lifelong}
\bibfield{author}{\bibinfo{person}{Sebastian Thrun}.} \bibinfo{year}{1995}\natexlab{}.
\newblock \showarticletitle{A lifelong learning perspective for mobile robot control}. In \bibinfo{booktitle}{\emph{IROS}}.
\newblock


\bibitem[Tung and Mori(2019)]%
        {tung2019similarity}
\bibfield{author}{\bibinfo{person}{Frederick Tung} {and} \bibinfo{person}{Greg Mori}.} \bibinfo{year}{2019}\natexlab{}.
\newblock \showarticletitle{Similarity-preserving knowledge distillation}. In \bibinfo{booktitle}{\emph{ICCV}}.
\newblock


\bibitem[Turner et~al\mbox{.}(2019)]%
        {turner2019label}
\bibfield{author}{\bibinfo{person}{Alexander Turner}, \bibinfo{person}{Dimitris Tsipras}, {and} \bibinfo{person}{Aleksander Madry}.} \bibinfo{year}{2019}\natexlab{}.
\newblock \showarticletitle{Label-consistent backdoor attacks}.
\newblock \bibinfo{journal}{\emph{arXiv preprint arXiv:1912.02771}} (\bibinfo{year}{2019}).
\newblock


\bibitem[Van~der Maaten and Hinton(2008)]%
        {van2008visualizing}
\bibfield{author}{\bibinfo{person}{Laurens Van~der Maaten} {and} \bibinfo{person}{Geoffrey Hinton}.} \bibinfo{year}{2008}\natexlab{}.
\newblock \showarticletitle{Visualizing data using t-SNE.}
\newblock \bibinfo{journal}{\emph{Journal of machine learning research}} \bibinfo{volume}{9}, \bibinfo{number}{11} (\bibinfo{year}{2008}).
\newblock


\bibitem[Wang et~al\mbox{.}(2019)]%
        {wang2019neural}
\bibfield{author}{\bibinfo{person}{Bolun Wang}, \bibinfo{person}{Yuanshun Yao}, \bibinfo{person}{Shawn Shan}, \bibinfo{person}{Huiying Li}, \bibinfo{person}{Bimal Viswanath}, \bibinfo{person}{Haitao Zheng}, {and} \bibinfo{person}{Ben~Y Zhao}.} \bibinfo{year}{2019}\natexlab{}.
\newblock \showarticletitle{Neural cleanse: Identifying and mitigating backdoor attacks in neural networks}. In \bibinfo{booktitle}{\emph{S\&P}}.
\newblock


\bibitem[Wang et~al\mbox{.}(2024)]%
        {wang2024comprehensive}
\bibfield{author}{\bibinfo{person}{Liyuan Wang}, \bibinfo{person}{Xingxing Zhang}, \bibinfo{person}{Hang Su}, {and} \bibinfo{person}{Jun Zhu}.} \bibinfo{year}{2024}\natexlab{}.
\newblock \showarticletitle{A comprehensive survey of continual learning: Theory, method and application}.
\newblock \bibinfo{journal}{\emph{IEEE Transactions on Pattern Analysis and Machine Intelligence}} (\bibinfo{year}{2024}).
\newblock


\bibitem[Wei et~al\mbox{.}(2023)]%
        {wei2023shared}
\bibfield{author}{\bibinfo{person}{Shaokui Wei}, \bibinfo{person}{Mingda Zhang}, \bibinfo{person}{Hongyuan Zha}, {and} \bibinfo{person}{Baoyuan Wu}.} \bibinfo{year}{2023}\natexlab{}.
\newblock \showarticletitle{Shared adversarial unlearning: Backdoor mitigation by unlearning shared adversarial examples}. In \bibinfo{booktitle}{\emph{NeurIPS}}.
\newblock


\bibitem[Weng et~al\mbox{.}(2020)]%
        {weng2020trade}
\bibfield{author}{\bibinfo{person}{Cheng-Hsin Weng}, \bibinfo{person}{Yan-Ting Lee}, {and} \bibinfo{person}{Shan-Hung~Brandon Wu}.} \bibinfo{year}{2020}\natexlab{}.
\newblock \showarticletitle{On the trade-off between adversarial and backdoor robustness}. In \bibinfo{booktitle}{\emph{NeurIPS}}.
\newblock


\bibitem[Wu et~al\mbox{.}(2022)]%
        {wu2022backdoorbench}
\bibfield{author}{\bibinfo{person}{Baoyuan Wu}, \bibinfo{person}{Hongrui Chen}, \bibinfo{person}{Mingda Zhang}, \bibinfo{person}{Zihao Zhu}, \bibinfo{person}{Shaokui Wei}, \bibinfo{person}{Danni Yuan}, {and} \bibinfo{person}{Chao Shen}.} \bibinfo{year}{2022}\natexlab{}.
\newblock \showarticletitle{Backdoorbench: A comprehensive benchmark of backdoor learning}. In \bibinfo{booktitle}{\emph{NeurIPS}}.
\newblock


\bibitem[Wu and Wang(2021)]%
        {wu2021adversarial}
\bibfield{author}{\bibinfo{person}{Dongxian Wu} {and} \bibinfo{person}{Yisen Wang}.} \bibinfo{year}{2021}\natexlab{}.
\newblock \showarticletitle{Adversarial neuron pruning purifies backdoored deep models}. In \bibinfo{booktitle}{\emph{NeurIPS}}.
\newblock


\bibitem[Wu et~al\mbox{.}(2017)]%
        {wu2017tiny}
\bibfield{author}{\bibinfo{person}{Jiayu Wu}, \bibinfo{person}{Qixiang Zhang}, {and} \bibinfo{person}{Guoxi Xu}.} \bibinfo{year}{2017}\natexlab{}.
\newblock \showarticletitle{Tiny imagenet challenge}.
\newblock \bibinfo{journal}{\emph{Technical report}} (\bibinfo{year}{2017}).
\newblock


\bibitem[Xu et~al\mbox{.}(2024)]%
        {xu2024towards}
\bibfield{author}{\bibinfo{person}{Xiong Xu}, \bibinfo{person}{Kunzhe Huang}, \bibinfo{person}{Yiming Li}, \bibinfo{person}{Zhan Qin}, {and} \bibinfo{person}{Kui Ren}.} \bibinfo{year}{2024}\natexlab{}.
\newblock \showarticletitle{Towards reliable and efficient backdoor trigger inversion via decoupling benign features}. In \bibinfo{booktitle}{\emph{ICLR}}.
\newblock


\bibitem[Yin et~al\mbox{.}(2024)]%
        {yin2024physical}
\bibfield{author}{\bibinfo{person}{Wen Yin}, \bibinfo{person}{Jian Lou}, \bibinfo{person}{Pan Zhou}, \bibinfo{person}{Yulai Xie}, \bibinfo{person}{Dan Feng}, \bibinfo{person}{Yuhua Sun}, \bibinfo{person}{Tailai Zhang}, {and} \bibinfo{person}{Lichao Sun}.} \bibinfo{year}{2024}\natexlab{}.
\newblock \showarticletitle{Physical Backdoor: Towards Temperature-based Backdoor Attacks in the Physical World}. In \bibinfo{booktitle}{\emph{CVPR}}.
\newblock


\bibitem[Yu et~al\mbox{.}(2022)]%
        {yu2022availability}
\bibfield{author}{\bibinfo{person}{Da Yu}, \bibinfo{person}{Huishuai Zhang}, \bibinfo{person}{Wei Chen}, \bibinfo{person}{Jian Yin}, {and} \bibinfo{person}{Tie-Yan Liu}.} \bibinfo{year}{2022}\natexlab{}.
\newblock \showarticletitle{Availability attacks create shortcuts}. In \bibinfo{booktitle}{\emph{SIGKDD}}.
\newblock


\bibitem[Zeng et~al\mbox{.}(2022)]%
        {zeng2021adversarial}
\bibfield{author}{\bibinfo{person}{Yi Zeng}, \bibinfo{person}{Si Chen}, \bibinfo{person}{Won Park}, \bibinfo{person}{Z~Morley Mao}, \bibinfo{person}{Ming Jin}, {and} \bibinfo{person}{Ruoxi Jia}.} \bibinfo{year}{2022}\natexlab{}.
\newblock \showarticletitle{Adversarial unlearning of backdoors via implicit hypergradient}. In \bibinfo{booktitle}{\emph{ICLR}}.
\newblock


\bibitem[Zeng et~al\mbox{.}(2023)]%
        {zeng2023narcissus}
\bibfield{author}{\bibinfo{person}{Yi Zeng}, \bibinfo{person}{Minzhou Pan}, \bibinfo{person}{Hoang~Anh Just}, \bibinfo{person}{Lingjuan Lyu}, \bibinfo{person}{Meikang Qiu}, {and} \bibinfo{person}{Ruoxi Jia}.} \bibinfo{year}{2023}\natexlab{}.
\newblock \showarticletitle{Narcissus: A practical clean-label backdoor attack with limited information}. In \bibinfo{booktitle}{\emph{CCS}}.
\newblock


\bibitem[Zhang et~al\mbox{.}(2024)]%
        {zhang2024exploring}
\bibfield{author}{\bibinfo{person}{Kaiyuan Zhang}, \bibinfo{person}{Siyuan Cheng}, \bibinfo{person}{Guangyu Shen}, \bibinfo{person}{Guanhong Tao}, \bibinfo{person}{Shengwei An}, \bibinfo{person}{Anuran Makur}, \bibinfo{person}{Shiqing Ma}, {and} \bibinfo{person}{Xiangyu Zhang}.} \bibinfo{year}{2024}\natexlab{}.
\newblock \showarticletitle{Exploring the orthogonality and linearity of backdoor attacks}. In \bibinfo{booktitle}{\emph{S\&P}}.
\newblock


\bibitem[Zheng et~al\mbox{.}(2022a)]%
        {zheng2022data}
\bibfield{author}{\bibinfo{person}{Runkai Zheng}, \bibinfo{person}{Rongjun Tang}, \bibinfo{person}{Jianze Li}, {and} \bibinfo{person}{Li Liu}.} \bibinfo{year}{2022}\natexlab{a}.
\newblock \showarticletitle{Data-free backdoor removal based on channel lipschitzness}. In \bibinfo{booktitle}{\emph{ECCV}}.
\newblock


\bibitem[Zheng et~al\mbox{.}(2022b)]%
        {zheng2022pre}
\bibfield{author}{\bibinfo{person}{Runkai Zheng}, \bibinfo{person}{Rongjun Tang}, \bibinfo{person}{Jianze Li}, {and} \bibinfo{person}{Li Liu}.} \bibinfo{year}{2022}\natexlab{b}.
\newblock \showarticletitle{Pre-activation distributions expose backdoor neurons}. In \bibinfo{booktitle}{\emph{NeurIPS}}.
\newblock


\bibitem[Zhu et~al\mbox{.}(2024)]%
        {zhu2024breaking}
\bibfield{author}{\bibinfo{person}{Mingli Zhu}, \bibinfo{person}{Siyuan Liang}, {and} \bibinfo{person}{Baoyuan Wu}.} \bibinfo{year}{2024}\natexlab{}.
\newblock \showarticletitle{Breaking the false sense of security in backdoor defense through re-activation attack}. In \bibinfo{booktitle}{\emph{NeurIPS}}.
\newblock


\end{thebibliography}
